\documentclass[12pt]{article}

\usepackage{times}
\usepackage{helvet}
\usepackage{courier}
\usepackage{graphicx}
\usepackage{natbib}
\usepackage{caption}
\usepackage[T1]{fontenc}
\usepackage{booktabs}
\usepackage{nicefrac}
\usepackage{microtype}
\usepackage{xcolor}
\usepackage{enumitem}
\usepackage{amsmath}
\usepackage{amssymb}
\usepackage{amsthm}
\usepackage{amsfonts}
\usepackage{algorithm}
\usepackage{algpseudocode}
\usepackage{newfloat}
\usepackage{listings}

\usepackage{lipsum}

\newtheorem{assumption}{Assumption}
\newtheorem{lemma}{Lemma}
\newtheorem{theorem}{Theorem}
\newtheorem{corollary}{Corollary}

\definecolor{darkyellow}{rgb}{1.0, 0.8, 0.5}
\definecolor{darkgreen}{rgb}{0.0, 0.5, 0.0}

\def \DefinedAs {\mathrel{\mathop:}=}
\def \Tr {\top}
\def \reals    {{\mathbb R}}
\def \E {\mathbb{E}}

\def \subjectto {\operatornamewithlimits{s.t.}}
\def \max {\operatornamewithlimits{max}}
\def \min {\operatornamewithlimits{min}}
\def \argmax {\operatornamewithlimits{argmax}}
\def \argmin {\operatornamewithlimits{argmin}}
\def \ccalP {{\ensuremath{\mathcal P}}}
\def \ccalV {{\ensuremath{\mathcal V}}}

\newcommand{\reddashedline}{%
  \raisebox{0.5ex}{\textcolor{red}{\rule{1.5mm}{1.0pt}}}%
  \hspace{0.5mm}%
  \raisebox{0.5ex}{\textcolor{red}{\rule{1.5mm}{1.0pt}}}%
}

\newcommand{\blackline}{%
  \raisebox{0.5ex}{{\rule{0.35cm}{1.0pt}}}%
}

\newcommand{\yellowline}{%
  \raisebox{0.5ex}{\textcolor{darkyellow}{\rule{0.35cm}{1.0pt}}}%
}

\newcommand{\yellowdashedline}{%
  \raisebox{0.5ex}{\textcolor{darkyellow}{\rule{1.5mm}{1.0pt}}}%
  \hspace{0.5mm}%
  \raisebox{0.5ex}{\textcolor{darkyellow}{\rule{1.5mm}{1.0pt}}}%
}

\newcommand{\bluedotted}{%
  \raisebox{0.5ex}{\textcolor{blue}{\rule{0.4mm}{1.0pt}}}%
  \hspace{0.5mm}%
  \raisebox{0.5ex}{\textcolor{blue}{\rule{0.4mm}{1.0pt}}}%
  \hspace{0.5mm}%
  \raisebox{0.5ex}{\textcolor{blue}{\rule{0.4mm}{1.0pt}}}%
}

\newcommand{\greenmixed}{%
  \raisebox{0.5ex}{\textcolor{darkgreen}{\rule{1.5mm}{1.0pt}}}%
  \hspace{0.5mm}%
  \raisebox{0.5ex}{\textcolor{darkgreen}{\rule{0.4mm}{1.0pt}}}%
  \hspace{0.5mm}%
  \raisebox{0.5ex}{\textcolor{darkgreen}{\rule{1.5mm}{1.0pt}}}%
}

\title{\bf Deterministic Policy Gradient Primal-Dual Methods \\ for Continuous-Space Constrained MDPs}

\author{
	Sergio~Rozada \footnote{S.\ Rozada, and A.\ G.\ Marques
    are with the Department of Signal Theory and Communications, King Juan Carlos University, Madrid, Spain.}~~\footnotemark[3]\quad Dongsheng~Ding \footnote{D.\ Ding, and A.\ Ribeiro are with the Department of Electrical and Systems Engineering, University of Pennsylvania, Philadelphia, United States.}~~\footnote{Correspondence to: s.rozada.2019@alumnos.urjc.es and dongshed@seas.upenn.edu} \quad Antonio~G.~Marques \footnotemark[1] \quad Alejandro~Ribeiro \footnotemark[2]
}
\date{}

\begin{document}

\maketitle

\begin{abstract}
We study the problem of computing deterministic optimal policies for constrained Markov decision processes (MDPs) with continuous state and action spaces, which are widely encountered in constrained dynamical systems. Designing deterministic policy gradient methods in continuous state and action spaces is particularly challenging due to the lack of enumerable state-action pairs and the adoption of deterministic policies, hindering the application of existing policy gradient methods. To this end, we develop a deterministic policy gradient primal-dual method to find an optimal deterministic policy with non-asymptotic convergence. Specifically, we leverage regularization of the Lagrangian of the constrained MDP to propose a deterministic policy gradient primal-dual (D-PGPD) algorithm that updates the deterministic policy via a quadratic-regularized gradient ascent step and the dual variable via a quadratic-regularized gradient descent step. We prove that the primal-dual iterates of D-PGPD converge at a sub-linear rate to an optimal regularized primal-dual pair. We instantiate D-PGPD with function approximation and prove that the primal-dual iterates of D-PGPD converge at a sub-linear rate to an optimal regularized primal-dual pair, up to a function approximation error. Furthermore, we demonstrate the effectiveness of our method in two continuous control problems: robot navigation and fluid control. This appears to be the first work that proposes a deterministic policy search method for continuous-space constrained MDPs.
\end{abstract}

\section{Introduction}

Constrained Markov decision processes (MDPs) are a standard framework for incorporating system specifications into dynamical systems \cite{altman2021constrained, brunke2022safe}. In recent years, constrained MDPs have attracted significant attention in constrained Reinforcement Learning (RL), whose goal is to derive optimal control policies through interaction with unknown dynamical systems \cite{achiam2017constrained, tessler2018reward}. Policy gradient-based constrained learning methods have become the workhorse driving recent successes across various disciplines, e.g., navigation \cite{paternain2022safe}, video compression \cite{mandhane2022muzero}, and finance \cite{chow2018risk}.

This paper is motivated by two observations. First, continuous state-action spaces are pervasive in dynamical systems, yet most methods in constrained RL are designed for discrete state and/or action spaces \cite{borkar2005actor,efroni2020exploration,ding2022convergence,singh2022learning}. Second, the literature on constrained RL largely focuses on stochastic policies. However, randomly taking actions by following a stochastic policy is often prohibitive in practice, especially in safety-critical domains~\cite{sehnke2010parameter,li2022continuous, gao2023improved}. Deterministic policies alleviate such concerns, but (i) they might lead to sub-optimal solutions \cite{ross1989randomized, altman2021constrained}; and (ii) computing them is NP-complete \cite{feinberg2000constrained, dolgov2005stationary}. Nevertheless, there is a rich body of constrained control literature that studies problems where optimal policies are  deterministic \cite{posa2016optimization, tsiamis2020risk, zhao2021primal, ma2022alternating}. Viewing this gap, we study the problem of finding optimal \emph{deterministic} policies for constrained MDPs with \emph{continuous} state-action spaces.

A key consideration of this paper is the fact that deterministic policies are sub-optimal in finite state-action spaces, but sufficient for constrained MDPs with continuous state-action spaces \cite{feinberg2002nonatomic, feinberg2019sufficiency}. This enables our formulation of a constrained RL problem with deterministic policies. To develop a tractable deterministic policy search method, we introduce a regularized Lagrangian approach that leverages proximal optimization methods. Moreover, we use function approximation to ensure scalability in continuous state-action spaces. Our main contribution is four-fold.

\begin{enumerate}
    \item[(i)] We introduce a deterministic policy constrained RL problem for a constrained MDP with continuous state-action spaces and prove that the problem exhibits zero duality gap, despite being constrained to deterministic policies.
    \item[(ii)] We propose a regularized deterministic policy gradient primal-dual (D-PGPD) algorithm that updates the primal policy via a proximal-point-type step and the dual variable via a gradient descent step, and we prove that the primal-dual iterates of D-PGPD converge to a set of regularized optimal primal-dual pairs at a sub-linear rate.
    \item[(iii)] We propose an approximation for D-PGPD by including function approximation. We prove that the primal-dual iterates of the approximated D-PGPD converge at a sub-linear rate, up to a function approximation error.
    \item[(iv)] We demonstrate that D-PGPD addresses the classical constrained navigation problem involving several types of cost functions and constraints. We show that D-PGPD can solve non-linear fluid control problems under constraints.
\end{enumerate}

\noindent \textbf{Related work.} 
Deterministic policy search has been studied in the context of unconstrained MDPs \cite{silver2014deterministic, lillicrap2015continuous, kumar2020zeroth, lan2022policy}. In constrained setups, however, deterministic policies have been largely restricted to occupancy measure optimization in finite state-action spaces \cite{dolgov2005stationary} or are embedded in hyper-policies~\cite{sehnke2010parameter,montenegro2024learning, montenegro2024last}. This work extends deterministic policy search to constrained MDPs with continuous state-action spaces, overcoming two main roadblocks: the sub-optimality of deterministic policies and the NP-completeness of computing them \cite{ross1989randomized, feinberg2000constrained, dolgov2005stationary, altman2021constrained, mcmahan2024deterministic}.  First, we show that deterministic policies are sufficient for constrained MDPs in continuous state-action spaces \cite{feinberg2002nonatomic, feinberg2019sufficiency}, leveraging the convexity of the value image to establish strong duality in the deterministic policy space. Second, we overcome computational intractability by introducing a quadratic regularization of the reward function and proposing a regularization-based primal-dual algorithm. This algorithm exploits the structure of value functions and achieves last-iterate convergence to an optimal deterministic policy. While last-iterate convergence of primal-dual algorithms has been explored in constrained RL \cite{moskovitz2023reload, ding2024last, ding2024resilient}, existing methods focus on stochastic policies and finite-action spaces. In control, extensive work addresses deterministic policies in constrained setups with continuous state-action spaces \cite{scokaert1998constrained, lim1999stochastic}. However, these approaches are typically model-based and tailored to specific structured problems \cite{posa2016optimization, tsiamis2020risk, zhao2021infinite, zhao2021primal, ma2022alternating}. Bridging constrained control and RL has also been explored \cite{kakade2020information, zahavy2021reward, li2023double}, but these methods remain model-based and focus on stochastic policies. In contrast, we propose a model-free deterministic policy search method for constrained MDPs with continuous state-action spaces.

\section{Preliminaries}

We consider a discounted constrained MDP, denoted by the tuple $(S, A, p, r, u, b, \gamma, \rho)$. Here, $S \subseteq \reals^{d_s}$ and $A\subseteq \reals^{d_a}$ are continuous state-action spaces with dimensions $d_s$ and $d_a$, and bounded actions $\| a \| \leq A_{\text{max}}$ for all $a \in A$; $p(\cdot\,|\,s, a)$ is a probability measure over $S$ parametrized by the state-action pairs $(s, a) \in S \times A$; $r$, $u$: $S \times A \mapsto [0, 1]$ are reward/utility functions; $b$ is a constraint threshold; $\gamma \in [0,1)$ is a discount factor; and $\rho$ is a probability measure that specifies an initial state. We consider the set of all deterministic policies $\Pi$ in which a policy $\pi$: $S \mapsto A$ maps states to actions. The transition $p$, the initial state distribution $\rho$, and the policy $\pi$ define a distribution over trajectories $\{s_t, a_t, r_t, u_t\}_{t=0}^\infty$, where $s_0 \sim \rho$, $a_t = \pi(s_t)$, $r_t=r(s_t, a_t)$, $u_t=u(s_t, a_t)$ and $s_{t+1} \sim p(\cdot \,|\, s_t, a_t)$. Given $\pi$, we define the value function $V^\pi_r$: $S \to \reals$ as the expected sum of discounted rewards
\begin{align}
    \nonumber
    V_r^\pi (s) 
    \; &\DefinedAs \;  
    \E_\pi \left[ \sum_{t=0}^\infty \gamma^t r(s_t, a_t) \; | \; s_0 = s \right].
\end{align}
For the utility function, we define the corresponding value function $V^\pi_u$. Their expected values over the initial state distribution $\rho$ are denoted as $V_r (\pi) \DefinedAs \E_\rho [V_r^\pi (s)]$ and $V_u (\pi) \DefinedAs \E_\rho [V_u^\pi (s)]$, where we drop the dependence on $\rho$ for simplicity of notation. Boundedness of $r$ and $u$ leads to $V_r(\pi)$, $ V_u(\pi) \in [ 0, 1/(1 - \gamma) ]$. We introduce a discounted state visitation distribution $d_{s_0}^\pi(B) \DefinedAs (1-\gamma) \sum_{t\,=\,0}^\infty \text{Pr}(s_t\in B\,\vert\,\pi,s_0)$ for any $B\subseteq S$ and define $d_\rho^\pi (s) \DefinedAs \E_{s_0 \,\sim\, \rho} [d_{s_0}^\pi(s)]$. For the reward function $r$, we define the state-action value function $Q^\pi_r$: $S \times A \to \reals$ given an initial action $a$ while following  $\pi$,
\begin{align}
    \nonumber
    Q_r^\pi (s, a) 
    \; \DefinedAs \;
    \E_\pi \left[ \sum_{t=0}^\infty \gamma^t r(s_t, a_t) | s_0 = s, a_0=a \right].
\end{align}
We let the associated advantage function $A^\pi_r$: $S \times A \to \reals$ be $A_r^\pi (s, a) \; \DefinedAs \; Q_r^\pi (s, a) - V_r^\pi (s)$. Similarly, we define $Q^\pi_u$: $S \times A \to \reals$ and $A^\pi_u$: $S \times A \to \reals$ for the utility function $u$.

A policy is optimal for a given reward function when it maximizes the corresponding value function. However, the value functions $V_r(\pi)$ and $V_u(\pi)$ are usually in conflict, e.g., a policy that maximizes $V_r(\pi)$ is not necessary good for $V_u(\pi)$. To trade off two conflicting objectives, constrained MDP aims to find an optimal policy $\pi^\star$ that maximizes the reward value function $V_r(\pi)$ subject to an inequality constraint on the utility value function $V_u(\pi) \geq b$, where we assume $b \in (0, 1/(1 - \gamma)]$ to avoid trivial solutions. We use a single constraint for the sake of simplicity, but our results extend to problems with multiple constraints. We translate the constraint $V_u(\pi) \geq b$ into the constraint $V_g(\pi) \geq 0$ for $g \DefinedAs u - (1-\gamma)b$, where $g$: $S \times A \mapsto [-1, 1]$ denotes the translated utility. This leads to the following problem
\begin{equation}
    \label{eq::P-CRL}
    \begin{array}{rl}
    \displaystyle\max_{\pi\, \in\, \Pi}
    & \;\; V_{r} (\pi) 
    \\[0.2cm]
    \subjectto & \;\; V_g (\pi) 
    \;\geq\; 0 .
    \end{array}
\end{equation}

Restricting Problem~\eqref{eq::P-CRL} to deterministic policies poses several challenges. Deterministic policies can be sub-optimal in constrained MDPs with finite state-action spaces \cite{ross1989randomized, altman2021constrained}, and when they exist, finding them is a NP-complete \cite{feinberg2000constrained}. 
Problem \eqref{eq::P-CRL} is non-convex in the policy but can be reformulated as a linear program using occupancy measures with stochastic policies \cite{paternain2019constrained}.
However, the occupancy measure representation of \eqref{eq::P-CRL} is a \emph{non-linear} and \emph{non-convex} problem when only deterministic policies are considered \cite{dolgov2005stationary}. Finally, multiple policies can achieve the optimal value function $V^{\pi^\star}_P$ while satisfying the constraint. We denote the set of all maximizers of \eqref{eq::P-CRL} that attain $V^{\pi^\star}_P$ as $\Pi^\star$.
To address these points, we observe that deterministic policies are sufficient in constrained MDPs with continuous state-action spaces under the following assumption \cite{feinberg2002nonatomic, feinberg2019sufficiency}.

\begin{assumption}[Non-atomicity]
    \label{as::non-atomicity}
    The MDP is non-atomic, i.e., $\rho(s)=0$ and $p(s'\,|\,s,a)=0$ for all $s$, $s' \in S$ and $a \in A$.
\end{assumption}

Assumption~\ref{as::non-atomicity} is mild in practice. Since stochastic perturbations are common in physical systems with continuous state and action spaces \cite{anderson2007optimal}, the probability measures $\rho$ and $p(\cdot\,\vert\,s,a)$ are normally atomless, i.e., for any measurable set $B \subseteq S$ with probability measures $\rho(B)$ and $p(B\,|\,s,a)$, there exists a measurable subset $B' \subset B$ that has smaller non-zero probability measures $\rho(B) > \rho(B') > 0$ and $p(B\,|\,s,a) > p(B'\,|\,s,a) > 0$ for any $s \in S$ and $a \in A$. In other words, the transition probability and the initial probability do not concentrate in a single state \cite{feinberg2019sufficiency}. When a constrained MDP is non-atomic, only considering deterministic policies is sufficient \cite{feinberg2019sufficiency}. Specifically, let $V(\pi) \DefinedAs [\,V_r(\pi)\; V_g(\pi)\,]^\Tr$ denote the vector of value functions for a given policy $\pi$. We define a \emph{deterministic value image} $\ccalV_D \DefinedAs \{ V(\pi) \,\vert\, \pi \in \Pi \}$, which is a set of all attainable vector value functions for deterministic policies. We denote by $\ccalV_T$ a \emph{value image} for all policies. The deterministic value image $\ccalV_D$ and the value image $\ccalV_T$ are equivalent under Assumption \ref{as::non-atomicity} for discounted MDPs (see Lemmas \ref{lm::unif_mdp_lemma} and \ref{lm::convex_span_lemma}  in Appendix \ref{app::supporting_lemmas}). Therefore, the optimal value function of a non-atomic constrained MDP is contained in the deterministic value image $\ccalV_D$. Furthermore, the deterministic value image $\ccalV_D$ is a convex set, even though each value function $V(\pi) \in \ccalV_D$ is non-convex in policy $\pi$ (see Lemmas \ref{lm::convex_span_lemma} and \ref{lm::sufficiency_deterministic_lemma} in Appendix \ref{app::supporting_lemmas}). These  observations are summarized below.

\begin{lemma}[Sufficiency of deterministic policies]
    \label{lm::sufficiency_deterministic}
    For a non-atomic discounted MDP, the deterministic value image $\ccalV_D$ is convex, and equals the value image $\ccalV_T$, i.e., $\ccalV_D = \ccalV_T$.
\end{lemma}

\subsection{Zero Duality Gap}
\label{ss::problem}

With the convexity of the deterministic value image $\ccalV_D$ in hand, we next establish zero duality gap for Problem~\eqref{eq::P-CRL}. We begin with a standard feasibility assumption. 

\begin{assumption}[Feasibility]
    \label{as::feasibility}
    There exists a deterministic policy $\Tilde{\pi} \in \Pi$ and $\xi > 0$ such that $V_g (\Tilde{\pi}) \geq \xi$.
\end{assumption}

We dualize the constraint by introducing the dual variable $\lambda \in \mathbb{R}^+$ and the Lagrangian $L(\pi, \lambda) \DefinedAs V_r(\pi) + \lambda V_g(\pi)$. For a fixed $\lambda$, let $\Pi(\lambda)$ be the set of Lagrangian maximizers. The Lagrangian $L(\pi, \lambda)$ is equivalent to the value function $V_\lambda(\pi)$ associated with the combined reward/utility function $r_\lambda(s,a) = r(s,a) + \lambda g(s,a)$. The dual function $D(\lambda) \DefinedAs \max_{\pi \in \Pi} V_\lambda(\pi)$ is an upper bound of Problem \eqref{eq::P-CRL}, and the dual problem searches for the tightest primal upper bound
\begin{align}
    \label{eq::D-CRL}
    \min_{\lambda \,\in\, \reals^+} \; D(\lambda).
\end{align}

We denote by $V_D^{\lambda^\star}$ the optimal value of the dual function, where $\lambda^\star$ is a minimizer of the dual Problem~\eqref{eq::D-CRL}. Despite being non-convex in the policy, if we replace the deterministic policy space in Problem~\eqref{eq::P-CRL} with the stochastic policy space, then it is known that Problem~\eqref{eq::P-CRL} has zero duality gap \cite{paternain2019constrained}. The proof capitalizes on the convexity of the occupancy measure representation of \eqref{eq::P-CRL} for stochastic policies. However, this occupancy-measure-based argument does not carry to deterministic policies, since the occupancy measure representation of Problem \eqref{eq::P-CRL} is non-convex when only deterministic policies are used \cite{dolgov2005stationary}. Instead, we leverage the convexity of the deterministic value image $\ccalV_D$ to prove that strong duality holds for Problem~\eqref{eq::P-CRL}; see Appendices \ref{app::convexity_value_image} and \ref{app::zero_duality_gap} for more details and the proof.

\begin{theorem}[Zero duality gap]
    \label{thm::zero_duality_gap}
    Let Assumption \ref{as::non-atomicity} hold. Then, Problem \eqref{eq::P-CRL} has zero duality gap, i.e., $V^{\pi^\star}_P = V^{\lambda^\star}_D$.
\end{theorem}

Theorem \ref{thm::zero_duality_gap} states that the optimal values of Problems \eqref{eq::P-CRL} and \eqref{eq::D-CRL} are equivalent, extending the zero duality gap result in \cite{paternain2019constrained} to deterministic policies under the non-atomicity assumption. However, recovering an optimal policy $\pi^\star$ can be non-trivial even if an optimal dual variable $\lambda^\star$ is obtained from the dual problem~\cite{zahavy2021reward}. The root cause is that the maximizers of the primal problem $\Pi^\star$ and those of the Lagrangian for an optimal multiplier $\Pi(\lambda^\star)$ are different sets \cite[Proposition 1]{calvo2023state}. To address this, we employ Theorem~\ref{thm::zero_duality_gap} to interpret Problem~\eqref{eq::P-CRL} as a saddle point problem. Zero duality gap implies that an optimal primal-dual pair $(\pi^\star, \lambda^\star)$ is a saddle point of the Lagrangian $L(\pi,\lambda)$, and satisfies the mini-max condition
\begin{equation}
\nonumber
    L(\pi,\lambda^\star)
    \; \leq \;
    L(\pi^\star,\lambda^\star)
    \; \leq \;
    L(\pi^\star,\lambda) \quad \forall (\pi,\lambda) \in \Pi \times \Lambda,
\end{equation}
where $\lambda$ is bounded in the interval $\Lambda \DefinedAs [0, \lambda_{\text{max}}]$, with $\lambda_{\text{max}} \DefinedAs 1/((1 - \gamma)\xi)$; see Lemma \ref{lm::bound_lambda} in Appendix \ref{app::supporting_lemmas}. In this paper, we refer to saddle points that satisfy the mini-max condition for all pairs $(\pi, \lambda) \in \Pi \times \Lambda$ as \emph{global} saddle points. Our main task in Section~\ref{ss::proposed_method} is to find a global saddle point of the Lagrangian $L(\pi,\lambda)$ that is a solution to Problem \eqref{eq::P-CRL}.
     
\subsection{Constrained Regulation Problem}
\label{ss::clq}

We illustrate Problem~\eqref{eq::P-CRL} using the following example
\begin{subequations} \label{eq::C-DS}
    \begin{align}
        \max_{\pi \in \Pi} 
        & \;\;  \E \left[ \sum_{t\,=\,0}^\infty  \gamma^t \left( s_t^\Tr G_1 s_t + a_t^\Tr R_1 a_t \right) \right] \label{eq::C-DS_obj} \\
        \subjectto & \;\; \E \left[ \sum_{t\,=\,0}^\infty  \gamma^t \left( s_t^\Tr G_2 s_t + a_t^\Tr R_2 a_t \right) \right] \geq b \label{eq::C-DS_a} \\
        & \;\; -b_s \leq C_s  s_t  \leq b_s, \;\; -b_a \leq C_a a_t \leq b_a  \label{eq::C-DS_b} \\
        & \;\; s_{t+1} = B_0 s_t + B_1 a_t + \omega_t, \;  s_0 \sim \rho \label{eq::C-DS_c}
    \end{align}
\end{subequations}
where $B_0 \in \reals^{d_s \times d_s}$ and $B_1 \in \reals^{d_s \times d_a}$ denote the system dynamics, $\omega_t$ is the standard Gaussian noise, $\rho$ is the initial state distribution, and $G_1$, $G_2 \in \reals^{d_s \times d_s}$ and $R_1$, $R_2 \in \reals^{d_a \times d_a}$ are negative semi-definite reward matrices. The constraint threshold is $b$, with $C_s \in \reals^{d_s \times d_s}$, $C_a \in \reals^{d_a \times d_a}$, $b_s \in \reals^{d_s}$, and $b_a \in \reals^{d_a}$ specifying state-action constraints, e.g., if $C_s$, $C_a$ are identity matrices, $b_s$, $b_a$ limit state and action ranges. Equations \eqref{eq::C-DS_obj}, \eqref{eq::C-DS_b}, and \eqref{eq::C-DS_c} describe the constrained regulation problem under Gaussian disturbances \cite{bemporad2002explicit, stathopoulos2016solving}, where the optimal policy is deterministic \cite{scokaert1998constrained}. We add a general constraint \eqref{eq::C-DS_a}. The Markovian transition dynamics \eqref{eq::C-DS_c} are linear, and the Gaussian noise $\omega_t$ is non-atomic, rendering the transition probabilities non-atomic. If $\rho$ is non-atomic, the underlying MDP of \eqref{eq::C-DS} is also non-atomic. The reward function $r(s, a) \DefinedAs s^\Tr G_1 s + a^\Tr R_1 a$ induces a value function $V_r(\pi)$, bounded within $[r_\text{min}/(1-\gamma), 0]$, with $r_\text{min} \DefinedAs b_s^\Tr G_1 b_s + b_a^\Tr R_1 b_a$. Similarly, for $u(s, a) \DefinedAs s^\Tr G_2 s + a^\Tr R_2 a$, the utility value $V_u$ is also bounded. Therefore, this problem is an instance of Problem \eqref{eq::P-CRL}, assuming the state space is bounded with $\|s\| \leq S_{\text{max}}$.

\section{Method and Theory}
\label{ss::proposed_method}

While our problem has zero duality gap, finding an optimal dual $\lambda^\star$ poses a significant challenge, due to the presence of multiple saddle points in the Lagrangian. To address it, we resort to the regularization method. More specifically, we introduce two regularizers. First, the term $h(\lambda) \DefinedAs \lambda^2$ promotes convexity in the Lagrange multiplier $\lambda$. Second, the term $h_a(a) \DefinedAs -\|a\|^2$ promotes concavity in the reward function $r$ by penalizing large actions selected by the policy $\pi$.
The associated value function is defined as $H^\pi(s) \DefinedAs \E_\pi \left[ \sum_{t=0}^\infty \gamma^t h_a(a_t) \,\vert\, s \right]$, and leads to the regularizer $H(\pi) \DefinedAs \E_\rho[H^\pi(s)]$. Now, we consider the problem
\begin{align}
    \label{eq::RD-CRL}
    \min_{\lambda \,\in\, \Lambda} \; \max_{\pi \,\in\, \Pi} \;\; L_\tau(\pi, \lambda) \DefinedAs V_\lambda(\pi) + \frac{\tau}{2} H(\pi) + \frac{\tau}{2} h(\lambda),
\end{align}
where $\tau \geq 0$ is the regularization parameter and $L_\tau(\pi, \lambda)$ is the regularized Lagrangian. For a fixed $\lambda$, the objective of Problem \eqref{eq::RD-CRL} is equivalent to an unconstrained regularized MDP plus a regularization of the dual variable. Consider the composite regularized reward function $r_{\lambda, \tau}(s, a) \DefinedAs r(s, a) + \lambda g(s, a) - \frac{\tau}{2} h_a(a)$. The value function associated with the reward function $r_{\lambda, \tau}$ can be expressed as $V_{\lambda, \tau} (\pi) = V_\lambda (\pi) + \frac{\tau}{2} H(\pi)$. Then, we can reformulate the regularized Lagrangian as $L_\tau(\pi, \lambda) \DefinedAs V_{\lambda, \tau}(\pi) + \frac{\tau}{2} \lambda^2$. The global saddle points of the regularized Lagrangian $\Pi^\star_\tau \times \Lambda^\star_\tau$ are guaranteed to exist; see Lemma \ref{lm::saddle_point} in Appendix \ref{app::proofs}. Moreover, a global saddle point $(\pi_\tau^\star, \lambda_\tau^\star)$ satisfies
\begin{equation}
    \label{eq::sandwich_property}
        V_{\lambda_\tau^\star}(\pi) + \frac{\tau}{2} H(\pi) \leq V_{\lambda_\tau^\star}(\pi_\tau^\star) \leq V_{\lambda}(\pi_\tau^\star) + \frac{\tau}{2} \lambda^2
    \end{equation}
for all $(\pi, \lambda) \in \Pi \times \Lambda$. Hence, $(\pi_\tau^\star, \lambda_\tau^\star)$ is also a global saddle point of the original Lagrangian $L(\pi, \lambda)$ up to two $\tau$-terms. 

\subsection{Deterministic Policy Search Method }

We propose a deterministic policy gradient primal-dual (D-PGPD) method for finding a global saddle point $(\pi^\star_\tau, \lambda^\star_\tau)$ of $L_\tau(\pi, \lambda)$. In the primal update, as is customary in RL, we maximize the advantage function rather than the value function directly. Specifically, we use the regularized advantage $A^\pi_{\lambda, \tau} (s, a) \DefinedAs Q^\pi_{\lambda, \tau} (s, a) - V^\pi_{\lambda, \tau} (s) - \frac{\tau}{2} (\|a\|^2 - \|\pi(s)\|^2)$ associated with the regularized reward $r_{\lambda, \tau}$. The primal update~\eqref{eq::primal_update} performs a proximal-point-type ascent step that solves a quadratic-regularized maximization sub-problem, while the dual update~\eqref{eq::dual_update} performs a gradient descent step that solves a quadratic-regularized minimization sub-problem
\begin{subequations}
    \label{eq::D-PGPD}
    \begin{align}
    \label{eq::primal_update}
    \!\!\!\!  \!\!
    \pi_{t+1}(s) &\,=\, \argmax_{a \,\in\, A} \; A^{\pi_t}_{\lambda_t, \tau} (s, a) 
     - 
    \frac{1}{2 \eta} \| a - \pi_t(s) \|^2  \\
    \label{eq::dual_update}
     \!\!\!\! \!\!
    \lambda_{t+1} &\,=\, \argmin_{\lambda \,\in\, \Lambda} \; \lambda \,(V_g(\pi_t) + \tau \lambda_t) 
     + 
    \frac{1}{2 \eta} \| \lambda - \lambda_t \|^2,
    \end{align}
\end{subequations}
where $\eta$ is the step-size. D-PGPD is a single-time-scale algorithm, in the sense that the primal and the dual updates are computed concurrently in the same time-step. We remark that implementing D-PGPD is difficult in practice, and to make it tractable, we will leverage function approximation in Section~\ref{ss::adpgpd}. Before proceeding, we show that the primal-dual iterates~\eqref{eq::D-PGPD} converge in the last iterate to the set of global saddle points of the regularized Lagrangian $\Pi^\star_\tau \times \Lambda^\star_\tau$.

\subsection{Non-Asymptotic Convergence}
\label{ss::proposed_method_convergence}

Finding deterministic optimal policies is a computationally challenging problem \cite{feinberg2000constrained, dolgov2005stationary}. To render the problem tractable, we assume concavity and Lipschitz continuity of the regularized action value functions.

\begin{assumption}[Concavity]
    \label{as::concavity}
    The regularized state-action value function $Q^\pi_{\lambda, \tau} (s, a) - \tau_0 \| \pi_0(s) - a \|^2$ is concave in action $a$ for any policy $\pi_0$ and some  $\tau_0 \in [0, \tau)$. 
\end{assumption}

\begin{assumption}[Lipschitz continuity]
    \label{as::lipschitz}
    The action-value functions $Q_r^\pi (s, a)$, $Q_g^\pi (s, a)$, and $H^\pi (s, a) \DefinedAs \E_\pi \left[ \sum_{t=0}^\infty \gamma^t h_a(a_t) \,\vert\, s_0=s, a_0=a \right]$ are Lipschitz in action $a$ with Lipschitz constants $L_r$, $L_g$, and $L_h$, i.e., 
    \begin{align}
        \nonumber
        \|Q_r^\pi (s, a) - Q_r^\pi (s, a')\| \;&\leq\;
        L_r \| a - a' \| \\
        \nonumber
        \|Q_g^\pi (s, a) - Q_g^\pi (s, a')\| \;&\leq\; 
        L_g \| a - a' \| \\
        \nonumber
        \|H^\pi (s, a) - H^\pi (s, a')\| \;&\leq\;
        L_h \| a - a' \|,\; \forall \text{ $a$, $a' \in A$. }
    \end{align}
\end{assumption}

Assumption~\ref{as::concavity} states that there exists  a $\tau_0$-strongly concave regularizer that renders $Q^\pi_{\lambda, \tau}$ concave in the action $a$. When $\tau_0=0$, $Q^\pi_{\lambda, \tau}$ is concave in the action $a$. An example of this is Problem \eqref{eq::C-DS}, where the original reward and utility functions are concave and the transition dynamics are linear, leading to concavity of the associated regularized value function. Assumption \ref{as::lipschitz} implies Lipschitz continuity of the reward function and the probability transition kernel, which holds for several dynamics that can be expressed as a deterministic function of the actual state-action pair and some stochastic perturbation; see Appendix \ref{app::lipschitz_action_value} for a detailed explanation over the example introduced in Section \ref{ss::clq}.

To show convergence of D-PGPD, we introduce first two projection operators. The operator $\ccalP_{\Pi^\star_\tau}$ projects a policy into the non-empty set of optimal policies with state visitation distribution $d^\star_\rho$, and the operator $\ccalP_{\Lambda^\star_\tau}$ projects a Lagrangian multiplier onto the non-empty set of optimal Lagrangian multipliers $\Lambda^\star_\tau$. 
Then, we characterize the convergence of the primal-dual iterates of D-PGPD using a potential function 
\begin{equation}
\Phi_t \DefinedAs \frac{1}{2} \E_{d^\star_\rho} \left[ \| \ccalP_{\Pi^\star_\tau}(\pi_t(s)) - \pi_t(s) \|^2 \right] + \frac{\| \ccalP_{\Lambda^\star_\tau}(\lambda_t) - \lambda_t \|^2}{2(1 + \eta (\tau - \tau_0))}, \nonumber
\end{equation}
which measures the distance between a iteration pair $(\pi_t, \lambda_t)$ of D-PGPD and the set of global saddle points of the regularized Lagrangian $\Pi_\tau^\star \times \Lambda_\tau^\star$. Theorem~\ref{thm::convergence_exact} shows that as $t$ increases, the potential function $\Phi_t$ decreases linearly, up to an error; see Appendix \ref{app::convergence_exact} for the proof.

\begin{theorem}[Linear convergence]
\label{thm::convergence_exact}

Let Assumptions \ref{as::feasibility}--\ref{as::lipschitz} hold. For $\eta > 0$ and $\tau > \tau_0$, the primal-dual iterates \eqref{eq::D-PGPD} satisfy
\begin{equation}
    \Phi_{t+1} 
    \; \leq \; 
    {\rm e}^{- \beta_0 \,t} \,\Phi_1 
    \,+\,
    \beta_1 \,C_0^2,\;\;\text{where}
\end{equation}
\[
    \beta_0 \;\DefinedAs\; \frac{\eta( \tau - \tau_0)}{1 + \eta (\tau - \tau_0)} 
    \;\text{ and }\;
    \beta_1 \;\DefinedAs\; \frac{\eta (1 + \eta (\tau - \tau_0))}{\tau - \tau_0}
\]
\[
    C_0 
    \; \DefinedAs \; 
    L_r + \lambda_{\text{\normalfont max}} L_g + \tau L_h + \tau \sqrt{d_a} A_{\text{max}} + \frac{1+\frac{\tau}{\xi}}{1-\gamma}.
\]
\end{theorem}

Theorem \ref{thm::convergence_exact} states that the primal-dual updates of D-PGPD converge to a neighborhood of the set of global saddle points of the regularized Lagrangian $\Pi^\star_\tau \times \Lambda^\star_\tau$ in a linear rate. The size of the neighborhood depends polynomially on the parameters ($L_r$, $L_g$, $L_h$, $A_{\text{max}}$, $\tau$). When $\tau_0 = 0$, the regularization parameter $\tau$ can be arbitrarily small. Reducing the size of the convergence neighborhood can be achieved by selecting a sufficiently small $\eta$. However, a smaller the value of $\eta$ leads to slower convergence. To be more specific, for $\eta = \epsilon (\tau - \tau_0) C_0^{-2}$, the size of the convergence neighborhood is $O(\epsilon)$, and when $t \geq \Omega(\epsilon^{-1} \log (\epsilon^{-1}))$, the potential function $\Phi_t$ is $O(\epsilon)$ too, where $\Omega$ encapsulates some problem-dependent constants. After $O(\epsilon^{-1})$ iterations, the primal-dual iterates $(\pi_t, \lambda_t)$ of D-PGPD are $\epsilon$-close to the set $\Pi^\star_\tau \times \Lambda^\star_\tau$. 

The relationship between the solution to Problem \eqref{eq::P-CRL} and the solution to the regularized Problem \eqref{eq::RD-CRL} is given by Corollary~\ref{cor::near_optimality}; see its proof in Appendix \ref{app::corollary_reg_problem}.

\begin{corollary}[Near-optimality]
    \label{cor::near_optimality}
    Let Assumptions \ref{as::feasibility}--\ref{as::lipschitz} hold. If $\eta=O(\epsilon^4)$ and $\tau=O(\epsilon^2) + \tau_0$, and $t= \Omega(\epsilon^{-6} \text{\normalfont log}^2 \epsilon^{-1})$, then the primal-dual iterates \eqref{eq::D-PGPD} satisfy 
    \begin{align}
        \nonumber
        V_r(\pi^\star) - V_r(\pi_t) &\;\leq\;
        \epsilon - \tau_0 H(\pi^\star) \\\
        \nonumber
        V_g(\pi_t) & \;\geq\; 
        -\epsilon + \tau_0 H(\pi^\star)(\lambda_{\text{\normalfont max}}- \lambda^\star)^{-1} .    
    \end{align}
\end{corollary}

Corollary \ref{cor::near_optimality} highlights that the value functions corresponding to the policy iterates of D-PGPD can closely approximate the optimal solution to Problem~\eqref{eq::P-CRL}. Specifically, in problems where $\tau_0 = 0$, the final policy iterate of D-PGPD achieves $\epsilon$-optimality for Problem \eqref{eq::P-CRL} after $\Omega(\epsilon^{-6})$ iterations.  When $\tau_0 > 0$, D-PGDP converges to a saddle point of the original problem. However, the proximity of the final policy iterate to the optimal solution to Problem \eqref{eq::P-CRL} is proportional to $H(\pi^\star)$.

This work presents the first primal-dual convergence result for general constrained RL problems that directly work with \emph{deterministic} policies and \emph{continuous} state-action spaces. In the context of control, the convergence of different algorithms for solving constrained problems has been analyzed \cite{stathopoulos2016solving, zhang2020global, garg2020prescribed}. However, these analyses are limited to linear utility functions and box constraints. D-PGPD is a general algorithm that can be used for a broad range of transition dynamics and cost functions.

\section{Function Approximation}
\label{ss::adpgpd}

To instantiate D-PGPD~\eqref{eq::D-PGPD} with function approximation we begin by expanding the objective in~\eqref{eq::primal_update} and dropping the terms that do not depend on the action $a$,
\begin{equation}
    \nonumber
    Q^\pi_{\lambda, \tau}(s, a) + \frac{1}{\eta} \pi(s)^\Tr a - \left( \frac{\tau}{2} + \frac{1}{2\eta} \right) \| a \|^2.
\end{equation}

The usual function approximation approach~\cite{agarwal2021theory, ding2022convergence} is to introduce a parametric estimator of the policy $\pi$, and a compatible parametric estimator of the action value function $Q^\pi_{\lambda, \tau}$. Instead, we approximate the augmented action-value function $J^\pi(s, a) \DefinedAs Q^\pi_{\lambda, \tau}(s, a) + \frac{1}{\eta} \pi(s)^\Tr a$ using a linear estimator $\Tilde{J}_\theta(s,a)=\phi(s, a)^\Tr \theta$ over the basis $\phi$. At time $t$, we estimate $J^{\pi_t}(s, a)$ by computing the parameters $\theta_t$ via a mean-squared-error minimization
\begin{equation}
    \label{eq::policy_evaluation}
    \!\!\!\! \!
    \theta_t 
    \;\DefinedAs\; \argmin_\theta 
    \, \E_{(s, a) \,\sim\, \nu} 
    \left[\,
    \| \phi(s,a)^\Tr \theta - J^{\pi_t}(s, a)\|^2
    \,\right],
\end{equation}
where $\nu$ is a pre-selected state-action distribution. Problem \eqref{eq::policy_evaluation} can be easily addressed using, e.g., stochastic approximation. A subsequent policy $\pi_{t+1}$ results from a primal update based on $\Tilde{J}_{\theta_t}$. This leads to an approximated D-PGPD algorithm (AD-PGPD) that updates  $\pi_t$ and $\lambda_t$ via
\begin{subequations}
\label{eq::AD-PGPD}
\begin{align}
    \label{eq::approx_primal_update}
    \!\!\!\! \!
    \pi_{t+1}(s)  & \,=\, \argmax_{a \,\in\, A} \; \Tilde{J}_{\theta_t}(s,a) - \left( \frac{\tau}{2} + \frac{1}{2\eta} \right) \| a \|^2 \\
    \label{eq::approx_dual_update}
    \!\!\!\! \!\!
    \lambda_{t+1} &\,=\, \argmin_{\lambda \,\in\, \Lambda} \; \lambda(V_g(\pi_t) + \tau \lambda_t) + \frac{1}{2 \eta} \| \lambda - \lambda_t \|^2.
\end{align}
\end{subequations}

Solving the sub-problem \eqref{eq::approx_primal_update} requires inverting the gradient of \eqref{eq::approx_primal_update} with respect to \(a\), which is a challenge  when the MDP model is unknown or the value functions cannot be computed in closed form. This is the focus of Section~\ref{ss::adpgpd_sampled}.

\subsection{Non-Asymptotic Convergence}
\label{ss::adpgpd_convergence}

To ease the computational tractability of AD-PGPD, we assume concavity of the approximated augmented action-value function and bounded approximation error.

\begin{assumption}[Concavity of approximation]
    \label{as::concavity_approx}
    The function $\Tilde{J}_{\theta_t}(s, a) - \tau_0 \| \pi_0(s) - a \|^2$ is concave with respect to the action $a$ for some arbitrary policy $\pi_0$ and some $\tau_0 \in [0, \tau)$.
\end{assumption}

\begin{assumption}[Approximation error]
    \label{as::approximation_error}
    The approximation error $\delta_{\theta_t} (s, a)$ is bounded, $\E_{s \sim d^\star_\rho, a \sim \mathsf{u}} [ \| \delta_{\theta_t} (s, a) \| ] \leq \frac{\epsilon_{\text{\normalfont approx}}}{2 (2 A_{\text{max}})^{d_a}}$, where $\mathsf{u}$ is the uniform distribution and $\epsilon_{\text{\normalfont approx}} \geq 0$ is a positive error constant.
\end{assumption}

The concavity of $\Tilde{J}_{\theta_t}(s, a)$ with respect to $a$ depends on the selection of the basis function $\phi$. When the augmented action-value function $J^{\pi_t}$ is a concave quadratic function, it can be represented as a weighted linear combination of concave and quadratic basis functions. If these basis functions are known, $J^{\pi_t}$ can be perfectly approximated, i.e., $\epsilon_{\text{approx}}=0$. Furthermore, when $J^{\pi_t}$ is concave with respect to the action $a$, the regularization parameter $\tau$ can be arbitrarily small. Upon these conditions, convergence is guaranteed, as we formalize in the following result, whose proof is provided in Appendix \ref{app::convergence_inexact}.

\begin{theorem}[Linear convergence]
\label{thm::convergence_inexact}

Let Assumptions \ref{as::feasibility}, \ref{as::lipschitz}-- \ref{as::approximation_error} hold. If $\eta > 0$ and $\tau > \tau_0$, the primal-dual iterates \eqref{eq::AD-PGPD} satisfy
\begin{equation}
    \Phi_{t+1} 
    \;\leq\;
    {\rm e}^{-\beta_0 t} \Phi_1
    \, + \,
    \beta_1 C_0^2 
    \, + \,
    \beta_2 \epsilon_{\text{\normalfont approx}},
\end{equation}
where $\beta_0$, $\beta_1$, and $C_0$ are defined in Theorem \ref{thm::convergence_exact}, and
\[
    \beta_2 
    \;\DefinedAs\;
    \frac{1 + \eta (\tau - \tau_0)}{\tau - \tau_0}.
\]
\end{theorem}

Theorem \ref{thm::convergence_inexact} shows that the primal-dual iterates of AD-PGPD converge to a neighborhood of $\Pi^\star_\tau \times \Lambda^\star_\tau$ at a linear rate. The result is similar to Theorem \ref{thm::convergence_exact}, up to an approximation error $\epsilon_{\text{approx}}$. In fact, when $\epsilon_{\text{approx}}=0$, Theorem \ref{thm::convergence_inexact} is equivalent to Theorem \ref{thm::convergence_exact}. Linear models can achieve $\epsilon_{\text{approx}}=0$ when the augmented action-value function $J^{\pi_t}$ can be expressed as a linear combination of the selected basis function $\phi$, e.g. when $J^{\pi_t}$ is convex. When the error is small, the following result relates Problem~\eqref{eq::P-CRL} to the regularized Problem~\eqref{eq::RD-CRL}.

\begin{corollary}[Near-optimality of approximation]
    \label{cor::near_optimality_approx}
    Let Assumptions \ref{as::feasibility} and \ref{as::lipschitz}--\ref{as::approximation_error} hold. If $\eta=O(\epsilon^4)$, $\tau=O(\epsilon^2) + \tau_0$, $\epsilon_{\text{\normalfont approx}}=O(\epsilon^4)$, and $t= \Omega(\epsilon^{-6} \text{\normalfont log}^2 \epsilon^{-1})$, then the primal-dual iterates \eqref{eq::AD-PGPD} satisfy 
    \begin{align}
        \nonumber
        V_r(\pi^\star) - V_r(\pi_t) &\;\leq\;
        \epsilon - \tau_0 H(\pi^\star) \\\
        \nonumber
        V_g(\pi_t) & \;\geq\; 
        -\epsilon + \tau_0 H(\pi^\star)(\lambda_{\text{\normalfont max}}- \lambda^\star)^{-1} .    
    \end{align}
\end{corollary}

Corollary \ref{cor::near_optimality_approx} states that Corollary \ref{cor::near_optimality} extends to function approximation. When the approximation error is sufficiently small, i.e., $\epsilon_{\text{approx}}=O(\epsilon^4)$, the proof of Corollary~\ref{cor::near_optimality} holds (see Appendix \ref{app::corollary_reg_problem}), and the value functions corresponding to the policy iterates of AD-PGPD  closely approximate an optimal solution to Problem~\eqref{eq::P-CRL}. In fact, when $\tau_0=0$ and $\epsilon_{\text{approx}}$ are small, then the last policy iterate of AD-PGPD is an $\epsilon$-optimal solution to Problem~\eqref{eq::P-CRL} after $\Omega(\epsilon^{-6})$ iterations.

\section{Model-Free Algorithm}
\label{ss::adpgpd_sampled}

When the model of the MDP is unknown or when value-functions cannot be computed in closed form, we can leverage sample-based approaches to compute the primal and dual iterates of AD-PGPD. To that end, we assume access to a simulator of the MDP from where we can sample trajectories given a policy $\pi$. The sample-based algorithm requires modifying the policy evaluation step in \eqref{eq::policy_evaluation}, and the dual update in \eqref{eq::approx_dual_update}. For the former, in time-step $t$ for a given policy $\pi_t$, we have the following linear function approximation problem
\begin{equation}
\label{eq::policy_evaluation_approx}
    \!\!\!\!
    \min_{\theta,\, \| \theta \| \,\leq\, \theta_{\text{max}}} \; \; \E_{s, a \sim \nu} \left[ \| \phi(s_n, a_n)^\Tr \theta - \hat{J}^{\pi_t}(s_n, a_n) \|^2 \right],
\end{equation}
where the parameters $\theta$ are bounded, i.e., $\| \theta \| \leq \theta_{\text{max}}$, and $\phi$ is the basis function.  The approximated augmented value-function $\hat{J}^{\pi_t} \DefinedAs \hat{Q}_{\lambda, \tau}^{\pi_t}(s_n, a_n) + \frac{1}{\eta} \pi(s_n)^\Tr a_n$ is estimated from samples, which comes down to approximating  $\hat{Q}_{\lambda, \tau}^{\pi_t}(s_n, a_n)$. The dual update \eqref{eq::approx_dual_update} also requires the approximated value-function $\hat{V}_g({\pi_t})$ to be estimated. We detail how to estimate $\hat{V}_g({\pi_t})$ and $\hat{Q}_{\lambda, \tau}^{\pi_t}(s_n, a_n)$ via rollouts in Algorithms 1 and 2, which can be found in Appendix \ref{app::algorithms}. We use random horizon rollouts \cite{paternain2020stochastic, zhang2020global} to guarantee that the stochastic estimates of $\hat{Q}_{\lambda, \tau}^{\pi_t}$ and $\hat{V}_g(\pi_t)$ are unbiased. From \cite[Proposition 2]{paternain2020stochastic}, we have 
$Q^{\pi_t}_{\lambda, \tau}(s, a) = \E [ \hat{Q}_{\lambda, \tau}^{\pi_t}(s, a) \, | \,  s, a ]$ and $\;\; V_g(\pi_t) = \E [ \hat{V}_g^{\pi_t}(s) ]$, where the expectations $\E$ are taken over the randomness of drawing trajectories following $\pi_t$. We solve Problem \eqref{eq::policy_evaluation_approx} at time $t$ using projected stochastic gradient descent (SGD),
\begin{align}
    \nonumber
    g_t^{(n)} & \;\;=\;\; 2 \left( \phi(s_n, a_n)^\Tr \theta_t^{(n)} - \hat{J}^{\pi_t}(s_n, a_n) \right) \phi(s_n, a_n) \\
    \label{eq::sgd_update}
    \theta_t^{(n + 1)} & \;\;=\;\; \ccalP_{\| \theta \|\, \leq\, \theta_{\text{max}}} \left( \theta_t^{(n)} \,-\, \alpha_n \, g_t^{(n)} \right),
\end{align}
where $n \geq 0$ is the iteration index, $\alpha_n$ is the step-size, $g_t^{(n)}$ is the stochastic gradient of \eqref{eq::policy_evaluation_approx}, and $\ccalP_{\| \theta \| \leq \theta_{\text{max}}}$ is an operator that projects onto the domain ${\| \theta \| \leq \theta_{\text{max}}}$, which is convex and bounded. Each projected SGD update \eqref{eq::sgd_update} forms the estimate $\hat{\theta}_t$. We run $N$ projected SGD iterations and form the weighted average $\hat{\theta}_t \DefinedAs \frac{2}{N(N + 1)} \sum_{n=0}^{N -1}(n + 1)\hat{\theta}_t$, which is the estimation of the parameters $\theta_t$. Combining \eqref{eq::AD-PGPD}, the SGD rule in \eqref{eq::sgd_update}, and averaging techniques lead to a sample-based algorithm presented in Algorithm \ref{alg::sample_adpgpd}, in Appendix \ref{app::algorithms}.

The convergence analysis of Algorithm \ref{alg::sample_adpgpd} has to account for the estimation error induced by the sampling process. The error $\delta_{\hat{\theta}_t}(s, a)= \tilde{J}_{\hat{\theta}_t}(s, a) - J^{\pi_t}(s, a)$ can be decomposed as $\delta_{\hat{\theta}_t}(s, a) = \delta_{\hat{\theta}_t}(s, a) - \delta_{\theta_t}(s, a) + \delta_{\theta_t}(s, a)$. The bias error term $\delta_{\theta_t}(s, a)$ is similar to the approximation error of AD-PGPD and captures how good the model approximates the true augmented value function. The term $\delta_{\hat{\theta}_t}(s, a) - \delta_{\theta_t}(s, a)$ is a statistical error that reflects the error introduced by the sampling mechanism for a given state-action pair. To deal with the randomness of the projected SGD updates, we assume that the bias error and the feature basis are bounded. We also assume that the feature covariance matrix is positive definite, and that the sampling distribution $\nu$ and the optimal state visitation frequency $d_\rho^\star$ are uniformly equivalent.

\begin{assumption}[Bounded feature basis]
    \label{as::bounded_feature_basis}
    The feature function is bounded, i.e., $\| \phi(s, a) \| \leq 1$ for all $s \in S$ and $a \in A$.
\end{assumption}

\begin{assumption}[Positive covariance]
    \label{as::positive_covariance}
    The feature covariance matrix $\Sigma_\nu = \E_{s,a \sim \nu} [ \phi(s, a) \phi(s, a)^\Tr]$ is positive definite $ \Sigma_\nu \geq \kappa_0 I$ for the state-action distribution $\nu$.
\end{assumption}

\begin{assumption}[Bias error]
    \label{as::bias_error}
    The bias error $\delta_{\theta_t} (s, a)$ is bounded $ \E_{s \sim d^\star_\rho, a \sim \mathsf{u}} [ \| \delta_{\theta_t} (s, a) \| ]  \leq \frac{\epsilon_{\text{\normalfont bias}}}{2 (2 A_{\text{max}})^{d_a}}$, where $\mathsf{u}$ is the uniform distribution and $\epsilon_{\text{\normalfont bias}}$ is a positive error constant.
\end{assumption}

\begin{assumption}[Uniformly equivalence]
    \label{as::est_error}
    The state-action distribution induced by the state-visitation frequency $d_\rho^\star$ and the uniform distribution $\mathsf{u}$ is uniformly equivalent to the state-action distribution $\nu$, i.e. 
    \[
        \frac{d_\rho^\star(s) \mathsf{u}(a)}{\nu(s, a)} 
        \; \leq \;
        L_\nu
        \; \text{ for all $(s, a) \in S \times A$.}
    \]
\end{assumption}

Assumption \ref{as::bounded_feature_basis} holds without loss of generality, as the basis functions are a design choice.  Assumption \ref{as::positive_covariance} ensures that the minimizer of \eqref{eq::policy_evaluation_approx} is unique, since $\Sigma_\nu \geq \kappa_0 I$ for some $\kappa_0 > 0$. Assumption \ref{as::bias_error} states that the selected model achieves a bounded error, and Assumption \ref{as::est_error} ensures that the sampling distribution $\nu$ is sufficiently representative of the optimal state visitation frequency $d_\rho^\star$. We characterize the convergence using the expected potential function $\E[\Phi_t]$, where the expectation is taken over the randomness of $\theta_t^{(n)}$. We have the following corollary; see the proof in Appendix \ref{app::convergence_sampled}.

\begin{corollary}[Linear convergence]
\label{cor::convergence_sampled}

Let Assumptions \ref{as::feasibility}, \ref{as::lipschitz}, \ref{as::concavity_approx}, and  \ref{as::bounded_feature_basis}--\ref{as::est_error} hold. Then, the sample-based AD-PGPD in Algorithm \ref{alg::sample_adpgpd} satisfies 
\begin{equation}
    \E [\Phi_{t+1}] \leq e^{- \beta_0 t} \E [\Phi_1] + \beta_1 C_0^2 + \beta_2 \left( \frac{ C_1^2}{\eta^2  (N + 1)} + \epsilon_{\text{\normalfont bias}} \right),
\end{equation}
where $\beta_0$, $\beta_1$, $\beta_2$, and $C_0$ are given in Theorems \ref{thm::convergence_exact} and \ref{thm::convergence_inexact}, and
\begin{equation}
    \nonumber
    C_1\! \DefinedAs \! \sqrt{2^{d_a + 5} A_{\text{max}}^{d_a}  L_\nu} \big( \theta_{\text{max}} + 2(1 - \gamma)^{-2}\xi^{-1} + d_a A_{\text{max}}^2 \big) \kappa_0^{-1} . 
\end{equation}
\end{corollary}

Corollary \ref{cor::convergence_sampled} is analogous to Theorem \ref{thm::convergence_inexact}, but accounting for the use of sample-based estimates. The sampling effect appears as the number $N$ of projected SGD steps performed at each time-step $t$. Corollary \ref{cor::near_optimality_approx} holds when the bias error $\epsilon_{\text{bias}} = O(\epsilon^4)$ and the estimation error $C_1^2 \eta^{-2} (N + 1)^{-1} = O(\epsilon^4)$. As $\eta=O(\epsilon^4)$, the latter holds when $N = \Omega(\epsilon^{-12})$, where $\Omega$ encapsulates problem-dependent constants. Therefore, the number of rollouts required to output an $\epsilon$-optimal policy is $tN= \Omega(\epsilon^{-18})$. While this result suggests potential improvement, it stands as the first sample-complexity result in the context of constrained MDPs with continuous spaces.

\section{Computational Experiments}

We test D-PGPD on constrained robot navigation and fluid control problems (Figure \ref{fig::fig_1}). See Appendix \ref{app::experiments} for more details.

\begin{figure}[t]
    \centering
    \includegraphics[width=14cm]{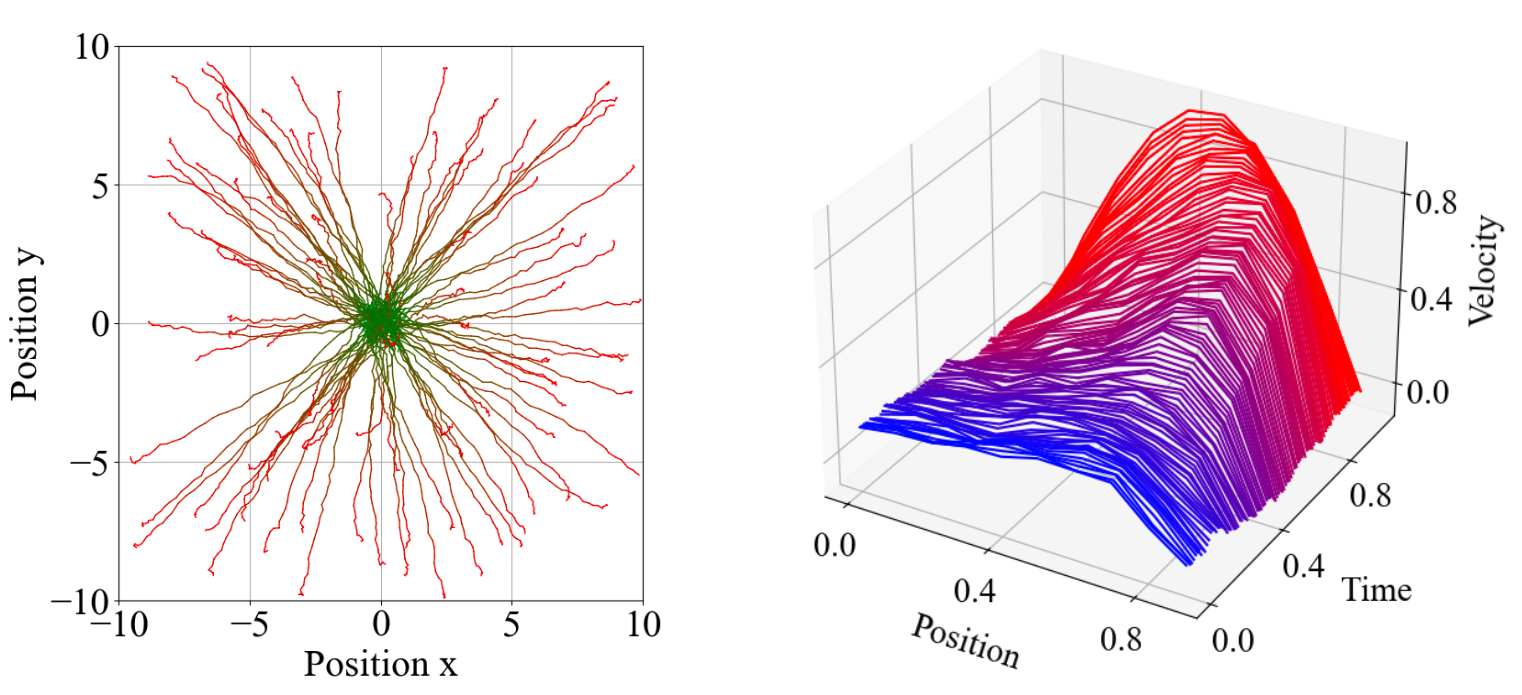}
    \caption{Navigation trajectories of an agent  (Left) and velocity profile of the fluid over time (Right).}
    \label{fig::fig_1}
\end{figure}

\noindent \textbf{Navigation Problem.}  An agent moves in a horizontal plane following some linearized dynamics with zero-mean Gaussian noise~\cite{shimizu2020motion, ma2022alternating}. We aim to drive the agent to the origin while constraining its velocity. When the dynamics are known and the reward function linearly weights quadratic penalties on position and action, this problem is an instance of the constrained linear regulation problem \cite{scokaert1998constrained}, which has closed-form solution. Hence, we can directly apply D-PGPD~\eqref{eq::D-PGPD} and AD-PGPD~\eqref{eq::AD-PGPD} (See Appendix \ref{app::experiments_navigation}). However, we consider the dynamics to be unknown, and we leverage our sample-based implementation of AD-PGPD. Furthermore, we use absolute value penalties instead of quadratic ones, as the latter can result in unstable behavior in sample-based scenarios \cite{engel2014line}. Conventional methods do not solve this problem straightforwardly. We compare our sample-based AD-PGPD with PGDual, a dual method with linear function approximation \cite{zhao2021primal, brunke2022safe}. Figure~\ref{fig::fig_2} shows the value functions of the policy iterates generated by AD-PGPD and PGDual over $40,000$ iterations. The oscillations of AD-PGPD are damped over time, and it converges to a feasible solution with low variance in reward and utility, indicating a near-deterministic behavior without constraint violation. In contrast, PGDual exhibits large variance, indicating that the resultant policy violates the constraint. Nevertheless, the final primal return performance of PGDual is similar to that of AD-PGPD on average.

\noindent \textbf{Fluid Velocity Control.} We apply D-PGPD~\eqref{eq::D-PGPD} to the control of the velocity of an incompressible Newtonian fluid described by the one-dimensional Burgers' equation \cite{baker2000nonlinear},  a non-linear stochastic control problem. The velocity profile of the fluid $z$ varies in a one-dimensional space $x \in [0, 1]$ and time $t \in [0, 1]$, and the goal is to drive the velocity of the fluid towards zero via the control action $a$, e.g., injection of polymers. By discretizing Burgers' equation, we have a non-linear system $s_{t+1} = B_0 s_t + B_1 a_t + B_2 s_t^2 + \omega_t$, where $s_t \in \reals^d$ is the state, $s_t^2$ is the element-wise squared state vector, \( a_t \in \reals^{d} \) is the control input, and $B_0$, $B_1$, $B_2 \in \reals^{d \times d}$ are matrices representing the discretized spatial operators and non-linear terms \cite{borggaard2020quadratic}. The details can be found in Appendix \ref{app::experiments_fluid}. We consider a reward function that penalizes the state quadratically, and a budget constraint that limits the total control action. We compare our sample-based AD-PGPD with PGDual. Figure~\ref{fig::fig_3} shows the value functions of the policy iterates generated by AD-PGPD and PGDual over $10,000$ iterations. The results are consistent with those of the navigation problem. The AD-PGPD algorithm successfully mitigates oscillations and converges to a feasible solution with low return variance. In contrast, although PGDual achieves similar objective value, it does not dampens oscillations, as indicated by the variance of the solution. This implies that PGDual violates the constraint in the last iterate.

\begin{figure}[t]
    \centering
    \includegraphics[width=14cm]{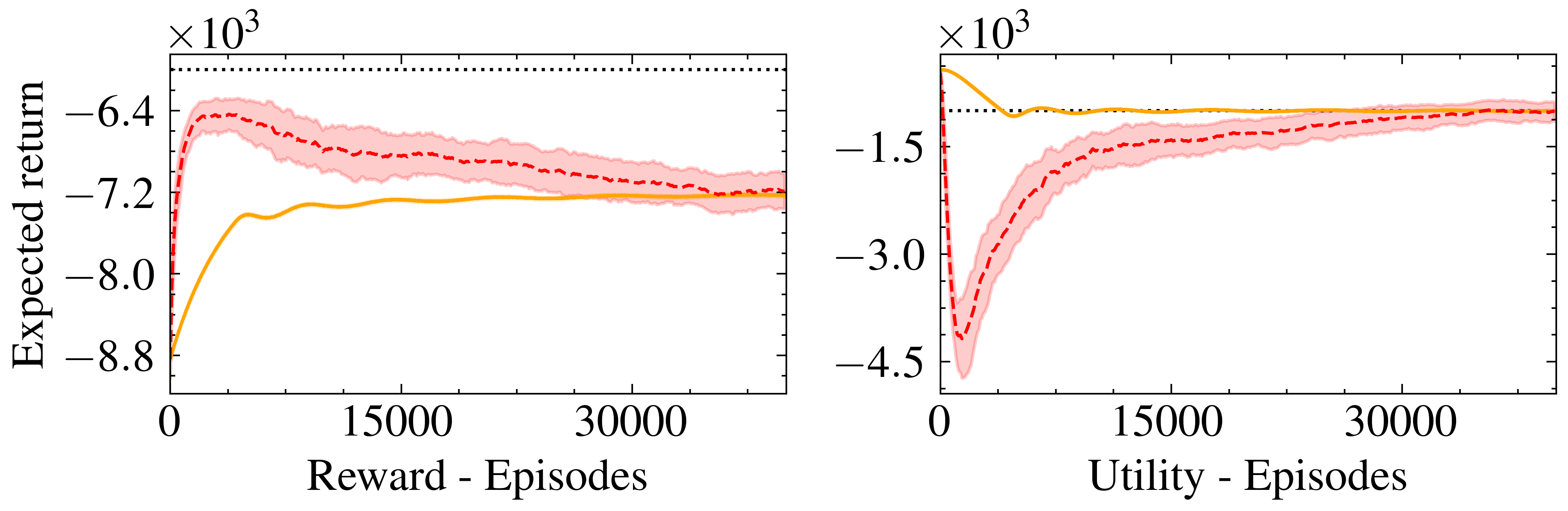}
    \caption{Avg. reward/utility value functions of AD-PGPD (\yellowline) and PGDual (\reddashedline ) iterates  in the navigation  problem.}
    \label{fig::fig_2}
\end{figure}

\begin{figure}[t]
    \centering
    \includegraphics[width=14cm]{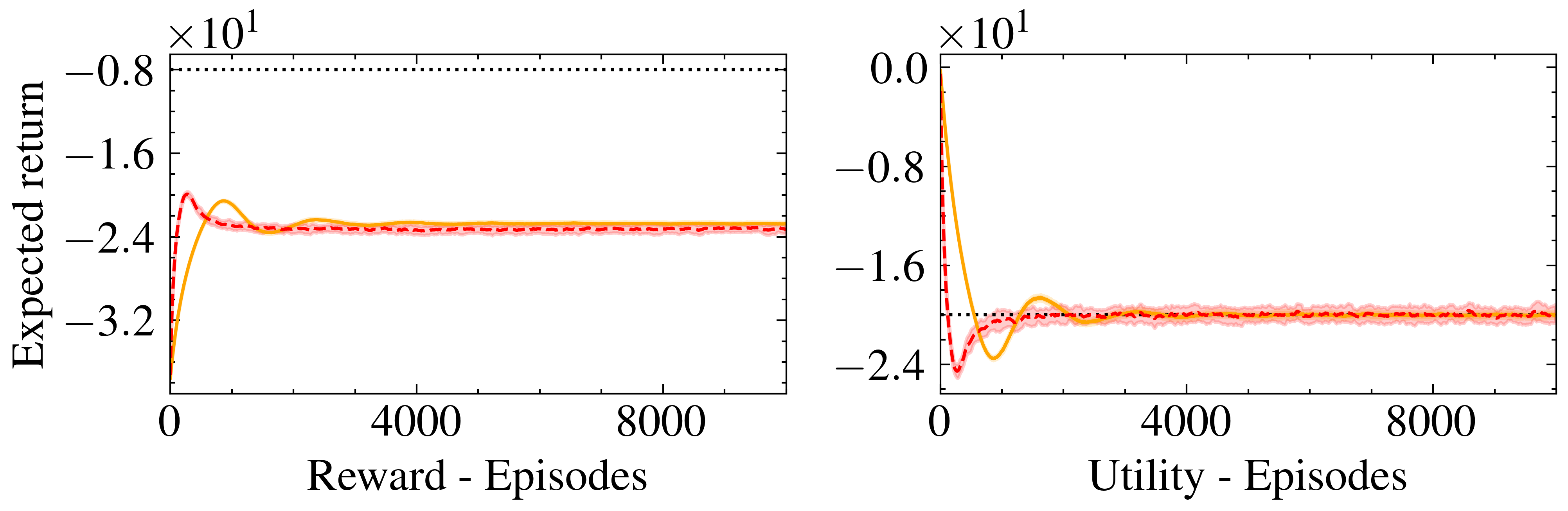}
    \caption{Avg. reward/utility value functions of AD-PGPD (\yellowline) and PGDual (\reddashedline ) iterates  in a fluid velocity control.}
    \label{fig::fig_3}
\end{figure}

\section{Concluding Remarks}

We have presented a deterministic policy gradient primal-dual method for continuous state-action constrained MDPs with non-asymptotic convergence guarantees. We have leveraged function approximation to make the implementation practical and developed a sample-based algorithm. Furthermore, we have shown the effectiveness of the proposed method in navigation and non-linear fluid constrained control problems. Our work opens new avenues for constrained MDPs with continuous state-action spaces, such as (i) minimal assumption on value functions; (ii) online exploration; (iii) optimal sample complexity; and (iv) general function approximation.

\section*{Acknowledgments}
We thank the anonymous reviewers for their insightful comments.
This work has been partially supported by the Spanish NSF (AEI/10.13039 /501100011033) grants TED2021-130347B-I00 and PID2022-136887NB-I00, and the Community of Madrid via the Ellis Madrid Unit and grant TEC-2024/COM-89.

\bibliography{arxiv}
\bibliographystyle{plainnat}

\newpage
\appendix

\section{Convexity of Value Images}

\label{app::convexity_value_image}

Strong duality holds for Problem \eqref{eq::P-CRL}, as established in Theorem \ref{thm::zero_duality_gap}, despite the non-convexity of the value functions $V_r(\pi)$ and $V_g(\pi)$ with respect to the policy $\pi$. This section aims to shed some light into the reasons behind this phenomenon. First, recall that the policy class $\Pi$ considered in this paper is restricted to deterministic policies. Importantly, this set is not necessarily convex. For any $\alpha \in [0, 1]$ and two deterministic policies $\pi, \pi' \in \Pi$, their convex combination
\begin{equation}
\nonumber
\pi_\alpha \;=\; \alpha \pi + (1 - \alpha) \pi',
\end{equation}
does not generally yield to a deterministic policy, i.e., $\pi_\alpha \notin \Pi$. Furthermore, consider the vector value function
\begin{equation}
\nonumber
V(\pi) \;\DefinedAs\; \begin{bmatrix}
V_r(\pi) \\[0.2cm]
V_g(\pi)
\end{bmatrix}
\end{equation}
associated with a given policy $\pi$. The value function $V(\pi)$ is non-convex in $\pi$. However, let us now focus on the set of all attainable vector value functions corresponding to deterministic policies, which defines the deterministic value image:
\begin{equation}
\nonumber
\mathcal{V}_D \;\DefinedAs\; \{ V(\pi) \;\vert\; \pi \in \Pi \}.
\end{equation}
Under the non-atomicity assumption (see Lemma \ref{lm::convex_span_lemma}), the set $\mathcal{V}_D$ is convex. This convexity implies that there exists a policy $\pi_\alpha \in \Pi$ such that
\begin{equation}
\nonumber
V(\pi_\alpha) \;=\; \alpha V(\pi) + (1-\alpha)V(\pi'),
\end{equation}
even though $\pi_\alpha$ is not a convex combination of $\pi$ and $\pi'$, and the vector value function $V$ remains non-convex.

Lastly, we consider the value image $\mathcal{V}_T$ for all policies. Under the non-atomicity assumption, these two sets, $\mathcal{V}_D$ and $\mathcal{V}_T$, are equivalent (see Lemma \ref{lm::sufficiency_deterministic_lemma}). Consequently, the optimal policy is contained within $\mathcal{V}_D$. Therefore, restricting the search space to deterministic policies is justified in the context of non-atomic MDPs. The convexity of the deterministic value image and its equivalence with the value image for all policies are illustrated in Figure~\ref{fig::fig_11}.

\begin{figure}[h]
    \centering
    \includegraphics[width=7cm]{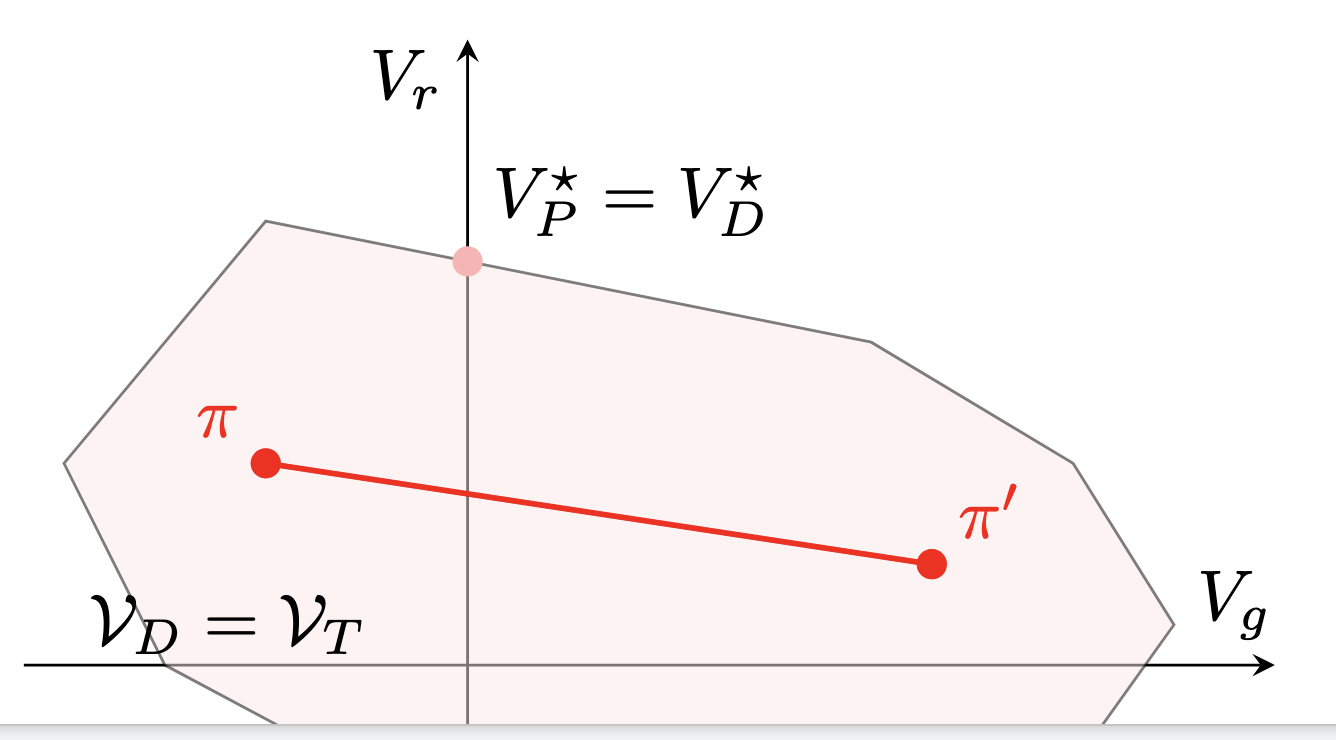}
    \caption{The deterministic value image $\ccalV_D$ is convex and equivalent to the the value image for all policies: $\ccalV_T$. Furthermore, constrained RL has zero duality gap in the deterministic policy space, i.e., $V_P^\star=V_D^\star$.}
    \label{fig::fig_11}
\end{figure}

\section{Supporting Lemmas}

\label{app::supporting_lemmas}

\begin{lemma}[Discounted and uniformly absorbing MDP equivalence]
    \label{lm::unif_mdp_lemma}
    A non-atomic discounted MDP can be equivalently represented as a non-atomic uniformly absorbing MDP.
\end{lemma}

\begin{proof}
    See~\citet[Lemma 3.12]{feinberg2019sufficiency}.
\end{proof}

\begin{lemma}[Convexity of the deterministic value image]
    \label{lm::convex_span_lemma}
    Consider the vector value function $V(\pi)\DefinedAs[V_r(\pi), V_g(\pi)]^\Tr$ for an arbitrary deterministic policy $\pi$. Let the span of value functions associated with the class of deterministic policies $\Pi$ be defined as
    $\ccalV_D \DefinedAs \{ V(\pi) : \pi \in \Pi \}$. 
    
    Then, the set $\ccalV_D$ is convex for a non-atomic uniformly absorbing MDP.
\end{lemma}

\begin{proof}
    See~\citet[Corollary 3.10]{feinberg2019sufficiency}.
\end{proof}

\begin{lemma}[Sufficiency of deterministic policies]
    \label{lm::sufficiency_deterministic_lemma}

    Consider the vector value function $V(\pi)\DefinedAs[V_r(\pi), V_g(\pi)]^\Tr$ for an arbitrary policy $\pi$. Let the span of value functions associated with the class of deterministic policies $\Pi$ be defined as
    $\ccalV_D \DefinedAs \{ V(\pi) : \pi \in \Pi \}$. Similarly, let the span of value functions associated with the general class of policies $\Pi_T$ be defined as  $\ccalV_T \DefinedAs \{ V(\pi)$ : $\pi \in \Pi_T \}$. 
    
    Then, for a non-atomic uniformly absorbing MDP, it holds that $\ccalV_D = \ccalV_T$.
\end{lemma}

\begin{proof}
    See~\citet[Theorem 3.8]{feinberg2019sufficiency}.
\end{proof}

\begin{lemma}[Quadratic growth Lemma]
    \label{lm::qg_lemma}
    Let $f$ be a $\mu$-strongly-concave function and let $x^\star$ denote its maximizer. Then
    \begin{equation}
        f(x) \;\leq \;
        f(x^\star) - \mu || x - x^\star||^2.
    \end{equation}
\end{lemma}

\begin{proof}
    See~\citet[Lemma 3.5]{lan2020first}.
\end{proof}

\begin{lemma}[Standard descent Lemma]
    \label{lm::descent_lemma}
    Let $f$ be a convex differentiable function, and let $x^\star$ denote the minimizer of $f$. Consider the following sequence
    \begin{equation}
        \nonumber
        x_{t+1} \;=\; x_t - \eta \nabla f(x_t),
    \end{equation}
    where $\eta$ is an step-size. Then, the following bound holds
    \begin{align}
        (x_t - x^\star) \nabla f(x_t) 
        \; \leq \; 
        \frac{1}{2 \eta} \left( (x_t - x^\star)^2 - (x_{t+1} - x^\star)^2\right) + \eta^2 \| \nabla f(x_t) \|^2.
    \end{align}
\end{lemma}

\begin{proof}
    See~\citet[Theorem 1]{zinkevich2003online}
\end{proof}

\begin{lemma}[Performance difference Lemma]
    \label{lm::pd_lemma}
    Consider the following regularized reward function $r_\tau(s_t, a_t) - \frac{\tau}{2} \| a_t \|^2$, and the associated value functions
    \begin{align}
        \nonumber
        V^\pi (s) &\;=\; \E_\pi \left[ \sum_{t=0}^\infty \gamma^t r_\tau(s_t, a_t) \; | \; s_0 = s \right] \\
        \nonumber
        Q^\pi (s, a) &\;=\; \E_\pi \left[ \sum_{t=0}^\infty \gamma^t r_\tau(s_t, a_t) \; | \; s_0 = s, a_0=a \right]. 
    \end{align}
    Consider the regularized advantage function $A^\pi (s, a)= Q^\pi(s, a) - V^\pi(s) - \frac{\tau}{2} \| a \|^2 + \frac{\tau}{2}\pi(s)^2$. Let $\pi$ and $\pi'$ be two feasible policies and $\rho$ be the initial state distribution. Then
    \begin{equation}
        V(\pi') - V(\pi) \;=\; \frac{1}{1 - \gamma} \E_{s \sim d_\rho^{\pi'}} \left[ A^\pi(s, \pi'(s)) \right].
    \end{equation}
\end{lemma}

\begin{proof}
    Leveraging the performance difference lemma of the regularized advantage (see \cite[Lemma 2.1]{lan2022policy}), it follows that:
    \begin{align}
        \nonumber
        V(\pi') - V(\pi) &\;=\; \E_{s_0 \sim \rho} [V^{\pi'}(s_0) - V^\pi(s_0)] \\
        \nonumber
        &\;=\; \frac{1}{1 - \gamma} \E_{s_0 \sim \rho}[\E_{s \sim d_{s_0}^{\pi'}}[A^\pi (s, \pi'(s))]] \\
        \nonumber
        &\;=\; \frac{1}{1 - \gamma} \E_{s \sim d_\rho^{\pi'} } [A^\pi (s, \pi'(s))].
    \end{align}
\end{proof}

\begin{lemma}[Fenchel-Moreau]
    \label{lm::fenchel_moreau}
    If Slater's condition holds for \eqref{eq::P-CRL} and its perturbation function $P(\delta)$ is concave on $\delta$, then \eqref{eq::P-CRL} has zero duality gap.
\end{lemma}

\begin{proof}
    See \citet[Theorem 5.1.4, Proposition 5.3.2]{auslender2006asymptotic}
\end{proof}

\begin{lemma}[Boundedness of $\lambda^\star$]
    \label{lm::bound_lambda}
    Let Assumption \ref{as::feasibility} hold. Then, 
    \begin{equation}
        \nonumber
        0 \;\leq\; \lambda^\star \;\leq \;\frac{V_r(\pi^\star) - V_r(\Bar{\pi})}{\xi}.
    \end{equation}
\end{lemma}

\begin{proof}
    See~\citet[Lemma 3]{ding2022convergence}. Let $\Lambda_a \DefinedAs \{\lambda \ge 0 \mid D(\lambda) \le a\}$ be a sublevel set of the dual function for $a \in \mathbb{R}$. Thanks to Assumption \ref{as::feasibility}, for any $\lambda \in \Lambda_a$    
    \[
        a 
        \;\ge\; 
        D(\lambda) 
        \;\ge\; 
        \left( V_r(\bar{\pi}) + \lambda V_g(\bar{\pi}) \right) 
        \;\ge\; V_r(\bar{\pi}) + \lambda \xi,
    \]
    where $\bar{\pi}$ is a Slater point. Thus, 
    \[
        \lambda \;\le\; \frac{a - V_r(\bar{\pi})}{\xi}.
    \]
    If we take $a = V_r(\pi^\star) = V_D^{\lambda^\star}$, then $\Lambda_a = \Lambda^\star$ which concludes the proof.
\end{proof}

\begin{lemma}[Danskin's Theorem]
    \label{lm::danskins}
    Consider the function
    \begin{align}
        F(x) \DefinedAs \text{\normalfont sup}_{y \,\in\, Y}f(x,y),
    \end{align}
    where $f$: $\reals^n\times Y\to \reals \cup \{-\infty,+\infty\}$. If the following conditions are satisfied, 
    \begin{enumerate}
        \item[\text{\normalfont(i)}] The function $f(\cdot,y)$ is convex for all $y\in Y$;
        \item[\text{\normalfont(ii)}] The function $f(x,\cdot)$ is upper semicontinuous for all $x$ in a certain neighborhood of a point $x_0$;
        \item[\text{\normalfont(iii)}] The set $Y\subset \reals ^m$ is compact;
    \end{enumerate}
    then, $F$ is a convex function w.r.t. $x$, and 
    \begin{align}
    \partial F(x_0)
    \;=\;
    \text{\normalfont conv}\left( \underset{y \,\in\, \hat Y(x_0)}{\cup} \partial_x f(x_0,y) \right)
    \end{align}
    where $\partial_x f(x_0,y)$ denotes the subdifferential of the function $f(\cdot,y)$ at $x_0$, and $\hat Y(x_0)$ denote the set of maximizing points of $F$ at $x_0$.
\end{lemma}

\begin{proof}
    See \citet{ruszczynski2006nonlinear}[Theorem 2.87].
\end{proof}

\begin{lemma}
    \label{lm::lipschitz_lemma}
    (Lipschitz bound Lemma). Let Assumption \ref{as::lipschitz} hold, and let $A_{\lambda, \tau}^\pi (s, a)= Q_{\lambda, \tau}^\pi(s, a) - V_{\lambda, \tau}^\pi(s) - \frac{\tau}{2} \| a \|^2 + \frac{\tau}{2}\| \pi(s) \|^2$ and $\eta > 0$. Then, it follows that 
    \begin{equation}
        A_{\lambda, \tau}^{\pi}(s, a) - \frac{1}{2 \eta}\|\pi(s) - a\|^2 
        \;\leq\;
        \frac{\eta}{2} C_P^2,
    \end{equation}
    where $C_P \DefinedAs L_r + \lambda_{\text{\normalfont max}} L_g + \tau L_h + \tau \sqrt{d_a} A_{\text{\normalfont max}}$.
\end{lemma}

\begin{proof}
    The proof begins by manipulating the advantage using the fact that $V_{\lambda, \tau}^\pi(s) = Q_{\lambda, \tau}^\pi(s, \pi(s))$ for deterministic policies
    \begin{align}
        \nonumber
        &A_{\lambda, \tau}^\pi(s, a) \;=\; \underbrace{Q_{\lambda, \tau}^\pi(s, a) - Q_{\lambda, \tau}^\pi(s, \pi(s))}_{\text{\normalfont(i)}}
        - \underbrace{\frac{\tau}{2} \| a \|^2 + \frac{\tau}{2}\| \pi(s) \|^2}_{\text{\normalfont(ii)}}.
    \end{align}
    Now, we manipulate (i) leveraging the Lipschitz continuity assumption \ref{as::lipschitz}. 
    \begin{align}
        \label{eq::q_bound}
        &Q_{\lambda, \tau}^\pi(s, a) - Q_{\lambda, \tau}^\pi(s, \pi(s)) \\
        \nonumber
        &=\; \left( Q_r^\pi(s, a) - Q_r^\pi(s, \pi(s)) \right)  + \lambda \left( Q_g^\pi(s, a) - Q_g^\pi(s, \pi(s)) \right)  + \frac{\tau}{2}\left( H^\pi(s, a) - H^\pi(s, \pi(s)) \right) \\
        \nonumber
        &\leq\; \left( L_r + \lambda_{\text{max}} L_g + \tau L_h \right) \| \pi(s) - a \|.
    \end{align}
    For the term (ii), we have first that $\| \pi(s) \|^2 - \| a \|^2$ is a difference of quadratics, so we can use it in combination with the triangular inequality to show that
    \begin{align}
        \nonumber
        \frac{\tau}{2} \| a \|^2 + \frac{\tau}{2}\| \pi(s) \|^2 & \;=\; \frac{\tau}{2} (\pi(s) - a)^\Tr (\pi(s) + a) \\
        \nonumber
        & \;\leq\; \frac{\tau}{2} \|\pi(s) - a\| \|\pi(s) + a\|.
    \end{align}
    Then, using the boundedness of the actions as $\|\pi(s) + a\| \leq 2 \sqrt{d_a} A_{\text{max}}$, where $d_a$ is the dimensionality of $A$, we have that
    \begin{align}
        \label{eq::a_bound}
        \frac{\tau}{2} \|\pi(s) - a\|\times \|\pi(s) + a\| 
        \;\leq\;  
        \tau \sqrt{d_a} A_{\text{max}} \|\pi(s) - a\|.
    \end{align}
    Now, combining \eqref{eq::q_bound} and \eqref{eq::a_bound} we have that
    \begin{align}
        \label{eq::adv_bound}
        &A_{\lambda, \tau}^{\pi}(s, a) 
        \;\leq\;
        \left( L_r + \lambda_{\text{max}} L_g + \tau L_h + \tau \sqrt{d_a} A_{\text{max}} \right) \| \pi(s) - a \|.
    \end{align}

    Now, we define $C_P \DefinedAs L_r + \lambda_{\text{max}} L_g + \tau L_h + \tau \sqrt{d_a} A_{\text{max}}$ and using \eqref{eq::adv_bound} we have that
    \begin{align}
        \nonumber
        &A_{\lambda, \tau}^{\pi}(s, a) - \frac{1}{2 \eta}\|\pi(s) - a\|^2 
        \; \leq \; 
        C_P \| \pi(s) - a \| - \frac{1}{2 \eta}\|\pi(s) - a\|^2.
    \end{align}
    We add and substract the term $\frac{\eta}{2} C_P^2$, to then complete the squares to show that
    \begin{align}
        \nonumber
        & C_P \| \pi(s) - a \| - \frac{1}{2 \eta}\|\pi(s) - a\|^2 \\
        \nonumber
        &= \; C_P \| \pi(s) - a \| - \frac{1}{2 \eta}\|\pi(s) - a\|^2 - \frac{\eta}{2} C_P^2 + \frac{\eta}{2} C_P^2 \\
        \nonumber
        &= \; \frac{\eta}{2} C_P^2 - \left( \frac{1}{\sqrt{2\eta}} \| \pi(s) - a \| - \frac{\sqrt{2 \eta}}{2} C_P \right)^2.
    \end{align}
    Therefore, dropping the negative quadratic term we have that
    \begin{align}
        \nonumber
        A_{\lambda, \tau}^{\pi}(s, a) - \frac{1}{2 \eta}\|\pi(s) - a\|^2 
        &\;=\; \frac{\eta}{2} C_P^2 - \left( \frac{1}{\sqrt{2\eta}} \| \pi(s) - a \| - \frac{\sqrt{2 \eta}}{2} C_P \right)^2 \\
        \nonumber
        &\;\leq\; \frac{\eta}{2} C_P^2
    \end{align}
    \noindent which concludes the proof.
\end{proof}

\begin{lemma}
    \label{lm::error_bound}
    (Error bound lemma). For an arbitrary deterministic policy $\pi$, the expected approximation error $\delta(s, \pi(s))$ satisfies
    \begin{equation}
        \E_{s \sim d_\rho^\star} \left[ \| \delta(s, \pi(s))\|\right] 
        \;\leq\;
        (2 A_{\text{\normalfont max}})^{d_a}  \E_{s \sim d_\rho^\star, a \sim \mathsf{u}} \left[ \| \delta(s, a)\|\right].
    \end{equation}

    \noindent where $d_\rho^\star$ denotes the optimal state-visitation frequency, and $\mathsf{u}$ denotes the uniform distribution.
\end{lemma}

\begin{proof}

    Consider the multi-dimensional uniform distribution over the bounded action space $A \in [-A_{\text{max}}, A_{\text{max}}]^{d_a}$. The probability density is given by the expression
    \begin{equation}
        \label{eq::uniform_pdf}
        u(a) 
        \;\DefinedAs \;
        \begin{cases}
            \frac{1}{(2 A_{\text{max}})^{d_a}} &\text{if} \;\; a \in A\\
            0 &\text{otherwise}.
        \end{cases}
    \end{equation}

    The derivation of the bound of the expectation error begins by multiplying and dividing by the uniform distribution as
    \begin{align}
        \nonumber
        &\E_{s \sim d_\rho^\star} \left[ \| \delta(s, \pi(s))\|\right] 
        \;=\; 
        \E_{s \sim d_\rho^\star} \left[ \frac{u(\pi(s))}{u(\pi(s))} \| \delta(s, \pi(s)) \| \right].
    \end{align}
    Then, we have that for a generic measurable set $X$ and an element $x_0 \in X$, the relation $f(x_0)g(x_0) \leq \int_X f(x) g(x) dx $ holds for a generic probability density function $f$ if the image of the random variable $g$ is positive. Thus, as $\| \delta(s, \pi(s)) \|$ is a positive random variable, we have that
    \begin{align}
        \nonumber
        &\E_{s \sim d_\rho^\star} \left[ \frac{u(\pi(s))}{u(\pi(s))} \| \delta(s, \pi(s)) \| \right] \;\leq\; 
        \E_{s \sim d_\rho^\star, a \sim u} \left[ \frac{1}{u(\pi(s))} \| \delta(s, a) \| \right].
    \end{align}
    Finally, using the definition of the uniform density in \eqref{eq::uniform_pdf} we conclude showing that
    \begin{align}
        \nonumber
        \E_{s \sim d_\rho^\star} \left[ \| \delta(s, \pi(s))\|\right] 
        &\;\leq\; \E_{s \sim d_\rho^\star, a \sim u} \left[ \frac{1}{u(\pi(s))} \| \delta(s, a) \| \right] \\
        \nonumber
        &\;=\; (2 A_{\text{max}})^{d_a} \E_{s \sim d_\rho^\star, a \sim u} \left[  \| \delta(s, a) \| \right].
    \end{align}
\end{proof}

\section{Proofs}

\label{app::proofs}

\subsection{Proof of Existence of Global Saddle Points}
\begin{lemma}[Existence of global saddle points]
    \label{lm::saddle_point}

    There exists a primal-dual pair $(\pi_\tau^\star, \lambda_\tau^\star) \in \Pi \times \Lambda$ such that $L_\tau(\pi_\tau^\star, \lambda) \geq L_\tau(\pi_\tau^\star, \lambda_\tau^\star) \geq L_\tau(\pi, \lambda_\tau^\star)$. Furthermore, the following property holds for all $(\pi, \lambda) \in \Pi \times \Lambda$,
    \begin{equation}
        V_{\lambda_\tau^\star, \tau}(\pi) + \frac{\tau}{2} H(\pi) 
        \;\leq\; V_{\lambda_\tau^\star, \tau}(\pi_\tau^\star) 
        \;\leq\;
        V_{\lambda, \tau}(\pi_\tau^\star) + \frac{\tau}{2} \lambda^2.
    \end{equation}
\end{lemma}

\begin{proof}
    We know that the Lagrangian $L(\pi, \lambda)$ has global saddle points due to strong duality; see Theorem \ref{thm::zero_duality_gap}. We want to show that the regularized Lagrangian $L_\tau(\pi, \lambda)$ has global saddle points too. With that goal in mind, we consider the following regularized problem
    \begin{equation}
        \label{eq::P-CRL-REG}
        \begin{array}{rl}
        \displaystyle\max_{\pi\, \in\, \Pi}
        & \;\; V_{r} (\pi) \,+\, \frac{\tau}{2}H(\pi)
        \\[0.2cm]
        \subjectto & \;\; V_g (\pi) 
        \;\geq\; 0, 
        \end{array}
    \end{equation}
    where $H(\pi)$ is the regularizer introduced in Section \ref{ss::proposed_method}. The term $V_{r} (\pi) + \frac{\tau}{2}H(\pi)$ can be redefined in terms of the regularized reward function $r_\tau(s, a) \DefinedAs r(s, a) - \frac{\tau}{2} h_a(a)$. This leads to a value function $V_{r_\tau} (\pi)$. As Theorem \ref{thm::zero_duality_gap} does not require any Assumption in terms of the reward function, it follows that Problem \eqref{eq::P-CRL-REG} has zero duality gap too. This implies that its associated Lagrangian $\hat{L}_\tau(\pi, \lambda) \DefinedAs V_{r} (\pi) + \frac{\tau}{2}H(\pi) + \lambda V_g(\pi)$ has global saddle points.

    It remains to be shown that adding the regularization of the Lagrangian multiplier preserves some global saddle points, i.e., $L_\tau(\pi, \lambda) = \hat{L}_\tau(\pi, \lambda) + \frac{\tau}{2}\lambda^2$ has global saddle points. We first consider the dual problem associated with $\hat{L}_\tau(\pi, \lambda)$
    \begin{equation}
        \label{eq::dual-reg}
        \min_{\lambda \in \lambda} \; \max_{\pi \in \Pi} \; V_{r} (\pi) + \frac{\tau}{2}H(\pi) + \lambda V_g(\pi).
    \end{equation}
    
    The Lagrangian multiplier $\lambda^\star_\tau=0$ is always a solution to the dual Problem \eqref{eq::dual-reg} due to complementary slackness. We denote by $\pi^\star_\tau$ a primal solution associated with $\lambda^\star_\tau$. If $V_g(\pi^\star_\tau)$ is strictly feasible, i.e. $V_g(\pi^\star_\tau)>0$, it follows that $\lambda^\star_\tau=0$ is the minimizer of the dual function. Otherwise, if $V_g(\pi^\star_\tau)=0$ the dual function does not depend on $\lambda$, and we can select $\lambda^\star_\tau=0$ safely. Furthermore, the pair $(\pi^\star_\tau, \lambda^\star_\tau)$ satisfies for all $\pi \in \Pi$,
    \begin{align}
    \nonumber
        \hat{L}_\tau(\pi^\star_\tau, \lambda^\star_\tau) + \frac{\tau}{2} (\lambda^\star_\tau)^2 \;\geq\;
        \hat{L}_\tau(\pi, \lambda^\star_\tau) + \frac{\tau}{2} (\lambda^\star_\tau)^2.
    \end{align}
    
    Similarly, as the minimum element of the set $\Lambda$ is zero, it follows that for all $\lambda \in \Lambda$,
    \begin{align}
    \nonumber
        \hat{L}_\tau(\pi^\star_\tau, \lambda^\star_\tau) + \frac{\tau}{2} (\lambda^\star_\tau)^2 \;\leq\; \hat{L}_\tau(\pi^\star_\tau, \lambda) + \frac{\tau}{2} (\lambda)^2.
    \end{align}

    This implies that the pair $(\pi^\star_\tau, \lambda^\star_\tau)$ is a global saddle point of $L_\tau(\pi, \lambda)$, as it satisfies for all $(\pi, \lambda) \in \Pi \times \Lambda$,
    \begin{equation}
        \nonumber
        L_\tau(\pi, \lambda^\star_\tau) 
        \;\leq\; L_\tau(\pi^\star_\tau, \lambda^\star_\tau) 
        \;\leq\; L_\tau(\pi^\star_\tau, \lambda).
    \end{equation}

    Moreover, the first inequality implies that for any $\pi \in \Pi$ we have
    \begin{align}
        \nonumber
        &V_{\lambda_\tau^\star}(\pi_\tau^\star) + \frac{\tau}{2} H(\pi_\tau^\star) + \frac{\tau}{2} h(\lambda_\tau^\star)  \;\geq\;  V_{\lambda_\tau^\star}(\pi) + \frac{\tau}{2} H(\pi) + \frac{\tau}{2} h(\lambda_\tau^\star)
    \end{align}
    which leads to
    \begin{align}
        \nonumber
        V_{\lambda_\tau^\star}(\pi_\tau^\star)  
        \;\geq\;  V_{\lambda_\tau^\star}(\pi) + \frac{\tau}{2} H(\pi).
    \end{align}

    The second inequality implies that for any $\lambda \in \Lambda$ we have
    \begin{align}
        \nonumber
        &V_{\lambda}(\pi_\tau^\star) + \frac{\tau}{2} H(\pi_\tau^\star)  + \frac{\tau}{2} h(\lambda)  \\
        \nonumber
        &\;\geq\; V_{\lambda_\tau^\star}(\pi_\tau^\star) + \frac{\tau}{2} H(\pi_\tau^\star) + \frac{\tau}{2} h(\lambda_\tau^\star) \\
        \nonumber
        &\;\geq\; V_{\lambda_\tau^\star}(\pi_\tau^\star) + \frac{\tau}{2} H(\pi_\tau^\star)
    \end{align}
    which leads to
    \begin{align}
        \nonumber
        V_{\lambda}(\pi_\tau^\star)  + \frac{\tau}{2} h(\lambda)  \;\geq\;
        V_{\lambda_\tau^\star}(\pi_\tau^\star).
    \end{align}

    Combining both inequalities we have for all $(\pi, \lambda) \in \Pi \times \Lambda$,
    \begin{equation}
    \nonumber
        V_{\lambda_\tau^\star}(\pi) + \frac{\tau}{2} H(\pi) 
        \;\leq\; V_{\lambda_\tau^\star}(\pi_\tau^\star) 
        \;\leq\;
        V_{\lambda}(\pi_\tau^\star) + \frac{\tau}{2} \lambda^2
    \end{equation}
    which concludes the proof.
\end{proof}

\subsection{Proof of Theorem \ref{thm::zero_duality_gap}}
\label{app::zero_duality_gap}

\begin{proof}

To establish zero duality gap we consider the perturbed function of Problem \eqref{eq::P-CRL}. For any $\delta \in \reals$, the perturbation function $P(\delta)$ associated with \eqref{eq::P-CRL} is defined as follows
\begin{align}
    \label{eq::perturbed_P-CRL}
    P(\delta) 
    \;\DefinedAs\;
    \max_{\pi \in \Pi}
    & \;\; V_{r} (\pi) \\
    \nonumber
    \subjectto & \;\; V_g (\pi) \geq \delta. \notag
\end{align}

The proof relies on Lemma \ref{lm::fenchel_moreau}, which states that if Slater's condition holds for \eqref{eq::P-CRL} and its perturbation function $P(\delta)$ is concave on $\delta$, then \eqref{eq::P-CRL} has zero duality gap, even if the problem in \eqref{eq::P-CRL} is non-convex. Therefore, we have to show that $P(\delta)$ is convex. More precisely, we need to establish that for a perturbation $\delta_\alpha \DefinedAs \alpha \delta + (1 - \alpha) \delta'$, where $\alpha \in [0, 1]$, it holds that
\begin{equation}
\nonumber
    P(\delta_\alpha) 
    \;\DefinedAs\;
    P(\alpha \delta + (1 - \alpha) \delta) 
    \;\geq\;
    \alpha P(\delta) + (1 - \alpha) P(\delta').
\end{equation}

Let $\pi$ and $\pi'$ be the policies that achieve the maximum value of $P(\delta)$ for perturbations $\delta$ and $\delta'$ respectively. This implies 
\begin{align}
    \label{eq::perturbation_inequalities}
    V_g(\pi) 
    \;\geq\;
    \delta 
    \;\; \text{and} \;\; 
    V_g(\pi') 
    \;\geq\;
    \delta'.
\end{align}

Now, for a given policy $\pi$, we consider the vector value function $V(\pi)\DefinedAs[V_r(\pi), V_g(\pi)]^\Tr$. We define the deterministic value image as the span of value functions associated with the class of deterministic policies $\Pi$ as:
\begin{equation}
    \nonumber
    \ccalV_D 
    \;\DefinedAs\;
    \{ V(\pi) : \pi \in \Pi \}.
\end{equation}

The deterministic value image $\ccalV_D$ is convex for non-atomic MDPs (see \citet[Corollary 3.10]{feinberg2019sufficiency}. This implies that there exists $\pi_\alpha \in \Pi$ such that $V({\pi_{\alpha}})\DefinedAs \alpha V(\pi) + (1 - \alpha) V({\pi'})$. More precisely:
\begin{align}
    V_r ({\pi_{\alpha}}) 
    \;=\;
    \alpha V_r (\pi) + (1 - \alpha) V_r (\pi') \\
    \label{eq::perturbation_convex_combinations}
    V_g ({\pi_{\alpha}}) 
    \;=\;
    \alpha V_g (\pi) + (1 - \alpha) V_g({\pi'}).
\end{align}

It follows immediately that the objective of \eqref{eq::perturbed_P-CRL} is convex. Then, it only remains to be confirmed that $\pi_{\alpha}$ is feasible for the perturbation $\delta_\alpha$. We perform a convex combination of the inequalities in \eqref{eq::perturbation_inequalities} to show that
\begin{align}
    \nonumber
    \alpha V_g({\pi}) + (1 - \alpha) V_g({\pi'}) 
    \;\geq\;
    \delta_\alpha  
    \\
    \nonumber \text{ and }\;\;
    V_g({\pi_\alpha}) 
    \;\geq\; \delta_\alpha
\end{align}
which concludes the proof.
\end{proof}

\subsection{Proof of Unique Saddle Point}
\label{app::uniqueness_saddle}

\begin{proof}

The regularized Lagrangian $L_\tau(\pi, \lambda)$ is a regular unconstrained MDP for a fixed $\lambda$. For all $\lambda$, if the reward function $r_{\lambda, \tau}$ is convex on $s$ and strictly concave on $a$, and the dynamics are linear, the resultant unconstrained MDP is guaranteed to have a unique maximizer $\pi_\tau^\star(\lambda)$ \citep{cruz2004conditions, montes2013nonuniqueness}. Now, consider the regularized dual function
\begin{equation}
    \label{eq::dual_uniqueness}
    D_\tau(\lambda) 
    \; = \;
    \lambda^2 + \max_{\pi \in \Pi} V_{\lambda, \tau} (\pi).
\end{equation}

We know by Danskin’s Theorem (see Lemma \ref{lm::danskins}) that (i) the function \eqref{eq::dual_uniqueness} is convex w.r.t $\lambda$, and (ii) the minimizer $\lambda^\star_\tau$ is unique if the Lagrangian maximizer $\pi_\tau^\star(\lambda)$ is unique for all $\lambda \in \Lambda$. Therefore, we can conclude uniqueness of the primal-dual optimal pair of $L_\tau(\pi, \lambda)$ under concavity of the reward function $r_{\lambda, \tau}$ and linearity of the dynamics. 
\end{proof}

\subsection{Proof of Theorem \ref{thm::convergence_exact}}

\label{app::convergence_exact}

\begin{proof}

We set up the stage by introducing the constants that will be relevant throughout the proof. First, we have $C_P \DefinedAs L_r + \lambda_{\text{max}} L_g + \tau L_h + \tau \sqrt{d_a} A_{\text{max}}$, where $L_r$, $L_g$, and $L_h$ are the Lipschitz constants of action value functions introduced in Assumption \ref{as::lipschitz}, and $A_{\text{max}}$ is the maximum action of the bounded action space $A$. Then, we have $C_D\DefinedAs \frac{1}{1 - \gamma}(1 + \frac{\tau}{\xi}) \geq V_g (\pi_t) + \tau \lambda_t$, where $\xi$ comes from the feasibility Assumption \ref{as::feasibility}.

We also recall that the operator $\ccalP_{\Pi^\star_\tau}$  projects a policy  onto the set of regularized optimal policies $\Pi^\star_\tau$ with state visitation distribution $d^\star_\rho$. For a given policy $\pi_t$, we denote $\pi^\star_{\tau, t} \DefinedAs \ccalP_{\Pi^\star_\tau}(\pi_t)$. Similarly, the operator $\ccalP_{\Lambda^\star_\tau}$ projects a Lagrangian multiplier onto the set of regularized optimal Lagrangian multipliers $\Lambda^\star_\tau$, and the projection is denoted as $\lambda^\star_{\tau, t} \DefinedAs \ccalP_{\Lambda^\star_\tau}(\lambda_t)$. 

We proceed by decomposing the primal-dual gap by adding and substracting $L_\tau(\pi_t, \lambda_t)$,
\begin{align}
    \nonumber
    &L_\tau(\pi^\star_{\tau, t}, \lambda_t) - L_\tau(\pi_t, \lambda^\star_{\tau, t}) \\
    \nonumber
    &=\; \underbrace{L_\tau(\pi^\star_{\tau, t}, \lambda_t) - L_\tau(\pi_t, \lambda_t)}_{\text{\normalfont (i)}} 
    \,+\,
    \underbrace{L_\tau(\pi_t, \lambda_t) - L_\tau(\pi_t, \lambda^\star_{\tau, t})}_{\text{\normalfont (ii)}}.
\end{align}

Let's define the distance function $D_t(\pi^\star_{\tau, t})= \E_{d^\star_\rho} \left[ \| \pi^\star_{\tau, t}(s) - \pi_t(s) \|^2  \right]$. Now, using the definition of the regularized Lagrangian and the performance difference Lemma \ref{lm::pd_lemma}, we can show that for term (i) we have that
\begin{align}
    \label{eq::pd_exact_1a}
    L_\tau(\pi^\star_{\tau, t}, \lambda_t) - L_\tau(\pi_t, \lambda_t) 
    &\;=\; V_{\lambda_t, \tau} (\pi^\star_{\tau, t}) - V_{\lambda_t, \tau} ({\pi_t}) \\
    \nonumber
    &\;=\; \frac{1}{1 - \gamma} \E_{d^\star_\rho} \left[ A^{\pi_t}_{\lambda_t, \tau}(s, \pi^\star_{\tau, t}(s)) \right].
\end{align}
Note that the objective function in \eqref{eq::primal_update} is strongly concave w.r.t. $a$. Therefore, we can leverage the quadratic growth Lemma \ref{lm::qg_lemma} with the function $f(a) \DefinedAs A^{\pi_t}_{\lambda_t, \tau}(s, a) - \frac{1}{2 \eta}\|\pi_t(s) - a\|^2$, and its the maximizer being $a^\star=\pi_{t+1}(s)$. By the optimality conditions of \eqref{eq::primal_update} we have that
\begin{align}
    \label{eq::pd_exact_1b}
     \E_{d^\star_\rho} \left[ A^{\pi_t}_{\lambda_t, \tau}(s, \pi^\star_{\tau, t}(s)) \right] 
    &\;\leq\; \E_{d^\star_\rho} \left[ A^{\pi_t}_{\lambda_t, \tau}(s, \pi_{t+1}(s)) - \frac{1}{2\eta} \|\pi_{t+1}(s) - \pi_t(s)\|^2 \right] \\
    \nonumber
    & \;\;\;\; \;\;+ \E_{d^\star_\rho} \left[\frac{1}{2\eta} \| \pi^\star_{\tau, t}(s) - \pi_t(s)\|^2 \right] \\
    \nonumber
    & \;\;\;\; \;\;- \E_{d^\star_\rho} \left[ \left(\frac{1 + \eta (\tau - \tau_0)}{2 \eta}\right) \|\pi^\star_{\tau, t}(s) - \pi_{t+1}(s)\|^2 \right].
\end{align}
Now, we can use Lemma \ref{lm::lipschitz_lemma} with $a=\pi_{t+1}(s)$ to bound the term
\begin{align}
    \label{eq::pd_exact_1c}
    \E_{d^\star_\rho} \left[ A^{\pi_t}_{\lambda_t, \tau}(s, \pi_{t+1}(s)) - \frac{1}{2\eta} \|\pi_{t+1}(s) - \pi_t(s)\|^2 \right] 
    \;\leq\;
    \frac{\eta}{2} C_P^2.
\end{align}
Combining the expressions in \eqref{eq::pd_exact_1a}, \eqref{eq::pd_exact_1b} and \eqref{eq::pd_exact_1c} we have that for term (i) it holds that
\begin{align}
    \nonumber
    L_\tau(\pi^\star_{\tau, t}, \lambda_t) - L_\tau(\pi_t, \lambda_t) 
    \;\leq\;
    \frac{D_t(\pi^\star_{\tau, t}) - ( 1 + \eta (\tau - \tau_0)) D_{t+1}(\pi^\star_{\tau, t}) }{2 \eta (1 - \gamma)}  + \frac{\eta}{2 (1 - \gamma)} C_P^2.
\end{align}

For the second term (ii) we use the definition of the regularized Lagrangian and the definition of the regularized value function to show that
\begin{align}
    \label{eq::pd_exact_2a}
    &L_\tau(\pi_t, \lambda_t) - L_\tau(\pi_t, \lambda^\star_{\tau, t}) \\
    \nonumber
    &=\; V_{\lambda_t, \tau}({\pi_t}) - V_{\lambda^\star_{\tau, t}, \tau}({\pi_t}) + \frac{\tau}{2} (\lambda_t)^2 - \frac{\tau}{2} (\lambda^\star_{\tau, t})^2 \\
    \nonumber
    &=\; (\lambda_t - \lambda^\star_{\tau, t}) V_g (\pi_t) + \frac{\tau}{2} (\lambda_t)^2 - \frac{\tau}{2} (\lambda^\star_{\tau, t})^2.
\end{align}
Then, completing the squares we have that
\begin{align}
    \label{eq::pd_exact_2b}
    &(\lambda_t - \lambda^\star_{\tau, t}) V_g (\pi_t) + \frac{\tau}{2} (\lambda_t)^2 - \frac{\tau}{2} (\lambda^\star_{\tau, t})^2 \\
    \nonumber
    &=\; (\lambda_t - \lambda^\star_{\tau, t}) (V_g (\pi_t) + \tau \lambda_t) - \frac{\tau}{2}  \|\lambda^\star_{\tau, t} - \lambda_t\|^2.
\end{align}
We can now use the standard descent Lemma \ref{lm::descent_lemma} to show that
\begin{align}
    \label{eq::pd_exact_2c}
    &(\lambda_t - \lambda^\star_{\tau, t}) (V_g (\pi_t) + \tau \lambda_t) - \frac{\tau}{2}  \|\lambda^\star_{\tau, t} - \lambda_t\|^2 \\
    \nonumber
    &\leq \; \frac{\|\lambda^\star_{\tau, t} - \lambda_t \|^2 - \|\lambda^\star_{\tau, t} - \lambda_{t+1} \|^2}{2 \eta} + \frac{1}{2} \eta C_D^2 - \frac{\tau}{2} \| \lambda_t - \lambda^\star_{\tau, t} \|^2 \\
    \nonumber
    &=\; \frac{(1 - \eta \tau)\|\lambda^\star_{\tau, t} - \lambda_t \|^2 - \|\lambda^\star_{\tau, t} - \lambda_{t+1} \|^2}{2 \eta} + \frac{1}{2} \eta C_D^2,
\end{align}
where the last equality rearranges the term $\frac{\tau}{2} \| \lambda_t - \lambda^\star_{\tau, t} \|^2$ inside the fraction. Now, combining the expressions in \eqref{eq::pd_exact_2a}, \eqref{eq::pd_exact_2b} and \eqref{eq::pd_exact_2c} we have that for term (ii) it holds that
\begin{align}
    \nonumber
    &L_\tau(\pi_t, \lambda_t) - L_\tau(\pi_t, \lambda^\star_{\tau, t}) 
    \;=\;
    \frac{(1 - \eta \tau)\|\lambda^\star_{\tau, t} - \lambda_t \|^2 - \|\lambda^\star_{\tau, t} - \lambda_{t+1} \|^2}{2 \eta} + \frac{1}{2} \eta C_D^2.
\end{align}

The next step is to combine (i) and (ii). We define $C_0 \DefinedAs C_P + C_D$ and since the duality gap $L_\tau(\pi^\star_{\tau, t}, \lambda_t) - L_\tau(\pi_t, \lambda^\star_{\tau, t})$ is positive and holds for any $(\pi_\tau^\star,\lambda_\tau^\star) \in \Pi^\star_\tau \times \Lambda^\star_\tau$, we have that
\begin{align}
    \nonumber
    0 & \;\leq\; \eta (1 - \gamma) \left( L_\tau(\pi^\star_{\tau, t}, \lambda_t) - L_\tau(\pi_t, \lambda^\star_{\tau, t})\right) \\
    \nonumber
    & \;\leq\; \frac{1}{2} D_t(\pi^\star_{\tau, t}) + \frac{1 - \eta \tau}{2} \|\lambda^\star_{\tau, t} - \lambda_t \|^2  \\
    \nonumber
    & \;\;\;\;\;\;- \frac{ 1 + \eta (\tau - \tau_0)}{2} D_{t+1}(\pi^\star_{\tau, t}) - \frac{1}{2}\|\lambda^\star_{\tau, t} - \lambda_{t+1} \|^2 + \eta^2 C_0^2.
\end{align}
Rearranging the expression above, it follows that
\begin{align}
    \nonumber
    &\frac{ 1 + \eta (\tau - \tau_0)}{2} D_{t+1}(\pi^\star_{\tau, t}) + \frac{1}{2}\|\lambda^\star_{\tau, t} - \lambda_{t+1} \|^2 \\
    \nonumber
    &=\; ( 1 + \eta (\tau - \tau_0)) \left( \frac{1}{2} D_{t+1}(\pi^\star_{\tau, t}) + \frac{1}{2(1 + \eta (\tau - \tau_0))}\| \lambda^\star_{\tau, t} - \lambda_{t+1} \|^2 \right) \\
    \nonumber
    &\leq\; \frac{1}{2} D_t(\pi^\star_{\tau, t}) + \frac{1 - \eta \tau}{2}\|\lambda^\star_{\tau, t} - \lambda_t \|^2 + \eta^2 C_0^2\\
    \nonumber
    &\leq\; \frac{1}{2} D_t(\pi^\star_{\tau, t}) + \frac{1}{2(1 + \eta (\tau - \tau_0))}\|\lambda^\star_{\tau, t} - \lambda_t \|^2 + \eta^2 C_0^2,
\end{align}
where the last inequality follows from the fact that $\tau_0 \in [0, \tau)$, and therefore $\frac{1 - \eta \tau}{2} \leq \frac{1}{2(1 + \eta (\tau - \tau_0))}$. Therefore, in summary we have that
\begin{align}
    \label{eq::inequality_contraction_a}
    &( 1 + \eta (\tau - \tau_0))  \left( \frac{1}{2} D_{t+1}(\pi^\star_{\tau, t}) + \frac{1}{2(1 + \eta (\tau - \tau_0))}\| \lambda^\star_{\tau, t} - \lambda_{t+1} \|^2 \right) \\
    \nonumber
    &\;\leq\; \frac{1}{2} D_t(\pi^\star_{\tau, t}) + \frac{1}{2(1 + \eta (\tau - \tau_0))}\|\lambda^\star_{\tau, t} - \lambda_t \|^2 + \eta^2 C_0^2.
\end{align}
Now, since $\pi_\tau^\star = \ccalP_{\Pi^\star_\tau}(\pi_t)$ and $\lambda_\tau^\star = \ccalP_{\Lambda^\star_\tau}(\lambda_t)$ are projections, it holds that
\begin{align}
    \nonumber
    &\frac{ 1 + \eta (\tau - \tau_0)}{2} D_{t+1}(\pi^\star_{\tau, t + 1}) + \frac{1}{2}\|\lambda^\star_{\tau, t + 1} - \lambda_{t+1} \|^2 \\
    \nonumber
    & =\; \frac{ 1 + \eta (\tau - \tau_0)}{2} \E_{d^\star_\rho} \left[ \|(\ccalP_{\Pi^\star_\tau}(\pi_{t+1}(s)) - \pi_{t+1}(s)\|^2  \right]  + \frac{1}{2}\|\ccalP_{\Lambda^\star_\tau}(\lambda_{t+1}) - \lambda_{t+1} \|^2 \\
    \nonumber
    &\leq \;\frac{ 1 + \eta (\tau - \tau_0)}{2} \E_{d^\star_\rho} \left[ \|(\ccalP_{\Pi^\star_\tau}(\pi_{t}(s)) - \pi_{t+1}(s)\|^2  \right]  + \frac{1}{2}\| \ccalP_{\Lambda^\star_\tau}(\lambda_{t}) - \lambda_{t+1} \|^2 \\
    \nonumber
    &= \;\frac{ 1 + \eta (\tau - \tau_0)}{2} D_{t+1}(\pi^\star_{\tau, t}) + \frac{1}{2}\|\lambda^\star_{\tau, t} - \lambda_{t+1} \|^2.
\end{align}
In summary, we have that
\begin{align}
    \label{eq::inequality_projection_a}
    &\frac{ 1 + \eta (\tau - \tau_0)}{2} D_{t+1}(\pi^\star_{\tau, t + 1}) + \frac{1}{2}\|\lambda^\star_{\tau, t + 1} - \lambda_{t+1} \|^2 \\
    \nonumber
    &= \;\frac{ 1 + \eta (\tau - \tau_0)}{2} D_{t+1}(\pi^\star_{\tau, t}) + \frac{1}{2}\|\lambda^\star_{\tau, t} - \lambda_{t+1} \|^2.
\end{align}
We define the potential function $\Phi_t \DefinedAs \frac{1}{2} D_t(\pi_{\tau, t}^\star) + \frac{1}{2(1 + \eta (\tau - \tau_0))} \| \ccalP_{\Lambda^\star_\tau}(\lambda_{t}) - \lambda_t \|^2$ and combine the inequalities \eqref{eq::inequality_projection_a} and \eqref{eq::inequality_contraction_a} to show
\begin{equation}
    \label{eq::potential_contraction}
    (1 + \eta (\tau - \tau_0)) \Phi_{t+1} 
    \;\leq\;
    \Phi_t + \eta^2 C_0^2.
\end{equation}

The next step is to expand \eqref{eq::potential_contraction} so that we have
\begin{align}
    \nonumber
    \Phi_{t+1} &\;\leq\; \frac{1}{1 + \eta (\tau - \tau_0)} \Phi_t + \eta^2 C_0^2 
    \\
    \nonumber
    &\;\leq\; \left(\frac{1}{1 + \eta (\tau - \tau_0)} \right)^2 \Phi_{t-1}  + \left(\eta^2 + \eta^2 \left(\frac{1}{1 + \eta (\tau - \tau_0)}\right) \right) C_0^2.
\end{align}
If we keep expanding recursively we end up having
\begin{align}
    \nonumber
    \Phi_{t+1} &\;\leq\; \left(\frac{1}{1 + \eta (\tau - \tau_0)} \right)^t \Phi^{1}  + \left(\eta^2 \left( 1 + \left(\frac{1}{1 + \eta (\tau - \tau_0)}\right) + \hdots \right) \right) C_0^2.
\end{align}
Finally, using the definition of the exponential for the first term, and using the geometric series for the second one, we have that
\begin{align}
    \nonumber
    &\Phi_{t+1} \;\leq\; {\rm e}^{-\frac{\eta (\tau - \tau_0)}{1 + \eta (\tau - \tau_0) }t} \Phi_1 + \frac{\eta (1 + \eta (\tau - \tau_0))}{\tau - \tau_0} C_0^2
\end{align}
which completes the proof.
\end{proof}

\subsection{Proof of Corollary \ref{cor::near_optimality}}
\label{app::corollary_reg_problem}

\begin{proof}
    Due to Theorem \eqref{thm::convergence_exact}, taking $\tau=O(\epsilon) + \tau_0$ and $\eta=O(\epsilon)$ leads to $\Phi_{t+1}=O(\epsilon)$ for any $t= \Omega(\frac{1}{\epsilon^3} \text{\normalfont log} \frac{1}{\epsilon})$, where $\Omega$ encapsulates some problem-dependent constants. For some $t= \Omega(\frac{1}{\epsilon^3} \text{\normalfont log} \frac{1}{\epsilon})$ the primal-dual iterate $(\pi_t, \lambda_t)$ satisfies $D_t(\pi^\star_{\tau, t})=O(\epsilon)$ and $\frac{1}{2(1 + \eta (\tau - \tau_0))}\| \lambda^\star_{\tau, t} - \lambda_t \|^2 = O(\epsilon)$. Now, adding and substracting $V_r({\pi^\star_{\tau, t}})$ we have
    \begin{align}
        \nonumber
        &V_r({\pi^\star}) - V_r({\pi_t}) 
        \;=\;
        \underbrace{V_r({\pi^\star}) - V_r({\pi^\star_{\tau, t}})}_{\text{\normalfont (i)}} 
        \,+ \,
        \underbrace{V_r({\pi^\star_{\tau, t}}) - V_r({\pi_t})}_{\text{\normalfont (ii)}}.
    \end{align}

    For the term (i) we can take $\pi=\pi^\star$ in \eqref{eq::sandwich_property} to show
    \begin{equation}
    \nonumber
        V_r({\pi^\star}) + \frac{\tau}{2} H(\pi^\star) 
        \;\leq\; V_r({\pi^\star_{\tau, t}}) + \lambda_\tau^\star \left( V_g({\pi^\star_{\tau, t}}) - V_g({\pi^\star}) \right).
    \end{equation}
    Then, taking $\lambda=0$ in \eqref{eq::sandwich_property} leads to $\lambda_\tau^\star  V_g({\pi^\star_\tau}) \leq 0$. Noticing that $\pi^\star$ feasible implies $V_g({\pi^\star}) \geq 0$ we have
    \begin{equation}
        \label{eq::bound_first_term}
        V_r({\pi^\star}) - V_r({\pi^\star_{\tau, t}}) \;\leq\;
        - \frac{\tau}{2} H(\pi^\star).
    \end{equation}

    For the term (ii), by the performance difference Lemma in \ref{lm::pd_lemma} with $\tau=0$ we have that
    \begin{align}
        \nonumber
        &V_r({\pi^\star_{\tau, t}}) - V_r({\pi_t}) 
        \;=\;
        \frac{1}{1 - \gamma} \E_{d^\star_\rho} \left[ Q_r^{\pi_t} (s, \pi^\star_{\tau, t}(s)) - Q_r^{\pi_t}(s, \pi_t(s)) \right].
    \end{align}
    Then, we have that by Assumption \ref{as::lipschitz} the action-value function is $L_r$-Lipschitz
    \begin{align}
        \nonumber
        \frac{1}{1 - \gamma} \E_{d^\star_\rho} \left[ Q_r^{\pi_t} (s, \pi^\star_{\tau, t}(s)) - Q_r^{\pi_t}(s, \pi_t(s)) \right] 
        &\;\leq\;
        \frac{L_r}{1 - \gamma} \E_{d^\star_\rho} \left[ \| \pi^\star_{\tau, t} (s) - \pi_t(s) \|  \right] 
        \\
        \nonumber
        &\;=\; \frac{L_r}{1 - \gamma} \E_{d^\star_\rho} \left[ \sqrt{\| \pi^\star_{\tau, t} (s) - \pi_t(s) \|^2} \; \right],
    \end{align}
    where we have taken the square and the square root in the last equality. Note now that we can use Cauchy-Schwartz to show that
    \begin{align}
        \nonumber
        V_r({\pi^\star_{\tau, t}}) - V_r({\pi_t}) 
        &\;=\; \frac{L_r}{1 - \gamma} \E_{d^\star_\rho} \left[ \sqrt{\| \pi^\star_{\tau, t} (s) - \pi_t(s) \|^2} \; \right] \\
        \nonumber
        &\;\leq\; \frac{L_r}{1 - \gamma} \sqrt{\E_{d^\star_\rho} \left[ \| \pi^\star_{\tau, t} (s) - \pi_t(s) \|^2 \right]} \\
        \nonumber
        &\;=\; \frac{L_r}{1 - \gamma} \sqrt{D_t(\pi^\star_{\tau, t})}.
    \end{align}
    
    Therefore, combining (i) and (ii) leads to $V_r({\pi^\star}) - V_r({\pi_t}) \leq O(\sqrt{\epsilon}) - \frac{\tau_0}{2} H(\pi^\star)$.

    Then, for any $\pi^\star_\tau \in \Pi^\star_\tau$, we have
    \begin{equation}
    \nonumber
        -V_g({\pi_t}) 
        \; = \;
        \underbrace{-V_g({\pi^\star_\tau})}_{\text{\normalfont (iii)}} 
        \,+\, \underbrace{V_g({\pi^\star_\tau}) - V_g({\pi_t})}_{\text{\normalfont (iv)}}.
    \end{equation}

    For (iii) we take $\lambda=\lambda_{\text{max}}\DefinedAs\frac{1}{(1 - \gamma)\xi}$ in  \eqref{eq::sandwich_property} to show that $-(\lambda_{\text{max}} - \lambda_\tau^\star) V_g({\pi_\tau^\star}) \leq \frac{\tau}{2}(\lambda_{\text{max}})^2$. By the definition

    \begin{equation}
    \nonumber
        \lambda_\tau^\star \;\DefinedAs \;
        \argmin_{\lambda \in \Lambda} \; \lambda V_g({\pi_{\tau, t}^\star}) + \frac{\tau}{2} \lambda^2,
    \end{equation}
    $\lambda_\tau^\star$ can take three different values:
    \begin{enumerate}
        \item When $0 \leq - V_g({\pi_\tau^\star})/\tau \leq \lambda_{\text{max}}$, then $\lambda_\tau^\star=- V_g({\pi_\tau^\star})/\tau$ and $\lambda_{\text{max}}-\lambda_\tau^\star \geq 0$.
        \item When $- V_g({\pi_\tau^\star})/\tau \leq 0$, then $\lambda_\tau^\star=0$ and $\lambda_{\text{max}}-\lambda_\tau^\star \geq 0$.
        \item When $- V_g({\pi_\tau^\star})/\tau \geq \lambda_{\text{max}}$, then $\lambda_\tau^\star=\lambda_{\text{max}}$.
    \end{enumerate}

    For the third case, we take $\pi = \pi^\star$ in \eqref{eq::sandwich_property} to show
    \begin{align}
        \label{eq::part1}
        &\lambda_{\text{max}} (V_g({\pi^\star}) - V_g({\pi^\star_\tau})) \;\leq \;
        V_r({\pi^\star_\tau}) - V_r({\pi^\star}) - \frac{\tau}{2}H(\pi^\star).
    \end{align}

    Furthermore, for any global saddle point of the non-regularized Lagrangian $(\pi^\star, \lambda^\star) \in \Pi^\star \times \Lambda^\star$, it holds
    \begin{equation}
        \label{eq::part2}
        V_r({\pi^\star_\tau}) - V_r({\pi^\star}) 
        \;\leq\;
        \lambda^\star (V_g({\pi^\star}) - V_g({\pi^\star_\tau})).
    \end{equation}

    Combining \eqref{eq::part1} and \eqref{eq::part2}, and leveraging that $V_g(\pi^\star) \geq 0$, we have
    \begin{equation}
    \nonumber
        -(\lambda_{\text{max}} - \lambda^\star)V_g({\pi^\star_\tau}) 
        \;\leq\;
        - \frac{\tau}{2}H(\pi^\star).
    \end{equation}

    As $\lambda^\star \leq \lambda_{\text{max}}$ is a global saddle point,  we can conclude that
    \begin{equation}
    \nonumber
        -V_g({\pi_\tau^\star}) \;\leq\;
        - \frac{\tau}{2(\lambda_{\text{max}} - \lambda^\star)} H(\pi^\star).
    \end{equation}

    The term (iv) has a similar bound to (ii), so we can show that it is $O(\sqrt{\epsilon})$. Combining (iii) and (iv) we have $-V_g({\pi_t}) \leq O(\sqrt{\epsilon}) - \frac{\tau_0}{2(\lambda_{\text{max}} - \lambda^\star)} H(\pi^\star)$. We replace $\sqrt{\epsilon}$ for $\epsilon$ and combine all big O notation to conclude the proof.
\end{proof}

\subsection{Proof of Theorem \ref{thm::convergence_inexact}}

\label{app::convergence_inexact}

\begin{proof}

We recall the constants that will be relevant throughout the proof. First, we have $C_P \DefinedAs L_r + \lambda_{\text{max}} L_g + \tau L_h + \tau \sqrt{d_a} A_{\text{max}}$, where $L_r$, $L_g$, and $L_h$ are the Lipschitz constants of action value functions introduced in Assumption \ref{as::lipschitz}, and $A_{\text{max}}$ is the maximum action of the bounded action space $A$. Then, we have $C_D\DefinedAs \frac{1}{1 - \gamma}(1 + \frac{\tau}{\xi}) \geq V_g (\pi_t) + \tau \lambda_t$, where $\xi$ comes from the feasibility Assumption \ref{as::feasibility}. We also introduce again the operator $\ccalP_{\Pi^\star_\tau}$ that projects a policy  onto the set of regularized optimal policies $\Pi^\star_\tau$ with state visitation distribution $d^\star_\rho$. For a given policy $\pi_t$, we denote $\pi^\star_{\tau, t} \DefinedAs \ccalP_{\Pi^\star_\tau}(\pi_t)$. Similarly, the operator $\ccalP_{\Lambda^\star_\tau}$ projects a Lagrangian multiplier onto the set of regularized optimal Lagrangian multipliers $\Lambda^\star_\tau$, and the projection is denoted as $\lambda^\star_{\tau, t} \DefinedAs \ccalP_{\Lambda^\star_\tau}(\lambda_t)$. 

We begin by using the definition of the regularized advantage evaluated in $\pi^\star_{\tau, t}(s)$, and we add and substract $\frac{1}{\eta} \pi_t(s)^\Tr \pi^\star_{\tau, t}(s)$ so that we have
\begin{align}
    \nonumber
    A_{\lambda_t, \tau}^{\pi_t} (s, \pi^\star_{\tau, t}(s)) 
    &\;=\; Q_{\lambda_t, \tau}^{\pi_t}(s, \pi^\star_{\tau, t}(s)) - V_{\lambda_t, \tau}^{\pi_t}(s)   - \frac{\tau}{2} \| \pi^\star_{\tau, t}(s) \|^2 + \frac{\tau}{2} \| \pi_t(s) \|^2  \\
    \nonumber
    &\;=\; Q_{\lambda_t, \tau}^{\pi_t}(s, \pi^\star_{\tau, t}(s)) - V_{\lambda_t, \tau}^{\pi_t}(s) - \frac{\tau}{2} \| \pi^\star_{\tau, t}(s) \|^2  \\
    \nonumber
    & \;\;\;\;\;\;\; + \frac{\tau}{2} \| \pi_t(s) \|^2 + \frac{1}{\eta} \pi_t(s)^\Tr \pi^\star_{\tau, t}(s) - \frac{1}{\eta} \pi_t(s)^\Tr \pi^\star_{\tau, t}(s).
\end{align}
Note that we can equivalently use the augmented action value function $J^{\pi_t}$ such that it holds that
\begin{align}
    \nonumber
    A_{\lambda_t, \tau}^{\pi_t} (s, \pi^\star_{\tau, t}(s)) 
    & \;=\; J^{\pi_t}(s, \pi^\star_{\tau, t}(s)) - \frac{1}{\eta} \pi_t(s)^\Tr \pi^\star_{\tau, t}(s) - V_{\lambda_t, \tau}^{\pi_t}(s)  - \frac{\tau}{2} \| \pi^\star_{\tau, t}(s) \|^2 + \frac{\tau}{2} \| \pi_t(s) \|^2.
\end{align}
Using the definition of the approximation error introduced in Assumption \ref{as::approximation_error},  $\delta_{\theta_t} (s, a) \DefinedAs \Tilde{J}_{\theta_t}(s, a) - J^{\pi_t}(s, a)$, we have that
\begin{align}
    \nonumber
    A_{\lambda_t, \tau}^{\pi_t} (s, \pi^\star_{\tau, t}(s)) 
    & \;=\; \Tilde{J}_{\theta_t}(s, \pi^\star_{\tau, t}(s)) - \delta_{\theta_t}(s, \pi^\star_{\tau, t}(s)) - \frac{1}{\eta} \pi_t(s)^\Tr \pi^\star_{\tau, t}(s)  \\
    \nonumber
    &\;\;\;\;\; - V_{\lambda_t, \tau}^{\pi_t}(s)- \frac{\tau}{2} \| \pi^\star_{\tau, t}(s) \|^2 + \frac{\tau}{2} \| \pi_t(s) \|^2.
\end{align}
Now, we can use the quadratic growth Lemma \ref{lm::qg_lemma} and the optimality conditions of the strongly concave objective in \eqref{eq::approx_primal_update}, with $f(a) = \Tilde{J}_{\theta_t}(s, a) - \frac{\tau}{2} \| a \|^2 - \frac{1}{2 \eta} \| a \|^2$ and the maximizer being $a^\star=\pi_{t+1}(s)$ to show that the following bound holds
\begin{align}
    \label{eq::bound_adv_a}
    A_{\lambda_t, \tau}^{\pi_t} (s, \pi^\star_{\tau, t}(s)) 
    & \;=\; \Tilde{J}_{\theta_t}(s, \pi^\star_{\tau, t}(s)) - \delta_{\theta_t}(s, \pi^\star_{\tau, t}(s)) - \frac{1}{\eta} \pi_t(s)^\Tr \pi^\star_{\tau, t}(s)  \\
    \nonumber
    &\;\;\;\; \;\; - V_{\lambda_t, \tau}^{\pi_t}(s)- \frac{\tau}{2} \| \pi^\star_{\tau, t}(s) \|^2 + \frac{\tau}{2} \| \pi_t(s) \|^2 \\
    \nonumber
    &\;\leq \; \Tilde{J}_{\theta_t}(s, \pi_{t + 1}(s)) - \frac{\tau}{2} \| \pi_{t+1}(s) \|^2 - \frac{1}{2 \eta} \| \pi_{t+1}(s) \|^2 \\
    \nonumber
    &\;\;\;\; \;\; - \left( \frac{1 + \eta (\tau - \tau_0) }{ 2 \eta}\right)\|\pi^\star_{\tau, t}(s) - \pi_{t+1}(s)\|^2  \\
    \nonumber
    &\;\;\;\; \;\; + \frac{\tau}{2} \| \pi^\star_{\tau, t}(s) \|^2 + \frac{1}{2 \eta}\| \pi^\star_{\tau, t}(s) \|^2 - \delta_{\theta_t}(s, \pi^\star_{\tau, t}(s))      \\
    \nonumber
    &\;\;\;\; \;\;  - \frac{1}{\eta} \pi_t(s)^\Tr \pi^\star_{\tau, t}(s) - V_{\lambda_t, \tau}^{\pi_t}(s) - \frac{\tau}{2} \| \pi^\star_{\tau, t}(s) \|^2 \\
    \nonumber
    &\;\;\;\; \;\; + \frac{\tau}{2} \| \pi_t(s) \|^2.
\end{align}
Using again the definition of the approximation error introduced in Assumption \ref{as::approximation_error}, we know that $J^{\pi_t}(s, a) \DefinedAs \Tilde{J}_{\theta_t}(s, a) - \delta_{\theta_t} (s, a)$. Then, adding and substracting $\frac{\tau}{2} \| \pi_t(s) \|^2$ we have the following equality

\begin{align}
    \label{eq::bound_adv_b}
    &\Tilde{J}_{\theta_t}(s, \pi_{t + 1}(s)) - \frac{\tau}{2} \| \pi_{t+1}(s) \|^2 - \frac{1}{2 \eta} \| \pi_{t+1}(s) \|^2  - \left( \frac{1 + \eta (\tau - \tau_0) }{ 2 \eta}\right)\|\pi^\star_{\tau, t}(s) - \pi_{t+1}(s)\|^2 \\
    \nonumber
    & + \frac{\tau}{2} \| \pi^\star_{\tau, t}(s) \|^2  + \frac{1}{2 \eta}\| \pi^\star_{\tau, t}(s) \|^2   - \delta_{\theta_t}(s, \pi^\star_{\tau, t}(s)) - \frac{1}{\eta} \pi_t(s)^\Tr \pi^\star_{\tau, t}(s)  \\
    \nonumber
    &  - V_{\lambda_t, \tau}^{\pi_t}(s) - \frac{\tau}{2} \| \pi^\star_{\tau, t}(s) \|^2 + \frac{\tau}{2} \| \pi_t(s) \|^2 \\
    \nonumber
    &= \; J^{\pi_t}(s, \pi_{t + 1}(s)) + \delta_{\theta_t}(s, \pi_{t+1}(s))  - \frac{\tau}{2} \| \pi_{t+1}(s) \|^2 - \frac{1}{2 \eta} \| \pi_{t+1}(s) \|^2 \\
    \nonumber
    & \;\;\;\; \;\; - \left( \frac{1 + \eta (\tau - \tau_0) }{2 \eta}\right)\|\pi^\star_{\tau, t}(s) - \pi_{t+1}(s)\|^2  + \frac{\tau}{2} \| \pi^\star_{\tau, t}(s) \|^2 + \frac{1}{2 \eta}\| \pi^\star_{\tau, t}(s) \|^2 - \delta_{\theta_t}(s, \pi^\star_{\tau, t}(s)) \\
    \nonumber
    & \;\;\;\; \;\; - \frac{1}{\eta} \pi_t(s)^\Tr \pi^\star_{\tau, t}(s) - V_{\lambda_t, \tau}^{\pi_t}(s)- \frac{\tau}{2} \| \pi^\star_{\tau, t}(s) \|^2  + \frac{1}{2\eta}\| \pi_t(s) \|^2 - \frac{1}{2\eta}\| \pi_t(s) \|^2 + \frac{\tau}{2} \| \pi_t(s) \|^2.
\end{align}
We can rearrange some terms in the expression above by realizing that they are expanded differences of quadratics. Therefore, after some cumbersome algebraic manipulations we have that it holds that
\begin{align}
    \label{eq::bound_adv_c}
    &
    J^{\pi_t}(s, \pi_{t + 1}(s)) + \delta_{\theta_t}(s, \pi_{t+1}(s))   - \frac{\tau}{2} \| \pi_{t+1}(s) \|^2 - \frac{1}{2 \eta} \| \pi_{t+1}(s) \|^2 
    \\
    \nonumber
    & - \left( \frac{1 + \eta (\tau - \tau_0) }{2 \eta}\right)\|\pi^\star_{\tau, t}(s) - \pi_{t+1}(s)\|^2 + \frac{\tau}{2} \| \pi^\star_{\tau, t}(s) \|^2 + \frac{1}{2 \eta}\| \pi^\star_{\tau, t}(s) \|^2 - \delta_{\theta_t}(s, \pi^\star_{\tau, t}(s)) \\
    \nonumber
    &  - \frac{1}{\eta} \pi_t(s)^\Tr \pi^\star_{\tau, t}(s) - V_{\lambda_t, \tau}^{\pi_t}(s)- \frac{\tau}{2} \| \pi^\star_{\tau, t}(s) \|^2  + \frac{1}{2\eta}\| \pi_t(s) \|^2 - \frac{1}{2\eta}\| \pi_t(s) \|^2 + \frac{\tau}{2} \| \pi_t(s) \|^2 \\
    \nonumber
    &= \; A_{\lambda_t, \tau}^{\pi_t}(s, \pi_{t+1}(s)) - \frac{1}{2 \eta}\| \pi_t(s) - \pi_{t+1}(s) \|^2  + \frac{1}{2\eta}\| \pi^\star_{\tau, t}(s) - \pi_t(s) \|^2 \\
    \nonumber
    & \;\;\;\; \;\; - \left( \frac{1 + \eta (\tau - \tau_0) }{2 \eta}\right)\|\pi^\star_{\tau, t}(s) - \pi_{t+1}(s)\|^2  + \delta_{\theta_t}(s, \pi_{t+1}(s)) - \delta_{\theta_t}(s, \pi^\star_{\tau, t}(s)).
\end{align}
In summary, combining the bounds in \eqref{eq::bound_adv_a}, \eqref{eq::bound_adv_b}, and \eqref{eq::bound_adv_c} we have that
\begin{align}
    \label{eq:advantage_bound}
    A_{\lambda_t, \tau}^{\pi_t} (s, \pi^\star_{\tau, t}(a))
    & \;\leq\; A_{\lambda_t, \tau}^{\pi_t}(s, \pi_{t+1}(s)) - \frac{1}{2 \eta}\| \pi_t(s) - \pi_{t+1}(s) \|^2  + \frac{1}{2\eta}\| \pi^\star_{\tau, t}(s) - \pi_t(s) \|^2 \\
    \nonumber
    & \;\;\;\; \;\; - \left( \frac{1 + \eta (\tau - \tau_0) }{2 \eta}\right)\|\pi^\star_{\tau, t}(s) - \pi_{t+1}(s)\|^2 + \delta_{\theta_t}(s, \pi_{t+1}(s)) - \delta_{\theta_t}(s, \pi^\star_{\tau, t}(s)).
\end{align}

With this bound in place, we can now decompose the primal-dual gap by adding and substracting $L_\tau(\pi_t, \lambda_t)$ as we did in the proof of Theorem \eqref{thm::convergence_exact} (see Corollary \ref{app::convergence_exact})
\begin{align}
\nonumber
    &L_\tau(\pi^\star_{\tau, t}, \lambda_t) - L_\tau(\pi_t, \lambda^\star_{\tau, t})  \;=\; \underbrace{L_\tau(\pi^\star_{\tau, t}, \lambda_t) - L_\tau(\pi_t, \lambda_t)}_{\text{(i)}} 
    \,+\,
    \underbrace{L_\tau(\pi_t, \lambda_t) - L_\tau(\pi_t, \lambda^\star_{\tau, t})}_{\text{(ii)}}.
\end{align}

Let's define the distance function $D_t(\pi^\star_{\tau, t})= \E_{d^\star_\rho} \left[ \| \pi^\star_{\tau, t}(s) - \pi_t(s) \|^2  \right]$. Now, using the definition of the regularized Lagrangian and the performance difference Lemma \ref{lm::pd_lemma}, we can show that for term (i) we have that
\begin{align}
    \label{eq::pd_inexact_1a}
    L_\tau(\pi^\star_{\tau, t}, \lambda_t) - L_\tau(\pi_t, \lambda_t) 
    & \;=\; V_{\lambda_t, \tau} (\pi^\star_{\tau, t}) - V_{\lambda_t, \tau} ({\pi_t}) \\
    \nonumber
    &\;=\; \frac{1}{1 - \gamma} \E_{d^\star_\rho} \left[ A^{\pi_t}_{\lambda_t, \tau}(s, \pi^\star_{\tau, t}(s)) \right].
\end{align}
Now, using the bound in \eqref{eq:advantage_bound} that we have just derive we have that
\begin{align}
    \label{eq::pd_inexact_1b}
    \E_{d^\star_\rho} \left[ A^{\pi_t}_{\lambda_t, \tau}(s, \pi^\star_{\tau, t}(s)) \right] 
    & \;\leq\; \E_{d^\star_\rho} \left[ A^{\pi_t}_{\lambda_t, \tau}(s, \pi_{t+1}(s)) - \frac{1}{2\eta} \|\pi_{t+1}(s) - \pi_t(s)\|^2 \right] \\
    \nonumber
    & \;\;\;\;\;\; + \E_{d^\star_\rho} \left[\frac{1}{2\eta} \| \pi^\star_{\tau, t}(s) - \pi_t(s)\|^2 \right] \\
    \nonumber
    & \;\;\;\;\;\; - \E_{d^\star_\rho} \left[ \left(\frac{1 + \eta (\tau - \tau_0)}{2 \eta}\right) \|\pi^\star_{\tau, t}(s) - \pi_{t+1}(s)\|^2 \right] \\
    \nonumber
    & \;\;\;\;\;\; + \E_{d^\star_\rho} \left[ \delta_{\theta_t}(s, \pi_{t+1}(s)) - \delta_{\theta_t}(s, \pi^\star_{\tau, t}(s)) \right].
\end{align}
Also, we can use Lemma \ref{lm::error_bound}, Assumption \ref{as::approximation_error} and the triangular inequality to establish that
\begin{align}
    \label{eq::pd_inexact_1c}
    \E_{d^\star_\rho} \left[ \delta_{\theta_t}(s, \pi_{t+1}(s)) - \delta_{\theta_t}(s, \pi^\star_{\tau, t}(s)) \right] 
    &\;\leq\; 2 (2 A_{\text{max}})^{d_a} \E_{s \sim d^\star_\rho, a \sim \mathsf{u}} \left[ \| \delta_{\theta_t}(s, a) \| \right] \\
    \nonumber
    &\;\leq\; \epsilon_{\text{approx}},
\end{align}
where $\mathsf{u}$ is the uniform distribution. Furthermore, we can use Lemma \ref{lm::lipschitz_lemma} with $a=\pi_{t+1}(s)$ to bound the term
\begin{align}
    \label{eq::pd_inexact_1d}
    \E_{d^\star_\rho} \left[ A^{\pi_t}_{\lambda_t, \tau}(s, \pi_{t+1}(s)) - \frac{1}{2\eta} \|\pi_{t+1}(s) - \pi_t(s)\|^2 \right] 
    \;\leq\;
    \frac{\eta}{2} C_P^2.
\end{align}
Finally, combining \eqref{eq::pd_inexact_1a}, \eqref{eq::pd_inexact_1b}, \eqref{eq::pd_inexact_1c}, and \eqref{eq::pd_inexact_1d}, we have that for the term (i) it holds that
\begin{align}
    \nonumber
    &L_\tau(\pi^\star_{\tau, t}, \lambda_t) - L_\tau(\pi_t, \lambda_t) \\
    \nonumber
    &\leq \frac{D_t(\pi^\star_{\tau, t}) - ( 1 + \eta (\tau - \tau_0)) D_{t+1}(\pi^\star_{\tau, t}) }{2 \eta (1 - \gamma)}  + \frac{\eta}{2 (1- \gamma)} C_P^2 + \frac{1}{1 - \gamma} \epsilon_{\text{approx}}.
\end{align}

For the second term (ii) we use the definition of the regularized Lagrangian and the definition of the regularized value function to show that
\begin{align}
    \label{eq::pd_inexact_2a}
    L_\tau(\pi_t, \lambda_t) - L_\tau(\pi_t, \lambda^\star_{\tau, t}) &\;=\; V_{\lambda_t, \tau}({\pi_t}) - V_{\lambda^\star_{\tau, t}, \tau}({\pi_t}) + \frac{\tau}{2} (\lambda_t)^2 - \frac{\tau}{2} (\lambda^\star_{\tau, t})^2 \\
    \nonumber
    & \;=\; (\lambda_t - \lambda^\star_{\tau, t}) V_g (\pi_t) + \frac{\tau}{2} (\lambda_t)^2 - \frac{\tau}{2} (\lambda^\star_{\tau, t})^2.
\end{align}
Then, completing the squares we have that
\begin{align}
    \label{eq::pd_inexact_2b}
    &(\lambda_t - \lambda^\star_{\tau, t}) V_g (\pi_t) + \frac{\tau}{2} (\lambda_t)^2 - \frac{\tau}{2} (\lambda^\star_{\tau, t})^2 \;=\; 
    (\lambda_t - \lambda^\star_{\tau, t}) (V_g (\pi_t) + \tau \lambda_t) - \frac{\tau}{2}  \|\lambda^\star_{\tau, t} - \lambda_t\|^2.
\end{align}
We can now use the standard descent Lemma \ref{lm::descent_lemma} to show that
\begin{align}
    \label{eq::pd_inexact_2c}
    &(\lambda_t - \lambda^\star_{\tau, t}) (V_g (\pi_t) + \tau \lambda_t) - \frac{\tau}{2}  \|\lambda^\star_{\tau, t} - \lambda_t\|^2 \\
    \nonumber
    &\leq\; \frac{\|\lambda^\star_{\tau, t} - \lambda_t \|^2 - \|\lambda^\star_{\tau, t} - \lambda_{t+1} \|^2}{2 \eta} + \frac{1}{2} \eta C_D^2 - \frac{\tau}{2} \| \lambda_t - \lambda^\star_{\tau, t} \|^2 \\
    \nonumber
    &=\; \frac{(1 - \eta \tau)\|\lambda^\star_{\tau, t} - \lambda_t \|^2 - \|\lambda^\star_{\tau, t} - \lambda_{t+1} \|^2}{2 \eta} + \frac{1}{2} \eta C_D^2,
\end{align}
where the last equality rearranges the term $\frac{\tau}{2} \| \lambda_t - \lambda^\star_{\tau, t} \|^2$ inside the fraction. Now, combining the expressions in \eqref{eq::pd_inexact_2a}, \eqref{eq::pd_inexact_2b} and \eqref{eq::pd_inexact_2c} we have that for term (ii) it holds that
\begin{align}
    \nonumber
    &L_\tau(\pi_t, \lambda_t) - L_\tau(\pi_t, \lambda^\star_{\tau, t}) 
    \;=\;
    \frac{(1 - \eta \tau)\|\lambda^\star_{\tau, t} - \lambda_t \|^2 - \|\lambda^\star_{\tau, t} - \lambda_{t+1} \|^2}{2 \eta} + \frac{1}{2} \eta C_D^2.
\end{align}

We define $C_0 \DefinedAs C_P + C_D$ and combine (i) and (ii) to show
\begin{align}
    \nonumber
    0 & \; \leq\;  \eta (1 - \gamma) \left( L_\tau(\pi^\star_{\tau, t}, \lambda_t) - L_\tau(\pi_t, \lambda^\star_{\tau, t})\right) \\
    \nonumber
    & \;\leq\; \frac{1}{2} D_t(\pi^\star_{\tau, t}) + \frac{1 - \eta \tau}{2} \|\lambda^\star_{\tau, t} - \lambda_t \|^2 \\
    \nonumber
    & \;\;\;\;\;\; - \frac{ 1 + \eta (\tau - \tau_0)}{2} D_{t+1}(\pi^\star_{\tau, t}) - \frac{1}{2}\|\lambda^\star_{\tau, t} - \lambda_{t+1} \|^2  + \eta^2 C_0^2 + \eta \epsilon_{\text{approx}},
\end{align}
where the first inequality follows from the fact that the duality gap is positive and holds for any $(\pi_\tau^\star,\lambda_\tau^\star) \in \Pi^\star_\tau \times \Lambda^\star_\tau$. Rearranging the expression above, it follows
\begin{align}
    \label{eq::inequality_contraction}
    &\frac{ 1 + \eta (\tau - \tau_0)}{2} D_{t+1}(\pi^\star_{\tau, t}) + \frac{1}{2}\|\lambda^\star_{\tau, t} - \lambda_{t+1} \|^2 \\
    \nonumber
    &=\; ( 1 + \eta (\tau - \tau_0))  \left( \frac{1}{2} D_{t+1}(\pi^\star_{\tau, t}) + \frac{1}{2(1 + \eta (\tau - \tau_0))}\| \lambda^\star_{\tau, t} - \lambda_{t+1} \|^2 \right) \\
    \nonumber
    &\leq\; \frac{1}{2} D_t(\pi^\star_{\tau, t}) + \frac{1 - \eta \tau}{2}\|\lambda^\star_{\tau, t} - \lambda_t \|^2 + \eta^2 C_0^2 + \eta \epsilon_{\text{approx}}\\
    \nonumber
    &\leq\; \frac{1}{2} D_t(\pi^\star_{\tau, t}) + \frac{1}{2(1 + \eta (\tau - \tau_0))}\|\lambda^\star_{\tau, t} - \lambda_t \|^2 + \eta^2 C_0^2 + \eta \epsilon_{\text{approx}},
\end{align}
where the last inequality follows from the fact that $\tau_0 \in [0, \tau)$, and therefore $\frac{1 - \eta \tau}{2} \leq \frac{1}{2(1 + \eta (\tau - \tau_0))}$. Since $\pi_\tau^\star = \ccalP_{\Pi^\star_\tau}(\pi_t)$ and $\lambda_\tau^\star = \ccalP_{\Lambda^\star_\tau}(\lambda_t)$ are projections, it holds that
\begin{align}
    \label{eq::inequality_projection}
    &\frac{ 1 + \eta (\tau - \tau_0)}{2} D_{t+1}(\pi^\star_{\tau, t + 1}) + \frac{1}{2}\|\lambda^\star_{\tau, t + 1} - \lambda_{t+1} \|^2 \\
    \nonumber
     &=\;\frac{ 1 + \eta (\tau - \tau_0)}{2} \E_{d^\star_\rho} \left[ \|(\ccalP_{\Pi^\star_\tau}(\pi_{t+1}(s)) - \pi_{t+1}(s)\|^2  \right]  + \frac{1}{2}\|\ccalP_{\Lambda^\star_\tau}(\lambda_{t+1}) - \lambda_{t+1} \|^2 \\
    \nonumber
    &\leq\; \frac{ 1 + \eta (\tau - \tau_0)}{2} \E_{d^\star_\rho} \left[ \|(\ccalP_{\Pi^\star_\tau}(\pi_{t}(s)) - \pi_{t+1}(s)\|^2  \right]  + \frac{1}{2}\| \ccalP_{\Lambda^\star_\tau}(\lambda_{t}) - \lambda_{t+1} \|^2 \\
    \nonumber
    &=\; \frac{ 1 + \eta (\tau - \tau_0)}{2} D_{t+1}(\pi^\star_{\tau, t}) + \frac{1}{2}\|\lambda^\star_{\tau, t} - \lambda_{t+1} \|^2.
\end{align}

We define the potential function $\Phi_t \DefinedAs \frac{1}{2} D_t(\pi_{\tau, t}^\star) + \frac{1}{2(1 + \eta (\tau - \tau_0))} \| \ccalP_{\Lambda^\star_\tau}(\lambda_{t}) - \lambda_t \|^2$, and combine the inequalities \eqref{eq::inequality_projection} and \eqref{eq::inequality_contraction} to show
\begin{equation}
    \label{eq::potential_contraction_inexact}
    (1 + \eta (\tau - \tau_0)) \Phi_{t+1} 
    \;\leq\;
    \Phi_t + \eta^2 C_0^2 + \eta \epsilon_{\text{approx}}.
\end{equation}

The next step is to expand \eqref{eq::potential_contraction_inexact} so that we have
\begin{align}
    \nonumber
    \Phi_{t+1} & \;\leq\; \frac{1}{1 + \eta (\tau - \tau_0)} \Phi_t + \eta^2 C_0^2 + \eta \epsilon_{\text{approx}} \\
    \nonumber
    &\;\leq\; \left(\frac{1}{1 + \eta (\tau - \tau_0)} \right)^2 \Phi_{t-1}  + \left(\eta^2 + \eta^2 \left(\frac{1}{1 + \eta (\tau - \tau_0)}\right) \right) \left( C_0^2 + \frac{\epsilon_{\text{approx}}}{\eta}  \right).
\end{align}
If we keep expanding recursively we end up having
\begin{align}
    \nonumber
    \Phi_{t+1}  &\;\leq\; \left(\frac{1}{1 + \eta (\tau - \tau_0)} \right)^t \Phi^{1}   + \left(\eta^2 \left( 1 + \left(\frac{1}{1 + \eta (\tau - \tau_0)}\right) + \hdots \right) \right) \left( C_0^2 + \frac{\epsilon_{\text{approx}}}{\eta}  \right).
\end{align}
Finally, using the definition of the exponential for the first term, and using the geometric series for the second one, we have that
\begin{align}
    \nonumber
    &\Phi_{t+1} \leq {\rm e}^{-\frac{\eta (\tau - \tau_0)}{1 + \eta (\tau - \tau_0) }t} \Phi_1  + \frac{\eta (1 + \eta (\tau - \tau_0))}{\tau - \tau_0} C_0^2 + \frac{1 + \eta (\tau - \tau_0)}{\tau - \tau_0} \epsilon_{\text{approx}}
\end{align}
which completes the proof.

\end{proof}

\subsection{Proof of Corollary \ref{cor::convergence_sampled}}
\label{app::convergence_sampled}

\begin{proof}
    
The proof is an extension of the proof of Theorem \ref{thm::convergence_inexact}. However, we need to take into account the randomness of estimating $\theta_t$ using samples. As the updates in \eqref{eq::sgd_update} are performed using projected SGD, we can leverage known error bounds \cite{lacoste2012simpler} to bound the expected estimation error 
\[
    \delta(\theta_t, \hat{\theta}_t, \nu) \DefinedAs \E_\nu \left[ \delta_{\theta_t^{(n)}}(s, a) - \delta_{\theta_t}(s, a) \right].
\]

With that goal in mind, we need to check:
\begin{enumerate}
    \item[(i)] The domain $\| \theta \| \leq \theta_{\text{max}}$ is convex and bounded.
    \item[(ii)] The gradient $g_t^{(n)}$ is unbiased since $\hat{Q}^{\pi_t}_{\lambda, \tau}$ is an unbiased estimate of $Q^{\pi_t}_{\lambda, \tau}$.
    \item[(iii)] The minimizer of \eqref{eq::policy_evaluation_approx} is unique since Assumption \ref{as::positive_covariance} guarantees that $\Sigma_\nu \geq \kappa_0 I$ for some $\kappa_0 > 0$.
    \item[(iv)] The squared norm of the gradient $g_t^{(n)}$ is bounded.
\end{enumerate}

To prove this last point we use the Cauchy-Schwartz inequality twice to bound the norm of the subgradient as
\begin{align}
\nonumber
    \| g_t^{(n)} \|^2 
    &\;=\; 4 \|  \left( \phi(s_n, a_n)^\Tr \theta_t^{(n)} - \hat{J}^{\pi_t}(s_n, a_n) \right) \phi(s_n, a_n) \|^2 \\
    \nonumber
    &\;\leq\; 4 \|  \phi(s_n, a_n)^\Tr \phi(s_n, a_n) \theta_t^{(n)} \|^2  + 4 \| \hat{J}^{\pi_t}(s_n, a_n) \phi(s_n, a_n) \|^2 \\
    \nonumber
    &\;\leq\; 4 \| \phi(s_n, a_n)^\Tr \phi(s_n, a_n) \|^2 \| \theta_t^{(n)} \|^2  + 4 \| \hat{J}^{\pi_t}(s_n, a_n) \|^2 \| \phi(s_n, a_n) \|^2.
\end{align}
As the feature basis function $\phi$ is bounded by Assumption \ref{as::bounded_feature_basis}, we have that
\begin{align}
    \nonumber
    &\| g_t^{(n)} \|^2 
    \;\leq\; 
    4 \left( \| \theta_t^{(n)} \|^2 + \| \hat{Q}^{\pi_t}_{\lambda, \tau} (s_n, a_n)\|^2 + \frac{1}{\eta^2} \| \pi_t(s_n)^\Tr a_n \|^2 \right).
\end{align}
Finally, we can bound $\| \hat{Q}^{\pi_t}_{\lambda, \tau} (s_n, a_n)\|^2$, since $r_{\lambda, \tau} \leq \frac{2}{(1 - \gamma)\xi}$, so that we have that
\begin{align}
    \| g_t^{(n)} \|^2 
    &\;\leq\; 4 \left(  \theta_{\text{max}}^2 + \left( \frac{2}{(1 - \gamma)^2 \xi} \right)^2 + \frac{1}{\eta^2} d_a^2 A_{\text{max}}^4 \right) \\
    \nonumber
    &\;\leq\; 4 \left( \theta_{\text{max}} + \frac{2}{(1 - \gamma)^2 \xi} + \frac{1}{\eta} d_a A_{\text{max}}^2 \right)^2.
\end{align}

Taking the step-size of the projected SGD step to be $\alpha_n = \frac{2}{\kappa_0(k + 2)}$ \citet{lacoste2012simpler}, it follows that
\begin{equation}
    \label{eq::bound_est_error}
    \E \left[\delta(\theta_t, \hat{\theta}_t, \nu) \right] 
    \;\leq\; 
    \frac{16 \left( \theta_{\text{max}} + \frac{2}{(1 - \gamma)^2 \xi} + \frac{1}{\eta} d_a A_{\text{max}}^2 \right)^2 }{\kappa_0^2 (N + 1)},
\end{equation}
where $N$ is the terminal step, and the expectation is taken over the randomness of $\hat{\theta}_t$. The bound in \eqref{eq::bound_est_error} gives as an expression to deal with the estimation error induced by using sample-based methods to estimate the parameters. Then, from the proof of Theorem \ref{thm::convergence_inexact} in Appendix \ref{app::convergence_inexact} we have that
\begin{align}
    \nonumber
    &L_\tau(\pi^\star_{\tau, t}, \lambda_t) - L_\tau(\pi_t, \lambda_t) \\
    \nonumber
    &\leq \;
    \frac{D_t(\pi^\star_{\tau, t}) - ( 1 + \eta (\tau - \tau_0)) D_{t+1}(\pi^\star_{\tau, t}) }{2 \eta (1 - \gamma)}  + \frac{\eta}{2 (1- \gamma)} C_P^2  + \frac{1}{1 - \gamma}\E_{d^\star_\rho} \left[ \delta_{\hat{\theta}_t}(s, \pi_{t+1}(s)) - \delta_{\hat{\theta}_t}(s, \pi^\star_{\tau, t}(s)) \right].
\end{align}
We want to leverage this new bound in the estimation error to analyse the duality gap. We begin by using Lemma \ref{lm::error_bound} to show that
\begin{align}
    \nonumber
    &\E_{d^\star_\rho} \left[ \delta_{\hat{\theta}_t}(s, \pi_{t+1}(s)) - \delta_{\hat{\theta}_t}(s, \pi^\star_{\tau, t}(s)) \right] 
    \;\leq\;
    2(2 A_{\text{max}})^{d_a} \E_{s \sim d^\star_\rho, a \sim \mathsf{u}} \left[ \delta_{\hat{\theta}_t}(s, a) \right].
\end{align}
where $\mathsf{u}$ is the uniform distribution. Using now Assumption \ref{as::est_error}we have that 
\begin{align}
    \nonumber
    &2(2 A_{\text{max}})^{d_a} \E_{s \sim d^\star_\rho, a \sim \mathsf{u}} \left[ \delta_{\hat{\theta}_t}(s, a) \right] 
    \;\leq\;
    2(2 A_{\text{max}})^{d_a} L_\nu \E_{s, a \sim \nu} \left[ \delta_{\hat{\theta}_t}(s, a) \right].
\end{align}
Then, we can add and substract the bias error $\delta_{\theta_t}$ to show that
\begin{align}
    \nonumber
    2(2 A_{\text{max}})^{d_a} L_\nu \E_{s, a \sim \nu} \left[ \delta_{\hat{\theta}_t}(s, a) \right] 
    &= 2(2 A_{\text{max}})^{d_a} L_\nu  \left( \E_{s, a \sim \nu} \left[ \delta_{\hat{\theta}_t}(s, a) - \delta_{\theta_t}(s, a) \right] + \E_{s, a \sim \nu} \left[ \delta_{\theta_t}(s, a) \right] \right).
\end{align}
Finally, by leveraging Assumption \ref{as::bias_error}, we can bound the bias error as
\begin{align}
    \nonumber
    &2(2 A_{\text{max}})^{d_a} L_\nu  \left( \E_{s, a \sim \nu} \left[ \delta_{\hat{\theta}_t}(s, a) - \delta_{\theta_t}(s, a) \right] + \E_{s, a \sim \nu} \left[ \delta_{\theta_t}(s, a) \right] \right) \\
    \nonumber
    &\leq \; 2(2 A_{\text{max}})^{d_a} L_\nu \delta(\theta_t, \hat{\theta}_t, \nu) + \epsilon_{\text{bias}}.
\end{align}
Therefore, it holds that
\begin{align}
    \nonumber
    &\E_{d^\star_\rho} \left[ \delta_{\hat{\theta}_t}(s, \pi_{t+1}(s)) - \delta_{\hat{\theta}_t}(s, \pi^\star_{\tau, t}(s)) \right] 
    \;\leq\;
    2(2 A_{\text{max}})^{d_a} L_\nu \delta(\theta_t, \hat{\theta}_t, \nu) + \epsilon_{\text{bias}}.
\end{align}

The rest of the proof follows the steps of the proof of Theorem \ref{thm::convergence_inexact}, but accounting for the randomness of $\hat{\theta}_t$. We refer the reader to  Appendix \ref{app::convergence_inexact} for the exact details. For dealing with the randomness of $\hat{\theta}_t$ we take expectations of
\begin{align}
    \nonumber
    &\E [ L_\tau(\pi^\star_{\tau, t}, \lambda_t) ] - \E [ L_\tau(\pi_t, \lambda_t) ] \\
    \nonumber
    &\leq\; \frac{ \E [ D_t(\pi^\star_{\tau, t}) ] - ( 1 + \eta (\tau - \tau_0)) \E [ D_{t+1} (\pi^\star_{\tau, t}) ] }{2 \eta (1 - \gamma)}  + \frac{\eta}{2 (1- \gamma)} C_P^2  + \E [ \delta(\theta_t, \hat{\theta}_t, \nu) ] + \epsilon_{\text{bias}} \\
    \nonumber
    &\leq\; \frac{ \E [ D_t(\pi^\star_{\tau, t}) ] - ( 1 + \eta (\tau - \tau_0)) \E [ D_{t+1} (\pi^\star_{\tau, t}) ] }{2 \eta (1 - \gamma)} + \frac{\eta}{2 (1- \gamma)} C_P^2  + \frac{C_1^2}{\eta^2  (N + 1)} + \epsilon_{\text{bias}},
\end{align}
where the second inequality uses the bound in \eqref{eq::bound_est_error} and considers the step-size to be $\eta \leq 1$ , with the constant $C_1 \DefinedAs \sqrt{32 (2 A_{\text{max}})^{d_a}  L_\nu} ( \theta_{\text{max}} + \frac{2}{(1 - \gamma)^2 \xi} + d_a A_{\text{max}}^2 ) \kappa_0^{-1} $. We also have
\begin{align}
    \nonumber
    &\E [ L_\tau(\pi_t, \lambda_t)] - \E [ L_\tau(\pi_t, \lambda^\star_{\tau, t}) ]  \;\leq\; \frac{(1 - \eta \tau) \E [ \|\lambda^\star_{\tau, t} - \lambda_t \|^2] - \E [\|\lambda^\star_{\tau, t} - \lambda_{t+1} \|^2] }{2 \eta}  + \frac{1}{2} \eta C_D^2.
\end{align}

We define the constants $C_0 \DefinedAs C_P + C_D$, and introduce the potential function $\E [ \Phi_t ] \DefinedAs \frac{1}{2} \E [ D_t(\pi_{\tau, t}^\star) ] + \frac{1}{2(1 + \eta (\tau - \tau_0))} \E [ \| \ccalP_{\Lambda^\star_\tau}(\lambda_{t}) - \lambda_t \|^2]$. Combining both inequalities we can show
\begin{align}
    \label{eq::potential_contraction_sampled}
    &(1 + \eta (\tau - \tau_0)) \E [ \Phi_{t+1} ] 
    \;\leq\; \E [ \Phi_t ] + \eta^2 C_0^2 + \frac{C_1^2}{\eta^2  (N + 1)} + \eta \epsilon_{\text{bias}}.
\end{align}

Expanding \eqref{eq::potential_contraction_sampled} recursively we can show
\begin{align}
    \nonumber
    &\Phi_{t+1} 
    \;\leq\;
    {\rm e}^{-\frac{\eta (\tau - \tau_0)}{1 + \eta (\tau - \tau_0) }t} \Phi_1 + \frac{\eta (1 + \eta (\tau - \tau_0))}{\tau - \tau_0} C_0^2   + \frac{1 + \eta (\tau - \tau_0)}{\tau - \tau_0} \left(  \frac{C_1^2}{\eta^2  (N + 1)} + \epsilon_{\text{bias}} \right)
\end{align}
which completes the proof.
\end{proof}

\section{Control Regulation Problem}

\subsection{Lipschitz Continuity of Action-value Function}

\label{app::lipschitz_action_value}

We aim to show that the action-value function $Q_r^\pi$ for reward function $r$ is Lipschitz-continuous. For any state $s \in S$, and two different actions $a_1, a_2 \in A$, we can expand
\begin{align}
    \nonumber
    \| Q_r^\pi(s, a_1) - Q_r^\pi(s, a_1) \| 
    & \;=\; 
    \left\| r(s, a_1) + \gamma \E_{s' \sim p(\cdot | s, a_0)} \left[ V_r^\pi(s')\right]  - r(s, a_2) - \gamma \E_{s' \sim p(\cdot | s, a_2)} \left[ V_r^\pi(s')\right] \right\| \\
    \nonumber
    & \;=\; \left\| r(s, a_1) - r(s, a_2) + \gamma \int_S \left( p(s' | s, a_1) - p(s' | s, a_2) \right) V_r^\pi(s') ds'\right\|.
\end{align}

We know that $V_r^\pi(s')$ is bounded in the interval $[0, 1/(1 - \gamma)]$. Therefore, Lipschitz continuity of $Q_r^\pi$ is guaranteed if the terms $\| r(s, a_1) - r(s, a_2) \|$ and $\| p(s' | s, a_1) - p(s' | s, a_2) \|$ are Lipschitz continuous.

A function is Lipschitz continuous if it is continuously differentiable over a compact domain. The Lipschitz constant is equal to the maximum magnitude of the derivative. Consider now the example in Section \ref{ss::clq}. The reward function $r$ is a Lipschitz-continuous quadratic function with Lipschitz constant equal to the maximum eigenvalue of the matrix $R$. Now, recalling the linear Gaussian structure of the transition dynamics, we have
\begin{align}
    \label{eq::kernel}
    &p(s' | s, a_1) - p(s' | s, a_2) 
    \;\approx\;
    \exp \{ - \| s' - \mu(s, a_1) \|^2 \} - \exp \{ - \| s' - \mu(s, a_2) \|^2 \},
\end{align}
where the mean of the distribution is given by $\mu(s, a) \DefinedAs B_0 s + B_1 a$. It is evident that the difference of the Gaussian kernels in \eqref{eq::kernel} is continuously differentiable over $A$, and therefore, Lipschitz continuous. This implies that the action-value function $Q_r^\pi$ associated with the example defined in Section \ref{ss::clq} is Lipschitz continuous. A similar reasoning applies to the action-value functions $Q_g^\pi$ and $H^\pi$.

\section{Algorithms}

\label{app::algorithms}

\begin{algorithm}[H]
\caption{Unbiased estimate V}
\label{alg::estimate_V}
    \begin{algorithmic}[1]
    \Require Simulator $E$, policy $\pi$, initial distribution $\rho$, discount factor $\gamma$.
    
    \State $s_0 \sim \rho$. 
    \State $a_0 \gets \pi(s_0)$.
    \State $\tilde{V} \gets 0$.
    \State $T \sim \text{geom}(1 - \gamma)$.
    \For{$t = 1$ to $T$}
        \State $s_{t+1}, r_t \gets E(s_t, a_t)$.
        \State $\tilde{V} \gets \tilde{V} + r_t$.
        \State $a_t \gets \pi(s_{t+1})$. 
        \State $s_t \gets s_{t+1}$. 
    \EndFor
    \State \Return $(1 - \gamma) \tilde{V}$.
    \end{algorithmic}
\end{algorithm}

\begin{algorithm}[H]
\caption{Unbiased estimate Q}
\label{alg::estimate_Q}
    \begin{algorithmic}[1]
    \Require Simulator $E$, policy $\pi$, state $s$, action $a$, discount factor $\gamma$.
    \State $s_0 \gets s$.
    \State $a_0 \gets a$.
    \State $\tilde{Q} \gets 0$.
    \State $T \sim \text{geom}(1 - \gamma)$.
    \For{$t = 1$ to $T$} 
        \State $s_{t+1}, r_t \gets E(s_t, a_t)$. 
        \State $\tilde{Q} \gets \tilde{Q} + r_t$.
        \State $a_t \gets \pi(s_{t+1})$. 
        \State $s_t \gets s_{t+1}$.
    \EndFor
    \State \Return $(1 - \gamma) \tilde{Q}$.
    \end{algorithmic}
\end{algorithm}

\begin{algorithm}[H]
\caption{Sample-based AD-PGPD}
\label{alg::sample_adpgpd}
    \begin{algorithmic}[1]
    \Require Number of iterations $T$, number of SGD iterations $K$, step-size hyper-parameter $\eta$, sampling distribution $\nu$, initial state distribution $\rho$.
    
    \For{$t = 1$ to $T$}
        \State Initialize $\theta_t^{(0)}=0$.
        \For{$n = 1$ to $N$}
            \State Sample $s_n, a_n \sim \nu$.
            \State Estimate $\hat{Q}_{\lambda, \tau}^{\pi_t}(s_n,a_n)$ using Algorithm \ref{alg::estimate_Q}.
            \State Compute $\theta_t^{(n)}$ using \eqref{eq::sgd_update} with $\alpha_n = \frac{1}{2 \kappa_0(n + 2)}$.
        \EndFor
        \State Set $\hat{\theta}_t = \frac{2}{N(N + 1)} \sum_{n=0}^{N -1}(n + 1)\theta_t^{(n)}$.
        \State Estimate $\hat{V}_g({\pi_t})$ using Algorithm \ref{alg::estimate_V}.
        \State Perform the sample-based AD-PGPD update
        \begin{align}
            \nonumber
            \pi_{t+1}(s) & \;=\; \argmax_{a \in A}\;  \Tilde{J}_{\hat \theta_t(s,a)} - \left( \frac{\tau}{2} + \frac{1}{2\eta} \right) \| a \|^2 \\
            \nonumber
            \lambda_{t+1} & \;=\; \argmin_{\lambda \in \Lambda}\;  \lambda(\hat{V}_g(\pi_t) + \tau \lambda_t) + \frac{1}{2 \eta} \| \lambda - \lambda_t \|^2.
        \end{align}
    \EndFor
\end{algorithmic}
\end{algorithm}

\section{Additional experiments}

This section outlines the experiments conducted to evaluate the performance of the D-PGPD method. The experiments were executed on a computing cluster powered by an \texttt{AMD Ryzen Threadripper 3970X} processor, featuring a 64-core architecture with 128 threads, and supported by 220 GiB of RAM. The hyperparameters used in the experiments are detailed within this section. For precise implementation specifics, including seeding strategies and initialization procedures, please refer to the accompanying code repository.

\label{app::experiments}

\subsection{Constrained Quadratic Regulation}

\label{app::experiments_navigation}

We have tested our algorithms in a navigation control problem. An agent moves on a horizontal plane, where the linearized dynamics follow the double integrator model with zero-mean Gaussian noise. The general goal is to drive the agent to the origin. The state $s$ has four dimensions, the $2$-dimensional position $p_x$ and $p_y$ and the $2$-dimensional velocity $v_x$ and $v_y$. The control action is the $2$-dimensional acceleration $u_x$ and $u_y$. The linearized dynamics of the double integrator model used in the experiment are characterized by

\begin{align}
\nonumber
    &B_0 \;=\; \begin{bmatrix}
        1& \;\;  0& \; T_s& \; 0& \\
        0& \;\;  1& \; 0& \; T_s& \\
        0& \;\;  0& \;\;  1& \; 0& \\
        0& \;\;  0& \;\;  0& \; 1&
    \end{bmatrix}  
    \;\; \text{ and } \;\;
    B_1 \;=\; \begin{bmatrix}
        \frac{T_s^2}{2}& \;\;  0& \\
        0& \;\;  \frac{T_s^2}{2}& \\
        T_s& \;\;  0& \\
        0&\;\;   T_s&
    \end{bmatrix},
\end{align}

\noindent where $T_s = 0.05$. The noise is sampled from a multi-variate zero-mean Gaussian distribution with covariance
\begin{equation}
    \nonumber
    \Sigma \;=\; \begin{bmatrix}
        1.0& \;  0& \; 0& \; 0& \\
        0& \;  1.0& \; 0& \; 0& \\
        0& \;  0& \;  0.1& \; 0& \\
        0& \;  0& \;  0& \; 0.1&
    \end{bmatrix}.
\end{equation}

\noindent \textbf{Quadratic penalty.} We initially consider the dynamics to be known and a reward function which linearly weights a quadratic penalty on the position of the agent and a quadratic penalty on control action. The purpose of this setup, which is reminiscent of the constrained regulation problem \cite{scokaert1998constrained, brunke2022safe}, is to assess the performance of D-PGPD in terms of convergence behaviour, and sensibility to the regularization term $\tau$ and the step-size $\eta$. We introduce a quadratic constraint in the velocity of the robot. More specifically, the agent has to achieve the reference state $s_r = [0, 0, 0, 0]$ while minimizing the control action. The constraint imposes a limit in the expected velocity of the agent. The reward and the utility functions are detailed as follows
\begin{align}
    \nonumber
    &r(s, a) 
    \; = \; s^\Tr G_1 s + a^\Tr R_1 a 
    \;\; \text{ and }\;\;u(s, a) 
    \; = \; 
    s^\Tr G_2 s + a^\Tr R_2 a,
\end{align}
where
\begin{align}
\nonumber
    &G_1 \; = \; \begin{bmatrix}
        -1& \;  0& \; 0& \; 0& \\
        0& \;  -1& \; 0& \; 0& \\
        0& \;  0& \;  -0.1& \; 0& \\
        0& \;  0& \;  0& \; -0.1&
    \end{bmatrix} 
    \;\; \text{ and } \;\;
    G_2 \; = \; \begin{bmatrix}
        -0.1& \;  0& \; 0& \; 0& \\
        0& \;  -0.1& \; 0& \; 0& \\
        0& \;  0& \;  -1& \; 0& \\
        0& \;  0& \;  0& \; -1&
    \end{bmatrix},
\end{align}
\noindent and
\begin{align}
\nonumber
     R_1 & \;=\; \begin{bmatrix}
        -0.1& \;\;  0& \\
        0& \;\;  -0.1& \\
    \end{bmatrix} 
    \;\;  
    \text{ and }
    \;\; 
    R_2  \;=\; \begin{bmatrix}
        -0.1& \;\;  0& \\
        0& \;\;  -0.1& \\
    \end{bmatrix}.
\end{align}

For this experiment, the default values of the hyper-parameters are $\tau=0.01$, $\eta=0.01$ and b=$-90$. Fig. \ref{fig::fig_4} shows the value functions of the policy iterates generated by D-PGPD and AD-PGPD over $2,000$ iterations. We have used linear function approximation for AD-PGPD with the basis function being the quadratic $\phi(s,a) \DefinedAs [s^\Tr, a^\Tr]^\Tr \otimes  [s^\Tr, a^\Tr]^\Tr$, where $\otimes$ denotes the Kronecker product. The exact versions of D-PGPD and AD-PGPD achieve very similar results. Both oscillate at the beginning, but oscillations are damped over time. In fact, in the last iterate both algorithms satisfy the constraints. Furthermore, we have tested D-PGPD and AD-PGPD for different values of $\eta$ in Fig. \ref{fig::fig_5}. We can observe empirically that the hyper-parameter $\eta$ trades-off convergence speed and amplitude of oscillations. Larger values of $\eta$ lead to faster convergence. However, if $\eta$ is too large, both D-PGPD and AD-PGPD keep oscillating and do not converge, thus violating the constraint in several episodes. This aligns with the expectations set forth by Theorems \ref{thm::convergence_exact} and \ref{thm::convergence_inexact}, which establish that D-PGPD and A-PGPD converge linearly to a neighborhood whose size is determined by the value of $\eta$. Specifically, larger values of $\eta$ result in a larger convergence neighborhood, causing D-PGPD to oscillate around this region. In Fig. \ref{fig::fig_6}, the value functions associated with policies generated by D-PGPD are shown for different values of $\tau$. Increasing $\tau$ does not affect the convergence rate of D-PGPD, as stated in Theorems \ref{thm::convergence_exact} and \ref{thm::convergence_inexact}, given $\eta$ is sufficiently small (in this case, $\eta=0.01$). However, it impacts the sub-optimality of the solution, as predicted by Corollaries \ref{cor::near_optimality} and \ref{cor::near_optimality_approx}. Larger values of $\tau$ result in worse objective value functions. Nonetheless, the resultant regularized problem is more restrictive with respect to the utility value function of the original problem, resulting in policies generated by D-PGPD that are sub-optimal but satisfy the constraint. Finally, Fig. \ref{fig::fig_7} empirically assesses the convergence rates of D-PGPD and AD-PGPD as proposed in Theorems \ref{thm::convergence_exact} and \ref{thm::convergence_inexact}. We measure the primal optimality gap by the difference between the optimal objective value function $V_r^{\pi^\star}$ and the policy iterates of D-PGPD and AD-PGPD. Since the true $V_r^{\pi^\star}$ is unknown, we estimate it by running D-PGPD with hyperparameters $\eta=0.0001$ and $\tau=0.0001$ for $T=100,000$ episodes. The algorithms exhibit two regimes: an initial linear convergence to a neighborhood of the optimal solution, followed by a regime where the convergence rate changes within this neighborhood. This observation is consistent with theoretical predictions, which only guarantee linear convergence to a neighborhood of the optimal solution. Additionally, the convergence neighborhood sizes for D-PGPD and AD-PGPD are similar, which is a consequence of having a small approximation error $\epsilon_{\text{approx}}$.

\begin{figure}[t]
    \centering
    \includegraphics[width=10cm]{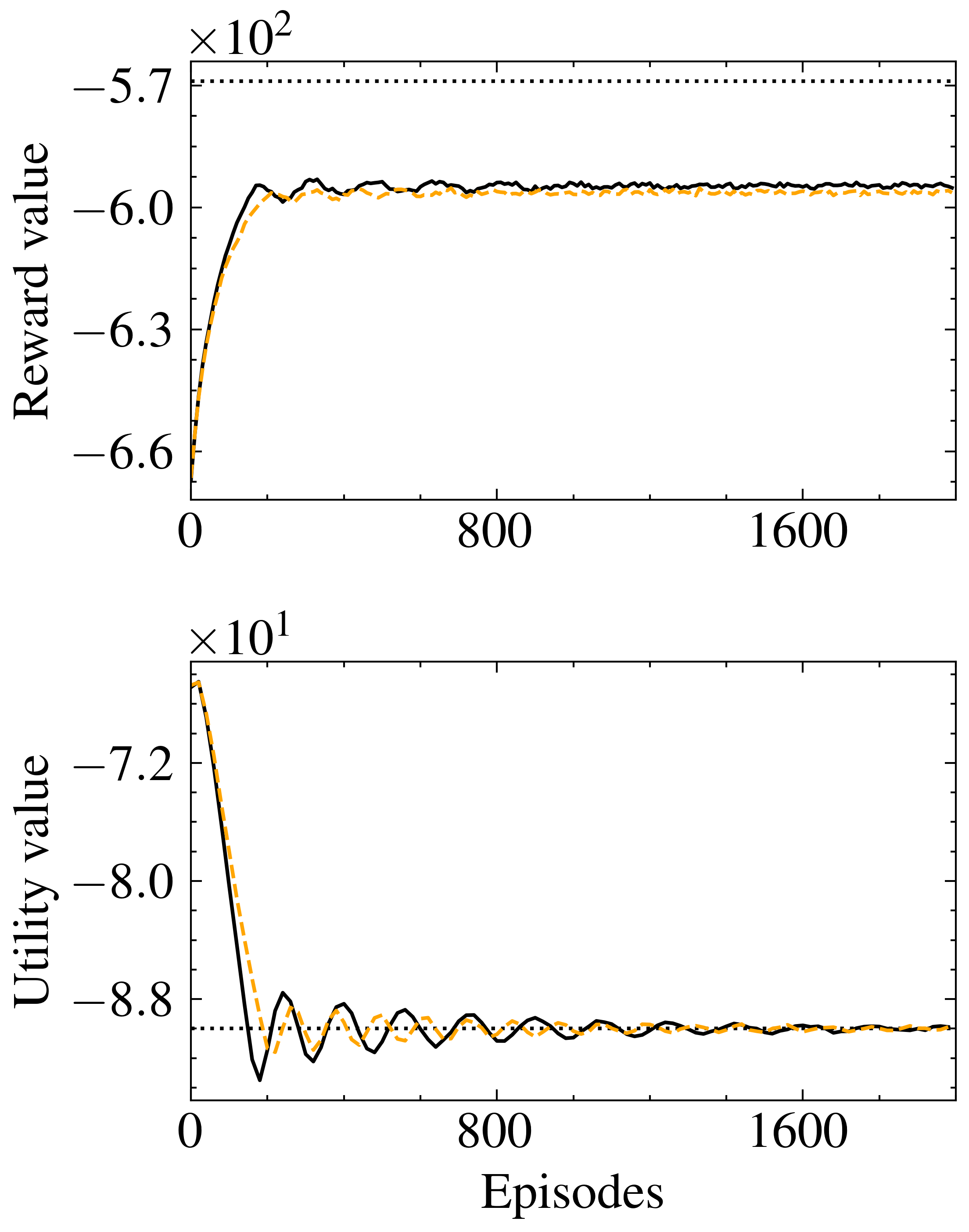}
    \caption{Reward and utility value functions of policy iterates generated by D-PGPD (\blackline) and AD-PGPD (\yellowdashedline )  in the navigation control problem with quadratic rewards.}
    \label{fig::fig_4}
\end{figure}

\begin{figure}[t]
    \centering
    \includegraphics[width=1.0\textwidth]{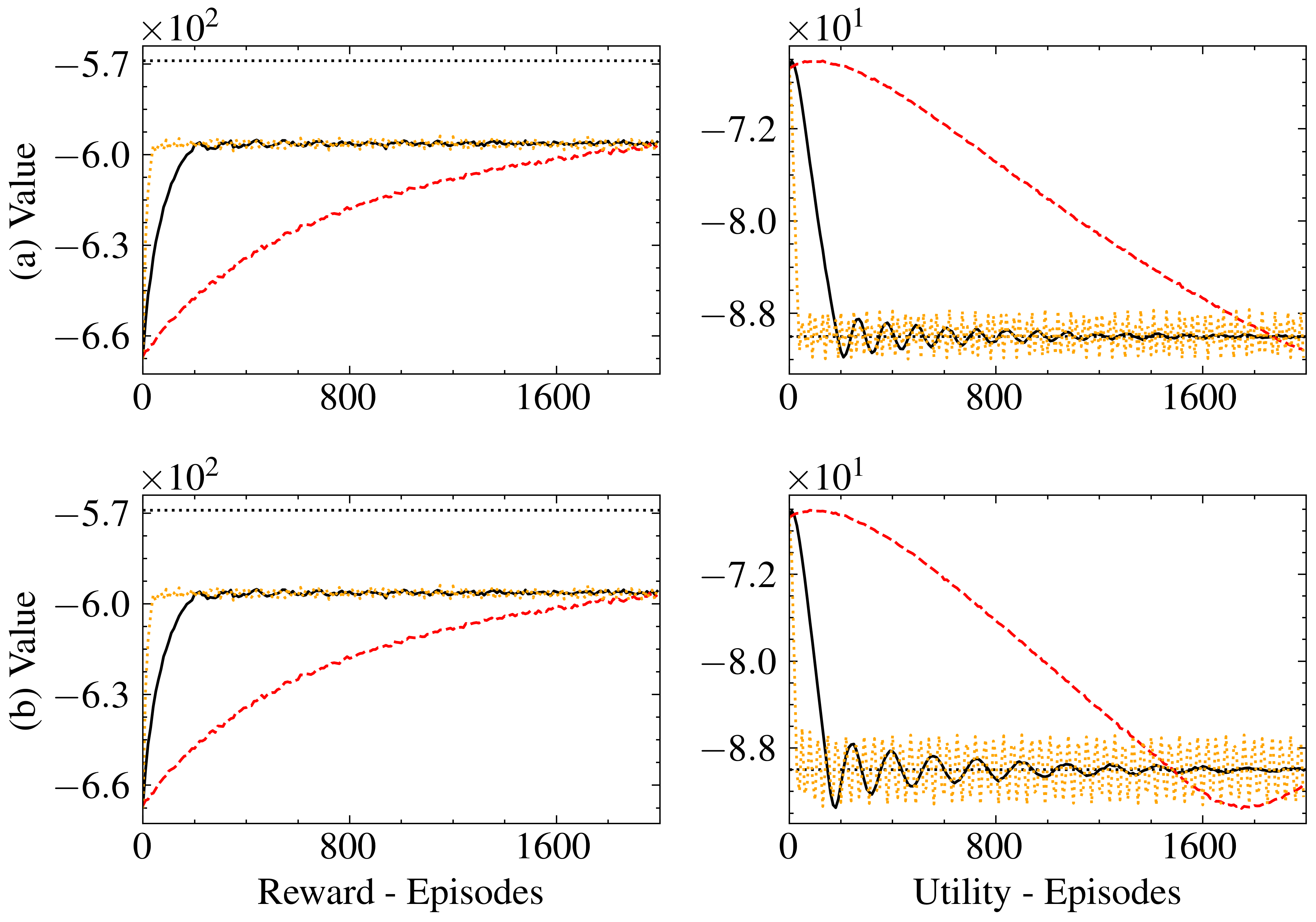}
    \caption{Reward value and utility value for (a) D-PGPD and (b) AD-PGPD in the navigation control problem for different values of the step-size: $\eta=0.1$ (\yellowdashedline), $\eta=0.01$ (\blackline) and $\eta=0.001$ (\reddashedline).}
    \label{fig::fig_5}
\end{figure}

\begin{figure}[t]
    \centering
    \includegraphics[width=1.0\textwidth]{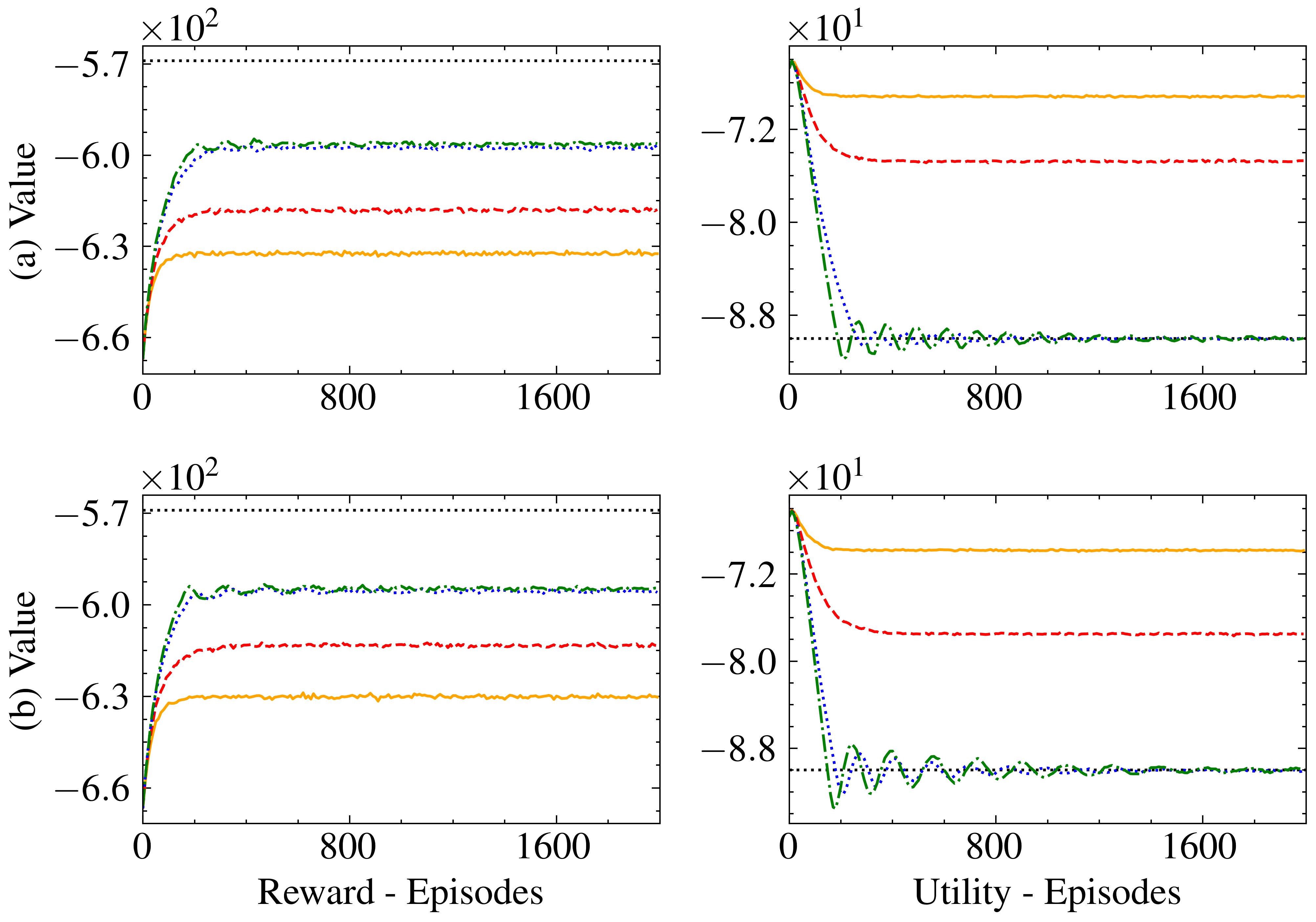}
    \caption{Reward value and utility value for (a) D-PGPD and (b) AD-PGPD in the navigation control problem for different values of the regularization parameter $\tau=1.0$ (\yellowline), $\tau=0.5$ (\reddashedline), $\tau=0.1$ (\bluedotted),  and $\eta=0.01$ (\greenmixed).}
    \label{fig::fig_6}
\end{figure}

\begin{figure}[t]
    \centering
    \includegraphics[width=10cm]{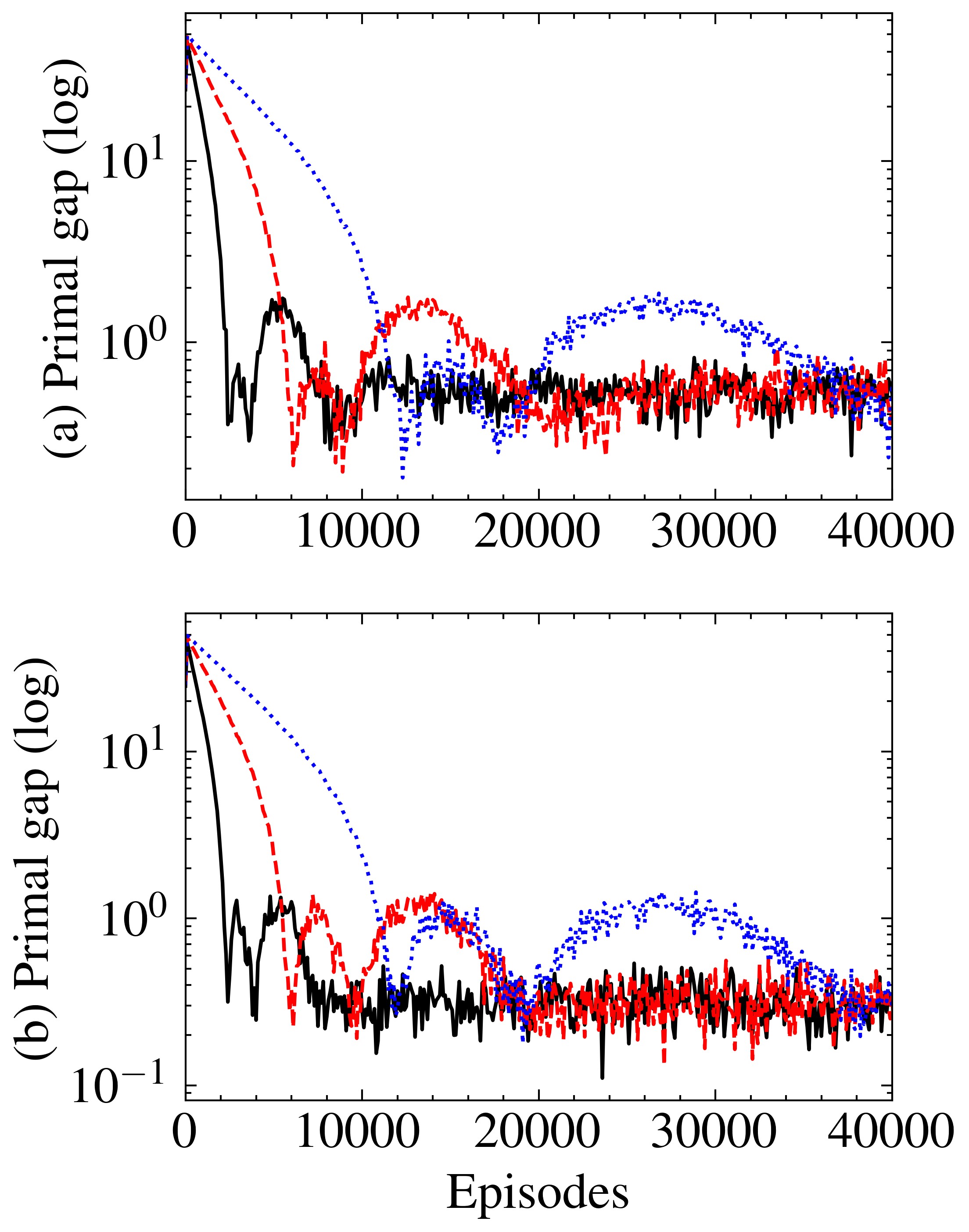}
    \caption{Reward value gap between optimal policy and the iterates of (a) D-PGPD and (b) AD-PGPD $\|V_r(\pi_\tau^\star) - V_r(\pi_t)\|$ in the navigation control problem for different values of the step-size: $\eta=0.0005$ (\blackline), $\eta=0.0002$ (\reddashedline) and $\eta=0.0001$ (\bluedotted).}
    \label{fig::fig_7}
\end{figure}

\noindent \textbf{Absolute-valued penalty.} D-PGPD allows tackling problems that are outside the scope of classical constrained quadratic problems with known dynamics. First of all, we consider the dynamics to be unknown, hence we need to approximate value functions via rollout averages. Therefore, we need to leverage the sample-based formulation of AD-DPPG to address the problem. Second, we consider the reward functions to in absolute value. In fact, absolute-valued penalties are preferred in sample-based scenarios since quadratic penalties can lead to unstable behaviors \cite{engel2014line}. We consider the same setup, but now the reward functions are defined as
\begin{align}
    \nonumber
    &r(s, a) \;=\; \| g_1 * s \|_1 + \| m_1  * a \|_1 
    \;\; \text{ and } \;\;
    u(s, a) \;=\; \| g_2 * s \|_1 + \| m_2  * a \|_1,
\end{align}
where $*$ denotes the Hadamard product, and
\begin{equation}
\nonumber
    g_1 \;=\; \begin{bmatrix}
        -1& \\
        -1& \\
        -0.001& \\
        -0.001&
    \end{bmatrix} 
    \;\; \text{ and } \;\;
    g_2 \;=\; \begin{bmatrix}
        -0.001& \\
        -0.001& \\
        -1& \\
        -1&
    \end{bmatrix},
\end{equation}
\noindent and
\begin{equation}
\nonumber
    m_1 \;=\; \begin{bmatrix}
        -0.01& \\
        -0.01&
    \end{bmatrix} 
    \;\; \text{ and } \;\;
    m_2 \;=\; \begin{bmatrix}
        -0.01& \\
        -0.01&
    \end{bmatrix}.
\end{equation}

The default values of the hyper-parameters of this experiment are $\tau=0.2$, $\eta=0.0001$ and b=$-1000$. Again, we have used linear function approximation with the basis function  $\phi(s,a) \DefinedAs [s^\Tr, a^\Tr]^\Tr \otimes  [s^\Tr, a^\Tr]^\Tr$. We have compared the sample-based AD-PGPD against PGDual, that implements a dual method with a linear model as a function-approximator \citep{zhao2021primal, brunke2022safe}. Fig. \ref{fig::fig_8} shows the average value functions of policy iterates generated by the sample-based version of AD-PGPD and PGDual over $40,000$ iterations, averaged across $50$ experiments. As observed in the previous experiment, the oscillations of AD-PGPD are damped over time, and it converges to a feasible solution. The low variance of the reward and utility value functions indicates that the policies generated by AD-PGPD exhibit near-deterministic behavior and do not violate the constraints. On the other hand, PGDual fails to dampen the oscillations, as evidenced by the large variance in the reward and utility value functions. This implies that PGDual outputs policies that violate the constraints in several episodes. PGDual requires more episodes to reach a solution, but its final performance in terms of primal return is similar to that of AD-PGPD.

\begin{figure}
    \centering
    \includegraphics[width=10cm]{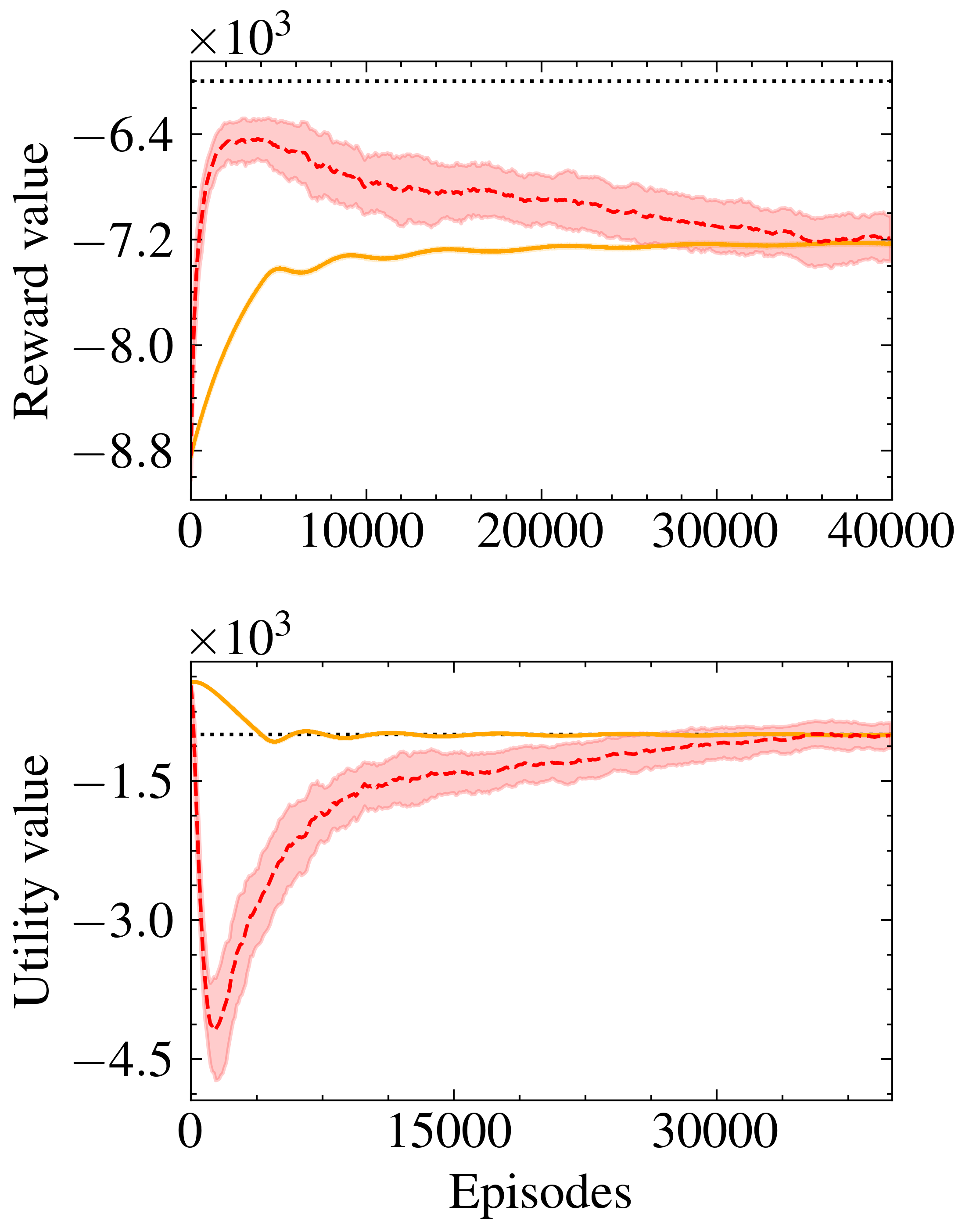}
    \caption{Average reward and utility value functions of policy iterates generated by AD-PGPD (\yellowline) and PGDual (\reddashedline ) in the navigation control problem with absolute-value rewards.}
    \label{fig::fig_8}
\end{figure}

\noindent \textbf{Zone control.} The sample-based AD-PGPD also allows us to deal with non-conventional reward functions. In this third setup, we solve a zone control problem. The agent has to achieve the reference state $s_r = [0, 0, 0, 0]$ while minimizing the control action. However, now the constraint imposes that the agent has to remain in the positive orthant. The reward and the utility functions are detailed as follows
\begin{align}
    \nonumber
    &r(s, a) \;=\; s^\Tr G s + a^\Tr R a \\
    \nonumber
    &u(s, a) \; = \; \begin{cases}
        0 \;\; &\text{if} \;\; p_x \geq 0 \;\; \text{and} \;\; p_y \geq 0 \\
        -100 \;\; &\text{otherwise},
    \end{cases}
\end{align}
\noindent where
\begin{align}
\nonumber
    &G \; = \; \begin{bmatrix}
        -1& \;  0& \; 0& \; 0& \\
        0& \;  -1& \; 0& \; 0& \\
        0& \;  0& \;  -0.1& \; 0& \\
        0& \;  0& \;  0& \; -0.1&
    \end{bmatrix} 
    \;\; \text{ and } \;\; 
    R \; = \; \begin{bmatrix}
        -0.1& \;\;  0& \\
        0& \;\;  -0.1& \\
    \end{bmatrix}.
\end{align}

For this experiment the default values are $\tau=0.01$, $\eta=0.00005$ and $b=-200$. We have employed the same basis function  $\phi(s,a) \DefinedAs [s^\Tr, a^\Tr]^\Tr \otimes  [s^\Tr, a^\Tr]^\Tr$. Again, we have compared the sample-based AD-PGPD against PGDual. Fig. \ref{fig::fig_9} illustrates the value functions of policy iterates generated by the sample-based versions of AD-PGPD and PGDual over $T=50,000$ iterations, averaged across $50$ experiments. The AD-PGPD algorithm effectively dampens oscillations, converging to a feasible solution with low variance in the returns, indicating near-deterministic behavior. Conversely, PGDual not only exhibits a slower convergence rate but also continues to oscillate, frequently violating constraints as evidenced by the high variance associated with its solutions. In terms of reward value, AD-PGPD outperforms PGDual, although the underperformance of the latter is partly attributed to insufficient convergence time.

\begin{figure}
    \centering
    \includegraphics[width=10cm]{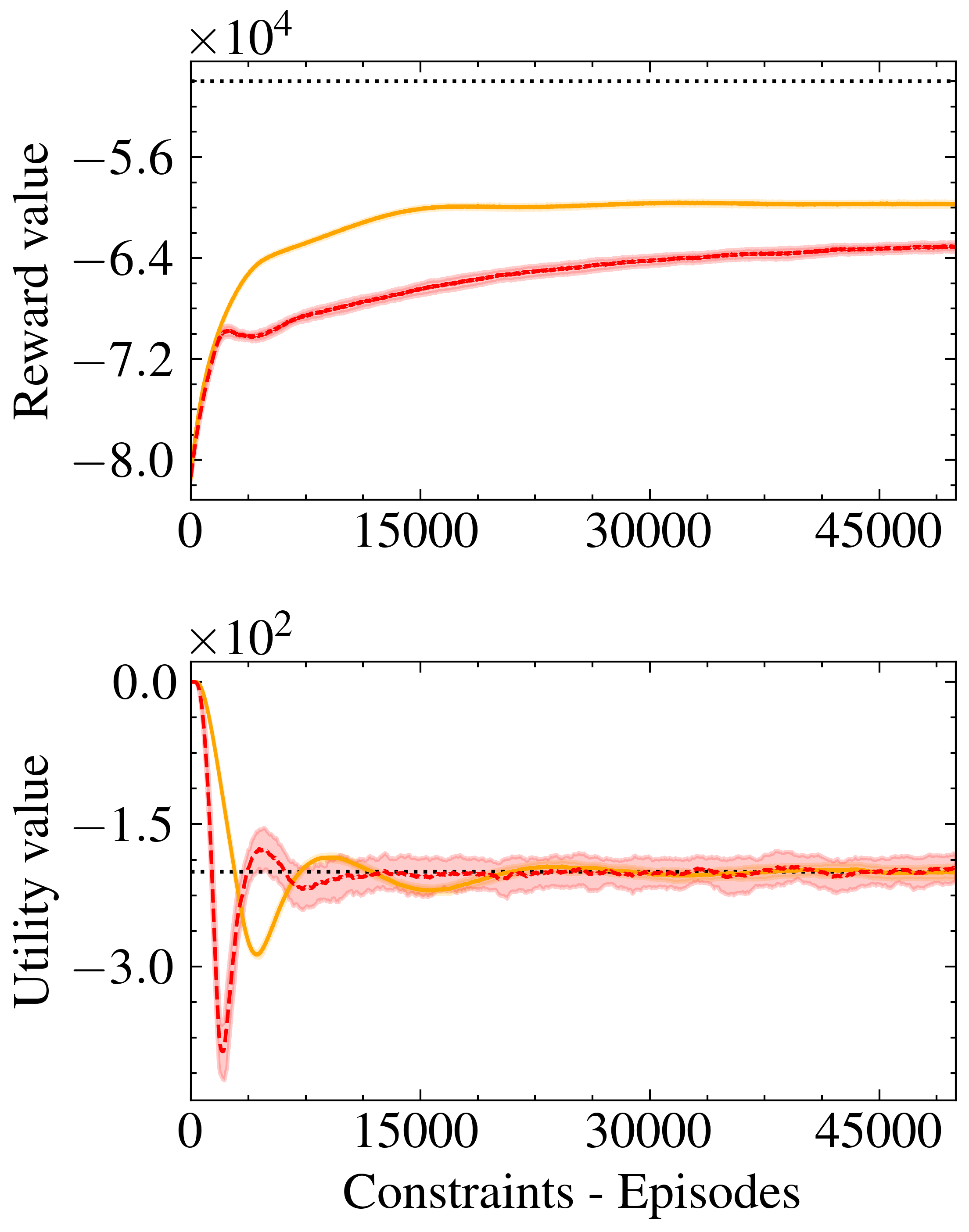}
    \caption{Average reward and utility value functions of policy iterates generated by AD-PGPD (\yellowline) and PGDual (\reddashedline )  in the navigation control problem with restricted zones.}
    \label{fig::fig_9}
\end{figure}

\begin{figure}[t]
    \centering
    \includegraphics[width=10cm]{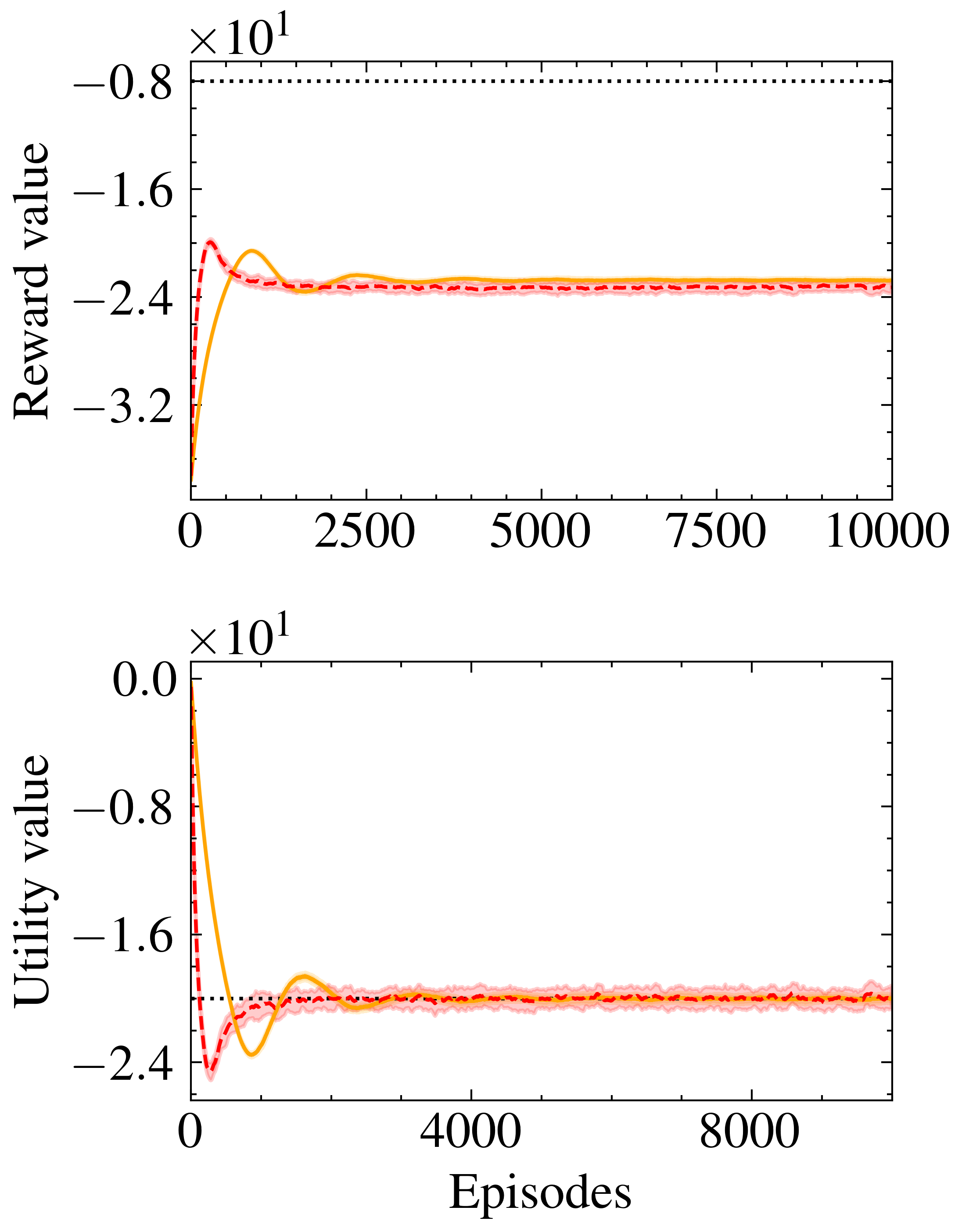}
    \caption{Average reward and utility value functions of policy iterates generated by AD-PGPD (\yellowline) and PGDual (\reddashedline )  in a fluid velocity control.}
    \label{fig::fig_10}
\end{figure}

\subsection{Non-linear Constrained Regulation}

\label{app::experiments_fluid}

We have assessed the performance of D-PGPD in a non-linear control problem, specifically the control of the velocity of an incompressible Newtonian fluid described by the one-dimensional Burgers' equation \cite{baker2000nonlinear}. The velocity profile \( s \) of the fluid varies in a one-dimensional space \( x \) bounded in the interval $[0, 1]$ and time \( t \), described by the equation
\begin{align}
    \nonumber
    \frac{\partial s(x, t)}{\partial t}  & \;=\; \varepsilon \frac{\partial^2 s(x, t)}{\partial x^2}  - \frac{1}{2} \frac{\partial (s(x, t)^2)}{\partial x}   + a(t) + \theta_{\text{max}}(t),
\end{align}

\noindent where $\varepsilon$ is a viscosity coefficient, \( \theta_{\text{max}}(t) \) is some Brownian noise, and  \( a \) is the bounded control action, such as the injection of polymers or mass transport through porous walls. The initial condition is given by \( s(\cdot, 0) \sim \rho \), where  $\rho$ is the initial state distribution.

We discretize the dynamics of the system using an Euler scheme \cite{borggaard2020quadratic}. This approach converts the continuous partial differential equation into a discrete form suitable for numerical computation. The spatial domain \([0, 1]\) is divided into \( d \) intervals, resulting in a grid with points \( x_i = \frac{i}{d} \) for \( i = 0, 1, \ldots, d \). The velocity profile \( s(x, t) \) is approximated at these grid points, resulting in a vector \( s_t \) at each time-step \( t \). Time is discretized into steps of size \( \Delta t \), with \( t = k \Delta t \) for \( k = 0, 1, 2, \ldots \) The time derivative is approximated using a forward Euler method,
\[
    \frac{\partial s(x, t)}{\partial t}  
    \; \approx \;
    \frac{s(x, t + \Delta t) - s(x, t)}{\Delta t}.
\]

The spatial derivatives \( \partial^2 s / \partial x^2 \) and \( \partial (s^2) \partial x \) are approximated using finite differences. The discretized form of Burgers' equation at each grid point \( x_i \) and time-step \( t \) is given by
\begin{align}
    \nonumber
    \frac{s_i^{t+1} - s_i^t}{\Delta t} & \; = \; \varepsilon \frac{s_{i+1}^t - 2s_i^t + s_{i-1}^t}{(\Delta x)^2}  - \frac{1}{2} \left( \frac{(s_{i+1}^t)^2 - (s_{i-1}^t)^2}{2\Delta x} \right) + a^t_i + \omega_i^t,
\end{align}
where $\omega_i^t$ is Gaussian noise, and $a_i^t$ is the control action at point $x_i$. Rearranging gives
\begin{align}
    \nonumber
    s_i^{t+1} & \; = \; s_i^t + \Delta t \left( \varepsilon \frac{s_{i+1}^t - 2s_i^t + s_{i-1}^t}{(\Delta x)^2}  - \frac{1}{2} \left( \frac{(s_{i+1}^t)^2 - (s_{i-1}^t)^2}{2\Delta x} \right) + a_i^t + \omega_t \right).
\end{align}

The discretized dynamics of the system can then be expressed in matrix form as the non-linear system
\[
s_{t+1} \; = \; B_0 s_t + B_1 a_t + B_2 s_t^2 + \omega_t,
\]
where $s_t^2$ is the element-wise squared state, the dimensionality of the state-action space is $d=d_s=d_a$, and  $B_0$, $B_1$, $B_2 \in \reals^{d \times d}$ are matrices representing the discretized spatial operators and non-linear terms, and \( a_t \in \reals^{n} \) is the control input vector. In the proposed scenario, we have selected a discrete grid of $d=10$, a time-step $\Delta t = 0.01$, and a viscosity coefficient $\varepsilon = 0.1$.

The goal is to drive the velocity of the fluid towards $0$, while minimizing the control action. Furthermore, we introduce a constraint to limit the expected control action that the agent can employ. To that end, consider the reward and utility functions
\begin{align}
    \nonumber
    r(s, a) & \;=\; - \|s\|^2 
    \;\; \text{ and } \;\;
    u(s, a)  \;=\; - \|a\|_1.
    \nonumber
\end{align}

The parameters of this experiment are $\tau=0.001$, $\eta=0.001$, and $b=-20$. We have compared the sample-based AD-PGPD against PGDual. The basis function  of the function approximation is $\phi(s,a) = [s^\Tr, a^\Tr]^\Tr \otimes  [s^\Tr, a^\Tr]^\Tr$. The value functions of policy iterates generated by the sample-based version of AD-PGPD and PGDual over $10,000$ iterations averaged across $50$ experiments can be seen in Fig. \ref{fig::fig_10}. The results are consistent with what we have observed in the navigation problem. The AD-PGPD algorithm successfully mitigates oscillations and converges to a feasible solution with low return variance. In contrast, although PGDual achieves similar objective value, it does not dampens oscillations, as indicated by the high variance in its solutions. 

\end{document}